\documentclass[11pt]{article}
\usepackage{natbib,amsmath,amssymb,amsfonts,amsthm,dsfont,enumitem,float,bbm}
\RequirePackage[colorlinks,citecolor=blue,urlcolor=blue,linkcolor=blue, breaklinks]{hyperref}
\usepackage[normalem]{ulem}
\usepackage{epsfig,longtable,xcolor,algpseudocode}
\usepackage[linesnumbered, ruled, noline]{algorithm2e}
\usepackage{wrapfig}
\usepackage{multirow}
\usepackage{caption}
\usepackage{subcaption}
\usepackage[export]{adjustbox}
\usepackage{booktabs}
\usepackage{array}
\usepackage{listings}
\usepackage{xcolor}
\usepackage{graphicx}
\usepackage[export]{adjustbox}
\usepackage{subcaption}
\usepackage{arydshln}
\usepackage{setspace}
\newcommand{\bfm}[1]{\ensuremath{\mathbf{#1}}}
               
\def\bb{\bfm b}               
     \def\bC{\bfm C}          
          \def\cD{{\cal  D}}

\def\bg{\bfm g}               
\def\bh{\bfm h}

\def\bp{\bfm p}

          \def\cS{{\cal  S}}     
     \def\bT{\bfm T}          
               
\def\bv{\bfm v}               
               
\def\bx{\bfm x}

\DeclareMathOperator*{\argmin}{argmin}
\DeclareMathOperator*{\argmax}{argmax}

\DeclareMathOperator{\softmax}{softmax}
\DeclareMathOperator{\ReLU}{ReLU}
\DeclareMathOperator{\conv}{conv}

\newcommand{\beq}  {\begin{equation}}
\newcommand{\eeq}  {\end{equation}}
\newcommand{\beqn} {\begin{eqnarray}}
\newcommand{\eeqn} {\end{eqnarray}}
\newcommand{\beqnn}{\begin{eqnarray*}}
\newcommand{\eeqnn}{\end{eqnarray*}}

\theoremstyle{definition}
\newtheorem{dfn}{Definition}

\theoremstyle{plain}
\newtheorem{lem}{Lemma}

\newtheorem{con}{Condition}

\newtheorem{thm}{Theorem}

\providecommand{\keywords}[1]
{
  \small	
  \textbf{Keywords:} #1
}

\title{\textbf{Wasserstein Filtering: A Sample Selection Method for Robust Distribution Learning}}
\author{Yikai Xu\thanks{Department of Applied Mathematics, The Hong Kong Polytechnic University, Hong Kong, China. Email: yikai.xu@polyu.edu.hk},  Zhao Chen\thanks{School of Data Science, Fudan University, Shanghai, China. Email: zchen\_fdu@fudan.edu.cn}, and
Jian Huang\thanks{Department of Data Science and Artificial Intelligence, and Department of Applied Mathematics, The Hong Kong Polytechnic University, Hong Kong, China. Email: j.huang@polyu.edu.hk}}
\begin{document}

\maketitle

\begin{abstract}
Given a dataset where a portion of the samples are contaminated, our goal is to recover the underlying clean population distribution. To this end, we propose \emph{Wasserstein Filtering} (WF), a novel sample selection framework that discards a fraction of suspicious samples and estimates the target distribution using the empirical measure of the remaining data. The core insight is to select a subset of samples whose empirical distribution maximizes its Wasserstein distance to the fully contaminated empirical distribution, thereby preferentially isolating and removing geometrically influential outliers.
To render this optimization computationally tractable, we introduce three algorithms: a marginal screening scheme, \texttt{SinkMarg}, and two joint optimization algorithms, \texttt{SinkWF} and \texttt{SlicedWF}, leveraging entropic optimal transport and sliced Wasserstein approximations, respectively. On the theoretical front, we introduce the \emph{Far Exclusion and Local Projection} (FELP) contamination model, which characterizes corruptions consisting of well-separated outliers and locally indistinguishable perturbations. Under this model, we prove that the WF estimator achieves minimax optimality over distribution families with bounded covariance.
Extensive numerical experiments on synthetic datasets, benchmark anomaly detection suites, and robust generative learning with diffusion models demonstrate that WF serves as a highly practical, model-agnostic preprocessing tool. It delivers competitive outlier detection performance and provides substantial downstream benefits for generative modeling under heavy contamination.

\end{abstract}


% AMS 62G05, 62G35

%%%%%%%%%%%%%%%%%%%%%%%%%%%%%%%%%%%%%%%%%%%%%%%%%%%%%%%%%%%%%%%%%%%%%%%%%%%%%%

\section{Introduction}

Let $\mathcal{D} = \{\mathbf{x}_1, \dots, \mathbf{x}_n\}$ be a set of i.i.d. samples drawn from a population distribution $P^*$. We consider an adversarial corruption model where an adversary can inspect the dataset and modify up to an $\alpha$-fraction of the samples, yielding a contaminated dataset $\widetilde{\mathcal{D}} = \{\widetilde{\mathbf{x}}_1, \dots, \widetilde{\mathbf{x}}_n\}$. The contamination budget $\alpha \in [0, 1)$ constrains the adversary such that:
\begin{align*}
\sum_{i=1}^n \mathbb{I}_{\{\widetilde{\mathbf{x}}_i \neq \mathbf{x}_i\}} \leq \alpha n.
\end{align*}
This formulation is commonly referred to as adversarial or Total Variation (TV) contamination, reflecting the fact that the TV distance between the clean and contaminated empirical measures satisfies $\|\widetilde{P}_n - P_n\|_{\mathrm{TV}} \leq \alpha$ \citep{diakonikolas2019robust, nietert2023robust}. Here, $P_n = \frac{1}{n}\sum_{i=1}^n \delta_{\mathbf{x}_i}$ is the clean empirical measure, $\widetilde{P}_n = \frac{1}{n}\sum_{i=1}^n \delta_{\widetilde{\mathbf{x}}_i}$ is the contaminated empirical measure, and $\delta_{\mathbf{x}}$ represents the Dirac delta measure centered at $\mathbf{x}$.

In this work, we address the following fundamental question in robust statistics:
\begin{quote}
\itshape
Can we design an accurate estimator and a computationally tractable algorithm to recover the true distribution $P^*$ under the Wasserstein-1 metric, given samples subject to Total Variation (TV) contamination?
\end{quote}

%In this work, we address the following basic question in robust statistics:
%\begin{center}
%\fbox{\parbox{0.85\linewidth}{
%\centering\itshape
%Can we design an accurate estimator and a computationally tractable algorithm to recover the true %distribution $P^*$ under the Wasserstein-1 metric, given samples subject to TV contamination?
%}}
%\end{center}

We adopt the Wasserstein-1 distance as our metric of distributional discrepancy because it metrizes weak convergence, naturally tolerates support mismatch, and is highly sensitive to the geometry of outliers \citep{nietert2023robust}. Furthermore, by considering the TV contamination model, we impose minimal assumptions on the noise structure and allow for worst-case, adaptive, and malicious corruptions of the data.

While existing methods, such as iterative filtering \citep{nietert2024robust}, achieve strong theoretical convergence guarantees under TV contamination, they suffer from practical limitations. Specifically, their algorithms rely heavily on the unrealistic assumption that $P^*$ is isotropic, and they require prior knowledge of unknown spectral parameters associated with the covariance of $P^*$.

To address these limitations, we propose an algorithm that dispenses with the isotropic assumption and operates entirely independently of the parameters of $P^*$. The trade-off for this flexibility is that we establish the theoretical optimality of our estimator under TV contamination subject to mild, physically meaningful structural constraints that make the optimization landscape amenable to analysis.

\subsection{Key Contributions}

In this work, given a contamination proportion $\alpha$, we formulate robust distribution learning as a constrained \emph{sample selection} problem. Specifically, we identify and filter out suspicious outliers, using the empirical measure of the remaining samples as a robust representative of the clean population distribution $P^*$.

From this perspective, we introduce the {Wasserstein Filtering (WF)} estimator. Conceptually, the WF estimator searches over all possible subsets of the contaminated data of a size determined by the contamination rate $\alpha$. It then selects the subset whose empirical distribution maximizes the Wasserstein-1 distance to the overall contaminated empirical distribution. By framing outlier detection as this targeted optimization problem, the estimator systematically isolates and discards the samples that exert the most severe geometric distortion on the empirical measure.

Our main contributions are both methodological and theoretical, summarized as follows:
\begin{enumerate}
    \item {Novel Algorithmic Frameworks:} We propose three distinct algorithms to compute the WF estimator: one marginal scheme (\texttt{SinkMarg}) and two joint optimization schemes (\texttt{SinkWF} and \texttt{SlicedWF}). These algorithms scale flexibly across different sample sizes and dimensionalities.
    \item {Rigorous Theoretical Guarantees:} Beyond showing strong empirical performance on par with state-of-the-art heuristic outlier detection methods, we establish the \emph{minimax optimality} of the WF estimator under robust distribution estimation. Furthermore, we provide elementary optimization convergence rates for our joint algorithms.
    \item {Task- and Model-Agnostic Utility:} A unique benefit of our sample selection formulation is its task-independence. It serves as a modular preprocessing filter that cleans the dataset prior to downstream modeling. In Section 5, we demonstrate this versatility by pairing our filter with modern deep generative models (\texttt{DDPM}). Unlike existing approaches that modify the loss function in a model-dependent manner \citep{balaji2020robust, nietert2023robust}, our approach is entirely model-agnostic and can safeguard any downstream generative framework.
\end{enumerate}

The rest of this paper is organized as follows. In Section 2, we introduce our Wasserstein Filtering framework, detailing both the mathematical estimator and its three algorithmic realizations. Rigorous theoretical properties, including minimax optimality and optimization convergence, are established in Section 3. In Section 4, we present systematic numerical experiments on synthetic and benchmark datasets to validate and compare the empirical performance of our methods. Section 5 demonstrates the practical utility of our framework as a task-agnostic data-filtering step for modern generative models. Finally, Section 6 concludes the paper.

\subsection{Related Work}

Robust statistics has accumulated a rich body of seminal research, with the fundamental concepts dating back to the pioneering work of \citet{huber1964robust}. Over the decades, substantial efforts have been devoted to designing robust estimation and inference procedures capable of handling corrupted data; for a comprehensive modern review, see \citet{loh2025theoretical}.

In this context, we highlight the literature most relevant to robust distribution learning. The seminal work of \citet{donoho1988automatic} first formulated the robust distribution estimation problem under the population TV-contamination model and introduced Minimum Distance Estimation (\texttt{MDE}) at the population level. More recently, \citet{zhu2022generalized} systematically generalized \texttt{MDE}, utilizing it to design mean and covariance estimation procedures that achieve population minimax optimality alongside strong empirical guarantees.

For nonparametric density estimation, \citet{liu2019density} established the minimax optimal rate of pointwise density estimation under Huber's contamination model, demonstrating that kernel density estimators (\texttt{KDE}) can achieve this rate when equipped with an optimal bandwidth. However, this optimal bandwidth relies heavily on the smoothness of the underlying clean density and the contamination proportion, both of which are typically unknown in practice. \citet{chao2023statistical} derived analogous minimax rates under Wasserstein contamination. On the methodological front, several algorithmic variants have been proposed to robustify density estimation, including the integration of Huber-type loss functions with \texttt{KDE} \citep{kim2012robust} and the application of the median-of-means principle to kernel frameworks \citep{humbert2022robust}.

Partial Optimal Transport (OT), also referred to as unbalanced OT or robust OT, is closely related to our framework. Originally developed within the mathematical analysis community \citep{caffarelli2010free, figalli2010optimal, liero2018optimal, chizat2018unbalanced}, partial OT was subsequently introduced to the statistics and machine learning communities to relax the strict mass-preserving constraint of standard OT.

Recent methodological developments have leveraged this relaxation to handle noisy data. For instance, \citet{balaji2020robust} defined a robust Wasserstein distance by incorporating a $\chi^2$-divergence penalty on the marginal distributions. By replacing the standard Wasserstein distance in Wasserstein GANs (\texttt{WGAN}) \citep{arjovsky2017wasserstein} with this robust formulation, they trained generative models that are structurally insensitive to outliers. \citet{mukherjee2021outlier} substituted the $\chi^2$-divergence with a Total Variation (TV) penalty, deriving an equivalent formulation that thresholds the cost matrix, which is significantly more computationally convenient.

Subsequently, the theoretical foundations of this TV-penalized robust Wasserstein distance have been rigorously established, spanning measure-theoretic dualities, optimization landscapes, and statistical convergence properties, particularly its convergence rates in robust distribution estimation when plugged into the \texttt{MDE} framework \citep{ma2025inference, nietert2023robust}. Parallel to these theoretical advances, the computational aspects of unbalanced OT \citep{pham2020unbalanced, nguyen2023unbalanced} and its associated \texttt{MDE} optimization schemes \citep{nietert2024robust} have also been extensively investigated.

Another closely related line of research is unsupervised outlier detection. Unlike robust statistics, this field is often heavily empirically driven, focusing on practical heuristic design and computational scalability rather than formal problem abstractions or minimax theoretical guarantees. Classical outlier detection methods typically assign an anomaly score to each sample based on criteria such as local density, distance, recursive isolation, reconstruction error, or marginal extremeness.

For instance, \texttt{COPOD} \citep{li2020copod} leverages empirical copulas and tail probabilities to construct a computationally efficient and highly interpretable outlier score. More recently, Deep Isolation Forest (\texttt{DIF}) \citep{xu2023deep} enhances traditional isolation-based detection by applying random neural network representations prior to isolation-tree partitioning. Large-scale benchmarks, such as \texttt{ADBench} \citep{han2022adbench}, have systematically evaluated these diverse anomaly detection algorithms across tabular and graph data.
 %, and time-series
%modalities, highlighting their practical utility in applications.

\section{Methodology}
\subsection{Estimator}
To define the proposed estimator, we first recall the definition of the Wasserstein distance \citep{villani2009optimal}. Let $(\mathcal{X}, d)$ be a Polish metric space, and let $\mathcal{Q}(\mathcal{X})$ denote the set of all Borel probability measures on $\mathcal{X}$. For $p \geq 1$, we define the space of probability measures with finite $p$-th moment as
\begin{align}
\label{PpDef}
\mathcal{Q}_p(\mathcal{X}) := \left\{ \mu \in \mathcal{Q}(\mathcal{X}) : \int_{\mathcal{X}} d(\mathbf{x}, \mathbf{x}_0)^p \, \mathrm{d}\mu(\mathbf{x}) < \infty \text{ for some } \mathbf{x}_0 \in \mathcal{X} \right\}.
\end{align}
For any $\mu, \nu \in \mathcal{Q}_p(\mathcal{X})$, the Wasserstein-$p$ distance between them is defined as:
\begin{align*}
W_p(\mu,\nu) := \left( \inf_{\pi\in\Pi(\mu,\nu)} \int_{\mathcal{X}\times\mathcal{X}} d(\mathbf{x},\mathbf{y})^p \, \mathrm{d}\pi(\mathbf{x},\mathbf{y}) \right)^{\frac{1}{p}},
\end{align*}
where $\Pi(\mu,\nu)$ denotes the set of all Borel probability measures (or couplings) on $\mathcal{X}\times\mathcal{X}$ having marginal distributions $\mu$ and $\nu$, respectively. Throughout this paper, unless specified otherwise, we assume $\mathcal{X} \subseteq \mathbb{R}^d.$ Let $d$ be the standard Euclidean metric (i.e., $d(\mathbf{x},\mathbf{y}) = \|\mathbf{x} - \mathbf{y}\|_2$), and set $p=1$. That is, we restrict our focus to the Wasserstein-1 distance on $\mathbb{R}^d$, denoted simply as $W_1.$

Let $\widetilde{\mathcal{D}} = \{\widetilde{\mathbf{x}}_1, \dots, \widetilde{\mathbf{x}}_n\} \subset \mathcal{X} \subseteq \mathbb{R}^d$ be a set of possibly contaminated data points, and let $\widetilde{P}_n = \frac{1}{n} \sum_{i=1}^n \delta_{\widetilde{\mathbf{x}}_i}$ denote the corresponding contaminated empirical measure. For a given contamination fraction $\alpha \in [0, 1)$, we seek to identify a target distribution within the family of sub-empirical measures obtained by removing $\alpha n$ outliers from $\widetilde{\mathcal{D}}$. Specifically, we look for the sub-measure that maximizes the Wasserstein distance to the contaminated empirical measure $\widetilde{P}_n$:
\begin{align}
\label{eq1}
\widehat{P}_{\operatorname{WF}} = \mathop{\operatorname{argmax}}_{P \in \mathcal{P}_\alpha} W_1(\widetilde{P}_n, P),
\end{align}
where the candidate family $\mathcal{P}_\alpha$ of trimmed empirical measures is defined as:
\begin{align*}
\mathcal{P}_\alpha = \left\{ \mu \in \mathcal{Q}(\mathcal{X}) : \mu = \frac{1}{(1-\alpha)n} \sum_{\widetilde{\mathbf{x}}_i \in \widetilde{\mathcal{D}}_\alpha} \delta_{\widetilde{\mathbf{x}}_i}, \ \widetilde{\mathcal{D}}_{\alpha} \subset \widetilde{\mathcal{D}}, \ |\widetilde{\mathcal{D}}_{\alpha}| = (1-\alpha)n \right\}.
\end{align*}
By reformulating \eqref{eq1} in vector form, our goal is to find an optimal probability mass vector $\widehat{\mathbf{p}} \in \mathbb{R}^n$ such that:
\begin{align}\label{problem:discrete}
\widehat{\mathbf{p}} = \mathop{\operatorname{argmax}}_{\mathbf{p} \in \Delta_\alpha} W_1(\widetilde{P}_n, P_{\mathbf{p}}),
\end{align}
where $P_{\mathbf{p}} = \sum_{i=1}^n p_i \delta_{\widetilde{\mathbf{x}}_i}$ is the discrete measure induced by the weights $\mathbf{p} = (p_1, \dots, p_n)^\top$. Here, $\Delta_\alpha$ represents the scaled, discretized simplex:
\begin{align}
\Delta_\alpha = \left\{ \mathbf{v} \in \mathbb{R}^n : \mathbf{v}^\top \mathbf{1} = 1, \ v_i \in \left\{0, \frac{1}{(1-\alpha)n}\right\} \text{ for all } i \in [n] \right\}.
\end{align}

As the discrete optimization problem \eqref{problem:discrete} is NP-hard, we relax it to its convex hull:
\begin{align}\label{problem:convex_hull_relaxation}
\widehat{\mathbf{p}}_{\rm WF} = \argmax_{\mathbf{p} \in \operatorname{conv}(\Delta_\alpha)} W_1(\widetilde{P}_n, P_{\mathbf{p}}),
\end{align}
where the convex hull of $\Delta_\alpha$ is given by
\begin{align}\label{constriant:cap_simplex}
\operatorname{conv}(\Delta_\alpha) = \left\{ \mathbf{v} \in \mathbb{R}^n : \mathbf{v}^\top \mathbf{1} = 1, \, 0 \leq v_i \leq \frac{1}{(1-\alpha)n}, \, \forall i \in [n] \right\}.
\end{align}
The optimal solution to \eqref{problem:convex_hull_relaxation}, denoted by $\widehat{\mathbf{p}}_{\rm WF}$, serves as our estimated mass vector. Once $\widehat{\mathbf{p}}_{\rm WF}$ is obtained, we filter out contaminated samples based on its entries. Specifically, we estimate the clean distribution $P^*$ using the empirical distribution of the retained samples:
\begin{align*}
\widehat{P}^* = \frac{1}{(1-\alpha)n} \sum_{i: \, \widehat{p}_{{\rm WF}, i} > \tau_\alpha} \delta_{\widetilde{\mathbf{x}}_i},
\end{align*}
where $\tau_\alpha$ is the $\alpha$-quantile of the entries of $\widehat{\mathbf{p}}_{\rm WF} = (\widehat{p}_{{\rm WF}, 1}, \dots, \widehat{p}_{{\rm WF}, n})^\top$. Since $\widehat{P}^* \in \mathcal{P}_\alpha$ by construction, it serves as a natural discrete approximation of $\widehat{P}_{\rm WF}$. In the following, we present one algorithm to solve the discrete formulation \eqref{problem:discrete} and two algorithms to solve the relaxed formulation \eqref{problem:convex_hull_relaxation}.

\subsection{Algorithms}
\subsubsection{SinkMarg}
By expressing the Wasserstein-1 distance in its primal form, we can reformulate the discrete optimization problem \eqref{problem:discrete} as a discrete optimal transport problem:
\begin{align}
\widehat{\mathbf{p}} = \argmax_{\mathbf{p} \in \Delta_\alpha} \min_{\mathbf{T} \in \Pi\left(\frac{1}{n}\mathbf{1}, \mathbf{p}\right)} \langle \mathbf{T}, \mathbf{C} \rangle,
\end{align}
where $\mathbf{C} \in \mathbb{R}_+^{n \times n}$ is the cost matrix with entries $C_{ij} = \|\widetilde{\mathbf{x}}_i - \widetilde{\mathbf{x}}_j\|_2$ representing the Euclidean distance between the $i$-th and $j$-th samples, and $\Pi(\mathbf{u}, \mathbf{v}) = \{ \mathbf{T} \in \mathbb{R}_+^{n \times n} : \mathbf{T}\mathbf{1} = \mathbf{u}, \, \mathbf{T}^\top\mathbf{1} = \mathbf{v} \}$ denotes the set of feasible transport plans.

Since computing the exact Wasserstein distance via linear programming is computationally expensive, we instead employ its entropy-regularized approximation \citep{cuturi2013sinkhorn}:
\begin{align*}
W_{1,\lambda}(\widetilde{P}_n, P_{\mathbf{p}}) = \min_{\mathbf{T} \in \Pi\left(\frac{1}{n}\mathbf{1}, \mathbf{p}\right)} \langle \mathbf{T}, \mathbf{C} \rangle + \lambda \sum_{i,j \in [n]} T_{ij} \log(T_{ij}),
\end{align*}
where $\lambda > 0$ is the regularization parameter, and $T_{ij}$ is the $(i,j)$-th entry of the transport plan $\mathbf{T}$.

Let $\mathbf{p}_{-i} \in \mathbb{R}^n$ be a probability vector whose $i$-th entry is $0$ and all other entries are equal to $\frac{1}{n-1}$. A natural heuristic for estimating $\widehat{\mathbf{p}}$ is to assign zero mass to the $\alpha n$ points that yield the largest leave-one-out (LOO) scores: $W_{1,\lambda}(\widetilde{P}_n, P_{\mathbf{p}_{-i}})$ for $i \in [n]$. The intuition behind this approach is straightforward: if removing a single sample causes a significant shift in the distribution, that sample is highly likely to be a contaminated outlier.

Based on this intuition, we introduce a marginal screening algorithm to solve \eqref{problem:discrete}, as outlined in Algorithm 1. As demonstrated in our numerical experiments, this algorithm performs well when the outliers are isolated. However, due to its purely marginal nature, it fails to capture the group structure of clustered outliers and incurs a heavy computational burden when $n$ is large.

\IncMargin{1em}
\begin{algorithm}[!htp]
    \SetAlgoLined
    \SetAlgoNlRelativeSize{-1}
    \SetNlSty{textbf}{}{:}
    \caption{Sinkhorn Marginal Screening Algorithm}

    \Indentp{-1.35em}
    \KwIn{$\widetilde{\cD}$, $\alpha\in(0,1)$, $\lambda\in\mathbb{R}_+$}
    \Indentp{1.35em}
    Normalize $\widetilde{\cD}$ \\
    \For{$i = 1, \ldots, n$}{
    Compute $s_i=W_{1,\lambda}(\widetilde{P}_n,P_{\bp_{-i}})$
    }
    Thresholding: filter $\widetilde{\cD}$ to obtain $\cS=\{\widetilde{\bx}_i | s_i<\tau_{1-\alpha} ,i\in[n]\}$ where $\tau_{1-\alpha}$ is the $1-\alpha$ quantile of leave-one-out scores $\{s_i,i\in[n]\}$ \\
    \Indentp{-1.35em}
    \KwOut{$\cS$}
    \Indentp{1.35em}
\end{algorithm}
\DecMargin{1em}

\subsubsection{SinkWF}

To solve the relaxed formulation \eqref{problem:convex_hull_relaxation}, we consider its entropy-regularized version:
\begin{align}\label{problem:sinkhorn}
\widehat{\mathbf{p}}_{\mathrm{SinkWF}} = \operatorname*{argmax}_{\mathbf{p} \in \conv(\Delta_\alpha)}W_{1,\lambda}(\widetilde{P}_n, P_{\mathbf{p}}),
% \underbrace{\min_{\mathbf{T} \in \Pi\left(\frac{1}{n}\mathbf{1}, \mathbf{p}\right)} \langle %\mathbf{T}, \mathbf{C} \rangle + \lambda \sum_{i,j \in [n]} T_{ij} %\log(T_{ij})}_{W_{1,\lambda}(\widetilde{P}_n, P_{\mathbf{p}})},
\end{align}
where
\[
W_{1,\lambda}(\widetilde{P}_n, P_{\mathbf{p}})=\min_{\mathbf{T} \in \Pi\left(\frac{1}{n}\mathbf{1}, \mathbf{p}\right)} \langle \mathbf{T}, \mathbf{C} \rangle + \lambda \sum_{i,j \in [n]} T_{ij} \log(T_{ij}).
\]

Our proposed optimization scheme is detailed in Algorithm 2. To leverage gradient-based optimization, we transform the constrained problem \eqref{problem:sinkhorn} into an unconstrained, smooth formulation.

First, we reparameterize the probability vector $\mathbf{p}$ using the softmax function to ensure it lies on the simplex: $\mathbf{p}(\mathbf{b}) = \softmax(\mathbf{b})$ for $\mathbf{b} \in \mathbb{R}^n$. Second, we incorporate the upper-bound constraint of the capped simplex $\conv(\Delta_\alpha)$ as a penalty term added to the objective:
\begin{align*}
\mathcal{L}(\mathbf{b}) = - W_{1,\lambda}(\widetilde{P}_n, P_{\softmax(\mathbf{b})}) + \beta \mathbf{1}^\top \ReLU\left( \softmax(\mathbf{b}) - \frac{1}{(1-\alpha)n}\mathbf{1} \right),
\end{align*}
where $\ReLU(\mathbf{x}) = \max(\mathbf{x}, \mathbf{0})$ is applied element-wise, and $\beta > 0$ is a penalty parameter.

To update $\mathbf{b}$ via gradient descent, we need to compute the gradient $\mathbf{g}_r = \nabla_{\mathbf{b}} \mathcal{L}(\mathbf{b}_r)$ at the $r$-th iteration. Rather than backpropagating directly through the Sinkhorn iterations using automatic differentiation (e.g., PyTorch's \texttt{AutoGrad}), which is highly memory-intensive, we exploit the duality theory of optimal transport. Specifically, the gradient of the entropic Wasserstein loss with respect to the marginal distribution is given by:
\begin{align*}
\frac{\partial W_{1,\lambda}(\widetilde{P}_n, P_{\mathbf{p}})}{\partial \mathbf{p}} = \mathbf{v}^*,
\end{align*}
where $\mathbf{v}^* \in \mathbb{R}^n$ is the optimal dual potential associated with $\mathbf{p}$, which is a direct byproduct of the Sinkhorn algorithm \citep{sinkhorn1964relationship, cuturi2013sinkhorn}. Using this dual formulation to compute the gradient $\mathbf{g}_r$ bypasses the need to store the computational graph of the Sinkhorn iterations, significantly reducing both time and memory complexity.

%\paragraph{Remark on Gradient Approximations.}

After solving the entropic optimal transport problem to obtain the optimal plan $\mathbf{T}_r$, one might consider using the unregularized cost $\langle \mathbf{T}_r, \mathbf{C} \rangle$ as a more accurate approximation of the true Wasserstein distance, as suggested by \cite{luise2018differential}. However, computing the gradient $\partial \langle \mathbf{T}_r, \mathbf{C} \rangle / \partial \mathbf{p}$ requires either backpropagating through the Sinkhorn steps or solving a system of linear equations with $\mathcal{O}(n^3)$ complexity, making it computationally prohibitive.

Using the obtained gradient $\mathbf{g}_r$, we employ the Adam optimizer \citep{kingma2014adam} to update $\mathbf{b}_r$ for $R$ iterations. To enhance the stability of our estimator, we repeat this optimization process $T$ times with different initializations. We then take the average of the $T$ resulting probability vectors, $T^{-1} \sum_{t=1}^T \softmax(\mathbf{b}_R^{(t)})$, as the final screening score, filtering out the $\alpha n$ samples with the smallest scores as suspected outliers.
The complete procedure for \texttt{SinkWF} is detailed in Algorithm 2.

In addition to this penalty-based approach for handling the capped simplex constraint, we also provide a projection-based optimization variant, alongside an automatic selection strategy for the hyperparameter $\alpha$. These complementary details are discussed in Appendix A.1-A.2.

%\IncMargin{1em}
\medskip
\begin{algorithm}[H]
    \SetAlgoLined
    \SetAlgoNlRelativeSize{-1}
    \SetNlSty{textbf}{}{:}
    \caption{Sinkhorn Wasserstein Filtering Algorithm}

    \Indentp{-1.35em}
    \KwIn{$\widetilde{\cD}$, $\alpha\in(0,1)$, $\bb_0$, $\lambda,\gamma\in\mathbb{R}_+$, $R,T\in\mathbb{Z}_+$}
    \Indentp{1.35em}
    Normalize $\widetilde{\cD}$ \\
    \For{$t=1,\ldots,T$}{
        \For{$r = 0, \ldots, R-1$}{
        Solve the following problem:
        \begin{align}\label{eq:entropy_problem_in_SinkWF}
        \bT_{r}=\argmin_{\bT\in\Pi\left(\frac{1}{n}\mathbf{1}, {\rm softmax}(\bb_r)\right)}\left<\bT,\bC\right> + \lambda\sum_{i,j\in[n]}t_{ij}\log(t_{ij})
        \end{align}
        Compute gradient on $\bb_r$:
        \begin{align}
        \bg_r=\left\{\bv^*-\frac{\partial\beta\mathbf{1}^\top {\rm Relu}\left({\rm softmax}(\bb_r)- \frac{1}{(1-\alpha)n}\mathbf{1}\right)}{\partial{\rm solftmax}(\bb_r)}\right\}\frac{\partial{\rm solftmax}(\bb_r)}{\partial \bb_r}
        \end{align}
        Adam optimizer update:
        \begin{align}
        \bb_{r+1}={\rm Adam}(\gamma, \bb_r, \bg_{j, j\leq r})
        \end{align}
        }
        Save $\bh_t={\rm softmax}(\bb_R)$
    }
    Thresholding: filter $\widetilde{\cD}$ to obtain $\cS=\{\bx_i| \bar{h}_i>\tau_\alpha ,i\in[n]\}$ where $\bar{h}_i$ is the $i$th entry of $\bar{\bh}=\frac{1}{T}\sum_{t=1}^T\bh_t$,  $\tau_\alpha$ is the $\alpha$ quantile of $\bar{h}_i,i\in[n]$ \\
    \Indentp{-1.35em}
    \KwOut{$\cS$}
    \Indentp{1.35em}
\end{algorithm}

\subsubsection{SlicedWF}

In the \texttt{SinkWF} formulation \eqref{problem:sinkhorn}, the cost matrix $\mathbf{C}$ must be explicitly constructed and stored. This incurs an $\mathcal{O}(n^2)$ space and time complexity, which can be computationally prohibitive for large-scale datasets. To address this limitation, we consider a more memory-efficient alternative to the classical Wasserstein distance: the Sliced Wasserstein (SW) distance \citep{peyre2019computational}. The SW distance is defined as:
\begin{align*}
W_{1,S}(\widetilde{P}_n, P_{\mathbf{p}}) = \frac{1}{S} \sum_{i=1}^S W_1\left( P_{\mathbf{u}_i^\top \widetilde{\mathbf{x}}, \, \widetilde{\mathbf{x}} \sim \widetilde{P}_n}, P_{\mathbf{u}_i^\top \mathbf{y}, \, \mathbf{y} \sim P_{\mathbf{p}}} \right),
\end{align*}
where $\mathbb{S}^{d-1}$ denotes the $d$-dimensional unit sphere, and $\mathbf{u}_i \sim \mathbb{S}^{d-1}$ are random projection directions. The parameter $S$ controls the number of projections used to approximate the continuous Sliced Wasserstein distance. By projecting the $d$-dimensional data onto these random $1$-dimensional directions, computing the SW distance simplifies to averaging $S$ independent $1$-dimensional Wasserstein distances.

The corresponding estimator under this formulation is given by:
\begin{align}\label{problem:sliced}
\widehat{\mathbf{p}}_{\mathrm{SlicedWF}} = \operatorname*{argmax}_{\mathbf{p} \in \conv(\Delta_\alpha)}W_{1,S}(\widetilde{P}_n, P_{\mathbf{p}}),
 %\underbrace{\frac{1}{S} \sum_{i=1}^S W_1\left( P_{\mathbf{u}_i^\top \widetilde{\mathbf{x}}, \, %\widetilde{\mathbf{x}} \sim \widetilde{P}_n}, P_{\mathbf{u}_i^\top \mathbf{y}, \, \mathbf{y} \sim %P_{\mathbf{p}}} \right)}_{W_{1,S}(\widetilde{P}_n, P_{\mathbf{p}})},
\end{align}
where
\[
W_{1,S}(\widetilde{P}_n, P_{\mathbf{p}})=\frac{1}{S} \sum_{i=1}^S W_1\left( P_{\mathbf{u}_i^\top \widetilde{\mathbf{x}}, \, \widetilde{\mathbf{x}} \sim \widetilde{P}_n}, P_{\mathbf{u}_i^\top \mathbf{y}, \, \mathbf{y} \sim P_{\mathbf{p}}}\right).
\]
A well-known property of the $1$-dimensional Wasserstein-1 distance is that it can be computed in closed form via the cumulative distribution functions \citep{peyre2019computational}:
\begin{align*}
W_1\left( P_{\mathbf{u}_i^\top \widetilde{\mathbf{x}}, \, \widetilde{\mathbf{x}} \sim \widetilde{P}_n}, P_{\mathbf{u}_i^\top \mathbf{y}, \, \mathbf{y} \sim P_{\mathbf{p}}} \right) = \int_{0}^1 \left| F^{-1}_{\mathbf{u}_i^\top \widetilde{\mathbf{x}}, \, \widetilde{\mathbf{x}} \sim \widetilde{P}_n}(t) - F^{-1}_{\mathbf{u}_i^\top \mathbf{y}, \, \mathbf{y} \sim P_{\mathbf{p}}}(t) \right| dt,
\end{align*}
where $F^{-1}(t)$ denotes the quantile function (inverse CDF) of the projected distributions.

This $1$-dimensional distribution matching can be solved sorted-by-index in $\mathcal{O}(n \log n)$ time and $\mathcal{O}(n)$ space. Consequently, the Sliced Wasserstein approach bypasses the storage and computation of the $\mathcal{O}(n^2)$ pairwise cost matrix $\mathbf{C}$, drastically reducing memory requirements and running time. However, to maintain approximation accuracy, the required number of projections $S$ generally scales with the data dimension $d$.

The complete optimization process is detailed in Algorithm 1 in the Appendix A.3.
%Algorithm 3.
An alternative projection-based variant for enforcing the capped simplex constraint $\conv(\Delta_\alpha)$ is also deferred to Appendix A.1.

\section{Theoretical Properties}

\subsection{Convergence of the Estimator}

In this section, we establish the statistical convergence of $\widehat{P}_{\rm WF}$ to the true distribution $P^*$. We begin by introducing several essential concepts required to define our data contamination model.
Recall that $(\mathcal{X}, d)$ is a Polish metric space (i.e., a complete and separable metric space), $\mathcal{Q}(\mathcal{X})$ denotes the space of all probability measures on $\mathcal{X}$, and $\mathcal{Q}_p(\mathcal{X})$ is defined as in \eqref{PpDef}. While our theoretical results are established within this general framework, it is helpful to keep in mind the standard setting where $\mathcal{X}$ is a compact subset of $\mathbb{R}^d$ equipped with the Euclidean metric $d(\mathbf{x}, \mathbf{y}) = \|\mathbf{x} - \mathbf{y}\|_2$, and $p=1$.

\begin{dfn}[$W_p$ Resilience {\citep[Definition 1]{nietert2023robust}}]
Let $0 \leq \alpha < 1$ and $\rho \geq 0$. A distribution $\mu \in \mathcal{Q}_p(\mathcal{X})$ is said to be $(\rho, \alpha)$-resilient with respect to the Wasserstein-$p$ distance if
\begin{align*}
W_p(\mu, \mu') \leq \rho
\end{align*}
for all $\mu' \in \mathcal{Q}_p(\mathcal{X})$ such that $\mu' \leq \frac{1}{1-\alpha}\mu$. We denote the family of all $(\rho, \alpha)$-resilient distributions in $\mathcal{Q}_p(\mathcal{X})$ by $\mathcal{W}_p(\rho, \alpha)$.
\end{dfn}

The inequality between the two measures, $\mu' \leq \frac{1}{1-\alpha}\mu$, is understood pointwise; that is, it holds for any Borel set $B \subseteq \mathcal{X}$. Intuitively, this condition implies that $\mu'$ is obtained by removing up to an $\alpha$-fraction of the probability mass from $\mu$ and subsequently normalizing the remaining measure.

Resilience characterizes the structural stability of a distribution: it guarantees that the core of the distribution remains close to the original distribution in the Wasserstein metric, even after an adversary discards an $\alpha$-fraction of its mass. Originally introduced to study robust mean estimation under adversarial contamination \citep{steinhardt2018resilience}, this concept has recently been generalized to a distribution-free setting \citep{zhu2022generalized, nietert2023robust}.

\begin{dfn}[Support Total Variation Distance]
For two probability distributions $P_1, P_2 \in \mathcal{Q}(\mathcal{X})$, we define the support total variation (STV) distance as:
\begin{align*}
\|P_1 - P_2\|_{\mathrm{STV}} = \inf \Big\{ \alpha \in [0, 1) : \exists \bar{S} \subseteq \operatorname{supp}(P_1) \cap \operatorname{supp}(P_2) \text{ s.t. } P_1(\bar{S}) = P_2(\bar{S}) = 1-\alpha \text{ and } P_1|_{\bar{S}} = P_2|_{\bar{S}} \Big\},
\end{align*}
where $\operatorname{supp}(\cdot)$ denotes the support of a probability measure, and $P|_{\bar{S}}$ represents the restriction of $P$ to $\bar{S}$, defined as $P|_{\bar{S}}(B) = P(B \cap \bar{S})$ for any Borel set $B \subseteq \mathcal{X}$.
\end{dfn}

The STV distance is a structurally restricted variant of the classical total variation (TV) distance. Recall that if two probability distributions satisfy $\|P_1 - P_2\|_{\mathrm{TV}} \leq \alpha$, then $P_2$ can be viewed as being obtained from $P_1$ by first removing an arbitrary $\alpha$-fraction of the mass from $P_1$ and replacing it with an arbitrary $\alpha$-fraction of mass elsewhere \citep{diakonikolas2019robust}.

Building upon this deletion-addition decomposition, the STV distance enforces the additional constraint that the mass removal and addition must occur on disjoint supports. In other words, if $\|P_1 - P_2\|_{\mathrm{STV}} \leq \alpha$, there exists a shared support $\bar{S}$ containing $1-\alpha$ of the total probability mass, such that $P_2$ is obtained by deleting the $\alpha$-mass on $\operatorname{supp}(P_1) \setminus \bar{S}$ and adding a new $\alpha$-mass restricted to $\mathcal{X} \setminus \bar{S}$. It immediately follows from this definition that $\|P_1 - P_2\|_{\mathrm{TV}} \leq \|P_1 - P_2\|_{\mathrm{STV}}$.

\subsubsection{Population Setting}

To isolate the geometric properties of our approach from finite-sample statistical fluctuations, we first analyze the problem in the population setting (corresponding to the infinite-sample limit).

Let the clean distribution $P^*$ belong to a $(\rho, \alpha)$-resilient family: $P^* \in \mathcal{G} \subseteq \mathcal{W}_1(\rho, \alpha)$, and suppose the observed contaminated distribution $\widetilde{P}$ satisfies $\|\widetilde{P} - P^*\|_{\mathrm{STV}} \leq \alpha$. We define the population-level estimator as:
\begin{align*}
\widehat{P}_{\mathrm{WF}} = \widetilde{P}_{S_{\mathrm{WF}}} := \frac{1}{1-\alpha} \widetilde{P}\vert_{S_{\mathrm{WF}}},
\end{align*}
where the optimal support $S_{\mathrm{WF}}$ is given by:
\begin{align*}
S_{\mathrm{WF}} = \operatorname*{argmax}_{\substack{S \subseteq \operatorname{supp}(\widetilde{P}) \\ \widetilde{P}(S) = 1-\alpha}} W_1(\widetilde{P}_S, \widetilde{P}),
\end{align*}
and $\widetilde{P}_S = \frac{1}{\widetilde{P}(S)} \widetilde{P}\vert_S$ denotes the normalized restriction of $\widetilde{P}$ to $S$. This formulation directly mirrors our empirical estimator, where the discrete support points each carry mass $1/n$.

We now introduce our proposed \emph{Far Exclusion and Local Projection} (FELP) contamination model at the population level:

\begin{dfn}[Population FELP Contamination]
Given $P^* \in \mathcal{W}_1(\rho, \alpha)$, we say that the contaminated distribution $\widetilde{P}$ belongs to the FELP contamination family $\mathcal{C}^{\mathrm{FELP}}_{P^*}(\alpha, \rho)$ if the following three conditions hold:
\begin{enumerate}
    \item {Support TV Bound:} The total variation of the support satisfies:
  %  \begin{align*}
    $
    \|\widetilde{P} - P^*\|_{\mathrm{STV}} \leq \alpha.
    $
  %  \end{align*}
    Let $\bar{S}$ be the shared support of mass $1-\alpha$ promised by the definition of the $\mathrm{STV}$ distance.

    \item {Far Exclusion (FE):} Let $E_\rho := \{ \mathbf{x} \in \operatorname{supp}(\widetilde{P}) \setminus \bar{S} \mid d(\mathbf{x}, \operatorname{supp}(P^*)) > \rho \}$ define the region of ``far outliers". There exists a subset $S_+ \subseteq \operatorname{supp}(\widetilde{P})$ with $\widetilde{P}(S_+) = 1-\alpha$ and $S_+ \cap E_\rho = \emptyset$ such that:
    \begin{align}\label{eq:3}
    \sup_{\substack{S \subseteq \operatorname{supp}(\widetilde{P}), \, \widetilde{P}(S) = 1-\alpha \\ \widetilde{P}(S \cap  E_\rho) > 0}} W_1(\widetilde{P}_S, \widetilde{P}) < W_1(\widetilde{P}_{S_+}, \widetilde{P}).
    \end{align}

    \item {Local Projection (LP):} Let $N_\rho := \{ \mathbf{x} \in \operatorname{supp}(\widetilde{P}) \setminus \bar{S} \mid d(\mathbf{x}, \operatorname{supp}(P^*)) \leq \rho \}$ denote the region of ``near outliers". There exists a measurable projection map $T : N_\rho \to \operatorname{supp}(P^*)$ such that:
    \begin{align}\label{eq:2}
    d(\mathbf{x}, T(\mathbf{x})) \leq \rho \quad \text{and} \quad T_{\sharp} \widetilde{P}\vert_{N_\rho} \leq P^*,
    \end{align}
    where $T_{\sharp}$ denotes the push-forward operator.
\end{enumerate}
\end{dfn}

In the FELP contamination model, the STV distance restricts the corruptive mass to support-wise perturbations bound by a budget of $\alpha$. The far exclusion and local projection conditions further dictate the precise geometric structure of the contamination under which our estimator achieves minimax optimality:
\begin{itemize}
    \item {Far Outliers ($E_\rho$):} Must be sufficiently far from the clean support that they are guaranteed to be excluded by our Wasserstein-maximization estimator.
    \item {Near Outliers ($N_\rho$):} Lie close to the clean support. Their impact on the distribution estimation error is negligible even if they are not removed, as they can be modeled as local perturbations from the clean distribution via the mass-non-increasing preimage map $T^{-1}$.
\end{itemize}

By definition, we have the following nested relationship among the contamination families:
\begin{align*}
\mathcal{C}_{P^*}^{\mathrm{FELP}}(\alpha, \rho) \subseteq \big\{ P \in \mathcal{Q}(\mathcal{X}) : \|P - P^*\|_{\mathrm{STV}} \leq \alpha \big\} \subseteq \big\{ P \in \mathcal{Q}(\mathcal{X}) : \|P - P^*\|_{\mathrm{TV}} \leq \alpha \big\}.
\end{align*}
This nesting shows that the FELP model is structurally more restrictive than the classical population total variation (TV) contamination model. Nevertheless, FELP remains sufficiently rich to capture realistic scenarios. Indeed, adversarial contamination naturally decomposes in this manner: one component consists of far, well-separated outliers, while the other consists of near, indistinguishable inliers. This far-and-local decomposition sharing a similar spirit with the local perturbation and global adversary contamination models of \cite{nietert2024robust}, as well as the mild and gross contamination frameworks of \cite{clark2024finding}.

In what follows, we provide a concrete, low-dimensional example to demonstrate when the FELP model conditions are satisfied.

\bigskip

\noindent\textbf{Example (Uniform Cube in the Population Setting).}
Let $S^* = [0,1]^d$ and consider the clean distribution $P^* = \operatorname{Unif}(S^*)$. We corrupt the data via a split-and-replace mechanism that deletes clean subregions in $S^*$ and replaces them with corrupted mass:
\begin{align*}
S_{\mathrm{loc}}^{\mathrm{in}} = [0,l] \times [0,1]^{d-1}
&\quad\to\quad
S^{\mathrm{out}}_{\mathrm{loc}} = [1,1+l] \times [0,1]^{d-1}, \\
S_{\mathrm{far}}^{\mathrm{in}} = [l,2l] \times [0,1]^{d-1}
&\quad\to\quad
\mathbf{x}_{\mathrm{far}}^{\mathrm{out}} = (L, 1, \dots, 1)^\top,
\end{align*}
where $0 < \alpha = 2l < \frac{1}{2}$, and $L > (1+\sqrt{d}) \lor C_{l,d}$ for a constant $C_{l,d}$ depending on $d$ and $l$. We then define the contaminated probability measure as:
\begin{align*}
\widetilde{P} = P^*\vert_{S^* \setminus (S_{\mathrm{loc}}^{\mathrm{in}} \cup S_{\mathrm{far}}^{\mathrm{in}})} + l \operatorname{Unif}(S_{\mathrm{loc}}^{\mathrm{out}}) + l \delta_{\mathbf{x}_{\mathrm{far}}^{\mathrm{out}}}.
\end{align*}
Under this construction, we have $\widetilde{P} \in \mathcal{C}_{P^*}^{\mathrm{FELP}}(\alpha, \sqrt{d})$.

\begin{proof}[Proof sketch]
By construction, $\bar{S} = S^* \setminus (S_{\mathrm{loc}}^{\mathrm{in}} \cup S_{\mathrm{far}}^{\mathrm{in}})$ is the valid shared support between $\widetilde{P}$ and $P^*$, yielding $\|\widetilde{P} - P^*\|_{\mathrm{STV}} \leq \alpha$.

For any $P' \in \mathcal{Q}_1(\mathcal{X})$ satisfying $P' \leq \frac{1}{1-\alpha}P^*$, we have:
\begin{align*}
W_1(P', P^*) \leq \operatorname{diam}(S^*) \|P' - P^*\|_{\mathrm{TV}} \leq \alpha \sqrt{d} \leq \sqrt{d},
\end{align*}
which implies $P^* \in \mathcal{W}_1(\sqrt{d}, \alpha)$.

Setting $\rho = \sqrt{d}$, and since $L > 1 + \sqrt{d}$, we identify the near and far outlier regions as $N_\rho = S_{\mathrm{loc}}^{\mathrm{out}}$ and $E_\rho = \{\mathbf{x}_{\mathrm{far}}^{\mathrm{out}}\}$, respectively. For any $\mathbf{x} \in S_{\mathrm{loc}}^{\mathrm{out}}$, we define the projection map $T(\mathbf{x}) = \mathbf{x} - \mathbf{e}_1$, where $\mathbf{e}_1 = (1, 0, \dots, 0)^\top$. This yields:
\begin{align*}
d(\mathbf{x}, T(\mathbf{x})) = 1 \leq \sqrt{d} = \rho \quad \text{and} \quad T_{\sharp}\widetilde{P}\vert_{N_\rho} = T_{\sharp} \big( l \operatorname{Unif}(S_{\mathrm{loc}}^{\mathrm{out}}) \big) = l \operatorname{Unif}(S_{\mathrm{loc}}^{\mathrm{in}}) \leq P^*,
\end{align*}
satisfying the local projection condition.

Next, choosing $S_+ = S^* \setminus (S_{\mathrm{loc}}^{\mathrm{in}} \cup S_{\mathrm{far}}^{\mathrm{in}})$ in the FELP definition, the Wasserstein distance $W_1(\widetilde{P}_{S_+}, \widetilde{P})$ is lower bounded by the cost of transporting the mass $l$ at the atom $\mathbf{x}_{\mathrm{far}}^{\mathrm{out}}$ to $S_+$:
\begin{align*}
W_1(\widetilde{P}_{S_+}, \widetilde{P}) \geq l \cdot d(\mathbf{x}_{\mathrm{far}}^{\mathrm{out}}, S_+) = l(L-1).
\end{align*}
Conversely, for any admissible support $S \subseteq \operatorname{supp}(\widetilde{P})$ with $\widetilde{P}(S) = 1-\alpha$ and $\widetilde{P}(S \cap E_\rho) > 0$, the set $S$ must contain the atom $\mathbf{x}_{\mathrm{far}}^{\mathrm{out}}$, meaning $\widetilde{P}_S(\{\mathbf{x}_{\mathrm{far}}^{\mathrm{out}}\}) = \frac{l}{1-\alpha}$. To upper bound $W_1(\widetilde{P}_S, \widetilde{P})$, we construct a feasible transport plan from $\widetilde{P}_S$ to $\widetilde{P}$ as follows:
\begin{enumerate}
    \item Keep $l$ mass at the atom $\mathbf{x}_{\mathrm{far}}^{\mathrm{out}}$ in place (incurring zero cost).
    \item Transport the excess mass $\frac{\alpha l}{1-\alpha}$ at $\mathbf{x}_{\mathrm{far}}^{\mathrm{out}}$ to the region $S_+ \cup S_{\mathrm{loc}}^{\mathrm{out}} = [2l, 1+l] \times [0,1]^{d-1}$, incurring a cost of at most $\frac{\alpha l}{1-\alpha} \sqrt{(L-2l)^2 + d-1}$.
    \item Transport the remaining mass $1 - \frac{l}{1-\alpha}$ within $S_+ \cup S_{\mathrm{loc}}^{\mathrm{out}}$, incurring a cost of at most $\big(1 - \frac{l}{1-\alpha}\big) \sqrt{(1-l)^2 + d-1}$.
\end{enumerate}
Since this bound holds for any such $S$, we obtain:
\begin{align*}
\sup_{\substack{S \subseteq \operatorname{supp}(\widetilde{P}), \, \widetilde{P}(S) = 1-\alpha \\ \widetilde{P}(S \cap E_\rho) > 0}} W_1(\widetilde{P}_S, \widetilde{P}) \leq \frac{\alpha l}{1-\alpha}\sqrt{(L-2l)^2+d-1} + \left(1-\frac{l}{1-\alpha}\right)\sqrt{(1-l)^2+d-1}.
\end{align*}
Choosing $L > C_{l,d}$ with
\begin{gather*}
C_{l,d} = \frac{a_l - 2l c_l^2 + c_l\sqrt{(a_l - 2l)^2 + (1 - c_l^2)(d-1)}}{1 - c_l^2}, \\
a_l = 1 + \frac{1-3l}{l(1-2l)}\sqrt{(1-l)^2 + d-1}, \qquad c_l = \frac{2l}{1-2l}
\end{gather*}
ensures that the far exclusion condition \eqref{eq:3} is satisfied.
\end{proof}

We establish the error bounds for the population estimator $\widehat{P}_{\mathrm{WF}}$ under the FELP contamination model as follows:

\begin{thm}[Population Risk of $\widehat{P}_{\mathrm{WF}}$]\label{thm:population_risk}
For any clean distribution $P^* \in \mathcal{G} \subseteq \mathcal{W}_1(\rho,\alpha)$ and contaminated distribution $\widetilde{P} \in \mathcal{C}^{\mathrm{FELP}}_{P^*}(\alpha,\rho)$ with $0 \leq \alpha \leq \frac{1}{2}$, the population estimator satisfies:
\begin{align*}
W_1(\widehat{P}_{\mathrm{WF}}, P^*) \lesssim \rho.
\end{align*}
\end{thm}

The core intuition behind Theorem \ref{thm:population_risk} is as follows. The STV distance condition in the definition of $\mathcal{C}^{\mathrm{FELP}}_{P^*}(\alpha,\rho)$ guarantees the existence of a shared support $\bar{S}$ of mass $1-\alpha$, which yields a reference trimmed distribution $\widetilde{P}_{\bar{S}}$. Note that $\widetilde{P}_{\bar{S}}$ is a feasible $(1-\alpha)$-truncated candidate of $\widetilde{P}$ that satisfies $\widetilde{P}_{\bar{S}} = P^*_{\bar{S}}$. By the $W_1$-resilience of $P^*$, it immediately follows that this reference distribution has a bounded risk: $W_1(\widetilde{P}_{\bar{S}}, P^*) \leq \rho$. Consequently, the optimal estimator $\widehat{P}_{\mathrm{WF}}$, which maximizes the Wasserstein distance to the contaminated distribution, is guaranteed to perform at least as well as this reference candidate. The FELP conditions then ensure that $\widehat{P}_{\mathrm{WF}}$ indeed identifies a support that achieves this optimal risk bound of $\mathcal{O}(\rho)$.

We next establish a minimax lower bound showing that this $\mathcal{O}(\rho)$ rate is optimal. Let $\mathcal{T}: \mathcal{Q}(\mathcal{X}) \to \mathcal{Q}(\mathcal{X})$ denote any population-level estimator, and define its maximum risk over the FELP contamination family as:
\begin{align*}
\mathcal{R}_{1,P^*,\mathcal{T}}(\alpha,\rho) = \sup_{\widetilde{P} \in \mathcal{C}^{\mathrm{FELP}}_{P^*}(\alpha, \rho)} W_1(\mathcal{T}(\widetilde{P}), P^*).
\end{align*}

We have the following minimax optimality result:

\begin{thm}[Minimax Optimal Population Risk]\label{thm:minimax_optimal}
Let $\mathcal{G}_{\mathrm{cov}}(\sigma) \subseteq \mathcal{W}_1(\rho_\sigma,\alpha)$ be the class of distributions with bounded covariance:
%\begin{align*}
$
\mathcal{G}_{\mathrm{cov}}(\sigma) = \left\{ \mu \in \mathcal{Q}(\mathcal{X}) : \mathbf{\Sigma}_\mu \preceq \sigma^2 \mathbf{I} \right\}.
$
%\end{align*}
If $0 < \alpha_0 \leq \alpha \leq \frac{1}{2}$ and the dimension satisfies $d \lesssim \frac{1}{\alpha(1-\alpha^2)}$, then:
\begin{align*}
\inf_{\mathcal{T}} \sup_{P^* \in \mathcal{G}_{\mathrm{cov}}(\sigma)} \mathcal{R}_{1,P^*,\mathcal{T}}(\alpha,\rho_\sigma) \asymp \rho_\sigma,
\end{align*}
where $\rho_\sigma \asymp \sigma\sqrt{d\alpha}$ is the sharp resilience factor for the family $\mathcal{G}_{\mathrm{cov}}(\sigma)$.
\end{thm}

In the case where the clean distribution class is restricted to bounded-covariance families, i.e., $\mathcal{G} = \mathcal{G}_{\mathrm{cov}}(\sigma)$, Theorem \ref{thm:population_risk} holds uniformly for any $P^* \in \mathcal{G}_{\mathrm{cov}}(\sigma)$ and $\widetilde{P} \in \mathcal{C}^{\mathrm{FELP}}_{P^*}(\alpha,\rho_\sigma)$. This immediately yields the upper bound:
\begin{align*}
\sup_{P^* \in \mathcal{G}_{\mathrm{cov}}(\sigma)} \mathcal{R}_{1,P^*,\widehat{P}_{\mathrm{WF}}}(\alpha,\rho_\sigma) \lesssim \rho_\sigma.
\end{align*}
In conjunction with the lower bound in Theorem \ref{thm:minimax_optimal}, this result establishes that our population estimator $\widehat{P}_{\mathrm{WF}}$ achieves the minimax optimal risk.

The choice of $\mathcal{G}_{\mathrm{cov}}(\sigma)$ is representative of standard robust statistics settings and imposes minimal assumptions on the clean distribution family $\mathcal{G}$ containing $P^*$. Indeed, our minimax optimality analysis can be readily extended to other standard distribution classes—such as sub-Gaussian families or families with bounded $q$-th moments ($q \geq 1$)—using analogous arguments based on their respective resilience properties \citep{nietert2023robust}.

Furthermore, the conditions on the model parameters are highly practical:
\begin{itemize}
    \item {Contamination Level ($\alpha$):} The condition $\alpha \le \frac{1}{2}$ is mild, indicating that our framework can tolerate a corruption level of up to $50\%$ of the data mass.
    \item {Dimensionality ($d$):} The dimension condition is quite natural, allowing the allowable dimension to grow as the contamination fraction $\alpha$ decreases.
\end{itemize}

\subsubsection{Empirical Setting}

In this section, we extend our analysis to the finite-sample regime by incorporating stochastic sampling errors. We first introduce the empirical counterpart of our proposed contamination model.

\begin{dfn}[Empirical FELP Contamination]
Let $\mathbf{x}_1, \dots, \mathbf{x}_n \overset{\mathrm{i.i.d.}}{\sim} P^* \in \mathcal{W}_1(\rho, \alpha)$ be clean samples with empirical measure $P_n = \frac{1}{n}\sum_{i=1}^n \delta_{\mathbf{x}_i}$, and let $\widetilde{\mathbf{x}}_1, \dots, \widetilde{\mathbf{x}}_n$ denote the observed contaminated samples. We say that the contaminated empirical measure $\widetilde{P}_n = \frac{1}{n}\sum_{i=1}^n \delta_{\widetilde{\mathbf{x}}_i}$ belongs to the empirical FELP family $\mathcal{C}^{\mathrm{FELP}}_{P_n}(\alpha, \rho)$ if the following three conditions hold:
\begin{enumerate}
    \item {Empirical Support TV Bound:} The fraction of corrupted samples is bounded by $\alpha$:
 %   \begin{align*}
 $
    \frac{1}{n} \sum_{i=1}^n \mathbb{I}_{\{\mathbf{x}_i \neq \widetilde{\mathbf{x}}_i\}} \leq \alpha.
 $
%    \end{align*}
    Let $\bar{S} \subseteq \{\mathbf{x}_j \mid j \in [n], \, \mathbf{x}_j = \widetilde{\mathbf{x}}_j\}$ with $|\bar{S}| = (1-\alpha)n$ be the uncorrupted shared support.

    \item {Empirical Far Exclusion (FE):} Let $E_{n,\rho} := \{ \mathbf{x} \in \operatorname{supp}(\widetilde{P}_n) \setminus \bar{S} \mid d(\mathbf{x}, \operatorname{supp}(P_n)) > \rho \}$ define the region of ``empirical far outliers". There exists a subset $S_+ \subseteq \operatorname{supp}(\widetilde{P}_n)$ with $\widetilde{P}_n(S_+) = 1-\alpha$ and $S_+ \cap E_{n,\rho} = \emptyset$ such that:
    \begin{align}\label{eq:emp_fe}
    \sup_{\substack{S \subseteq \operatorname{supp}(\widetilde{P}_n), \, \widetilde{P}_n(S) = 1-\alpha \\ \widetilde{P}_n(S \cap E_{n,\rho}) > 0}} W_1(\widetilde{P}_{n, S}, \widetilde{P}_n) < W_1(\widetilde{P}_{n, S_+}, \widetilde{P}_n),
    \end{align}
    where $\widetilde{P}_{n, S} = \frac{1}{\widetilde{P}_n(S)} \widetilde{P}_n\vert_S$ is the normalized restriction of $\widetilde{P}_n$ to $S$.

    \item {Empirical Local Projection (LP):} Let $N_{n,\rho} := \{ \mathbf{x} \in \operatorname{supp}(\widetilde{P}_n) \setminus \bar{S} \mid d(\mathbf{x}, \operatorname{supp}(P_n)) \leq \rho \}$ denote the region of  ``empirical near outliers". There exists a projection map $T_n : N_{n,\rho} \to \operatorname{supp}(P_n)$ such that:
    \begin{align}\label{eq:emp_lp}
    d(\mathbf{x}, T_n(\mathbf{x})) \leq \rho \quad \text{and} \quad T_{n\sharp} \widetilde{P}_n\vert_{N_{n,\rho}} \leq P_n,
    \end{align}
    where $T_{n\sharp}$ denotes the push-forward operator.
\end{enumerate}
\end{dfn}

Notably, our empirical FELP contamination model is defined directly on the realized sample rather than treating contaminated points as i.i.d. samples from a pre-specified population-level mixture $\widetilde{P} \in \mathcal{C}_{P^*}^{\mathrm{FELP}}(\alpha, \rho)$. The latter framework is commonly referred to as the \emph{oblivious contamination model} \citep{diakonikolas2019robust, zhu2022generalized}. Formulating contamination in terms of the realized sample is significantly more realistic because the underlying population distribution is typically unknown in practice.

Furthermore, empirical FELP aligns more closely with the classical TV contamination model discussed in Section 1, encompassing a broader class of adversarial corruptions than the oblivious model.

By construction, any empirical FELP contaminated measure satisfies $\|\widetilde{P}_n - P_n\|_{\mathrm{STV}} \leq \alpha$. Let $\mathbb{P}_n$ denote the joint distribution of the i.i.d. clean samples $(\mathbf{x}_1, \dots, \mathbf{x}_n) \in \mathcal{X}^n$, and let $\mathcal{D}_n^{\mathrm{FELP}}(P^*, \alpha)$ represent the family of all joint distributions $\widetilde{\mathbb{P}}_n \in \mathcal{Q}(\mathcal{X}^n)$ of the contaminated samples $(\widetilde{\mathbf{x}}_1, \dots, \widetilde{\mathbf{x}}_n)$ generated via empirical FELP contamination.

We establish the statistical convergence of our empirical estimator $\widehat{P}_{\mathrm{WF}}$ (introduced in Section 2.1) under this finite-sample model below:

\begin{thm}[Empirical Risk of $\widehat{P}_{\mathrm{WF}}$]\label{thm:empirical_risk}
Let $P^* \in \mathcal{G} \subseteq \mathcal{W}_1(\rho, \alpha)$ and let $0 \leq \alpha \leq \frac{1}{2}$. For any empirical contamination $\widetilde{P}_n \in \mathcal{C}^{\mathrm{FELP}}_{P_n}(\alpha, \rho)$ (i.e., $\widetilde{\mathbb{P}}_n \in \mathcal{D}_n^{\mathrm{FELP}}(P^*, \alpha)$), the empirical estimator satisfies:
\begin{align*}
\mathbb{E}_{\widetilde{\mathbb{P}}_n} \left[ W_1(\widehat{P}_{\mathrm{WF}}, P^*) \right] \lesssim \rho + \mathbb{E}_{\mathbb{P}_n} \left[ W_1(P_n, P^*) \right].
\end{align*}
\end{thm}

Similar to Theorem \ref{thm:population_risk}, the sample corruption constraint
%\begin{align*}
$
\frac{1}{n}\sum_{i=1}^n \mathbb{I}_{\{\widetilde{\mathbf{x}}_i \neq \mathbf{x}_i\}} \leq \alpha
$
%\end{align*}
implies that $\|\widetilde{P}_n - P_n\|_{\mathrm{STV}} \leq \alpha$. This guarantees the existence of a reference bridge estimator $\widetilde{P}_{n, \bar{S}}$, where $\bar{S} \subseteq \{\mathbf{x}_j \mid j \in [n], \, \mathbf{x}_j = \widetilde{\mathbf{x}}_j\}$ with $|\bar{S}| = (1-\alpha)n$. The expected $W_1$-risk of this reference estimator is bounded by $\rho + \mathbb{E}_{\mathbb{P}_n}[W_1(P_n, P^*)]$. Since $\widetilde{P}_{n, \bar{S}}$ discards all contaminated points and retains only clean samples, it represents the statistical benchmark (i.e., the best possible performance) for our filtering framework. Theorem \ref{thm:empirical_risk} establishes that under empirical FELP contamination, the Wasserstein Filtering (WF) estimator indeed matches this optimal benchmark.

To rigorously prove that the WF estimator is minimax optimal in the empirical setting, we introduce a mild condition ensuring that the Far Exclusion property holds with high probability under finite-sample realizations:

\begin{con}[Far Exclusion with High Probability]\label{con:1}
For every clean distribution $P^* \in \mathcal{G}$ and every population contamination $\widetilde{P} \in \mathcal{C}^{\mathrm{FELP}}_{P^*}\left(\frac{\alpha}{5}, \rho\right)$, the following holds. Let $\eta := P^*\vert_{\bar{S}} = \widetilde{P}\vert_{\bar{S}}$ represent the shared clean component, and decompose the outliers as $\nu_{\mathrm{loc}} := \widetilde{P}\vert_{N_\rho}$ and $\nu_{\mathrm{far}} := \widetilde{P}\vert_{E_\rho}$. Further, let $\lambda_{\mathrm{loc}} := T_{\sharp}\nu_{\mathrm{loc}}$ denote the projection of local outliers, and define the remaining clean mass as $\lambda_{\mathrm{far}} := P^*\vert_{\bar{S}^c} - \lambda_{\mathrm{loc}} \geq 0$.

For any coupling $\pi_{\mathrm{far}} \in \Pi(\lambda_{\mathrm{far}}, \nu_{\mathrm{far}})$, we construct a joint coupling $\pi \in \Pi(P^*, \widetilde{P})$ defined as:
\begin{align*}
\pi = (\mathrm{Id}, \mathrm{Id})_{\sharp}\eta + (T, \mathrm{Id})_{\sharp}\nu_{\mathrm{loc}} + \pi_{\mathrm{far}},
\end{align*}
where $\mathrm{Id}$ is the identity map. Under the product measure $(\mathbf{x}_i, \widetilde{\mathbf{x}}_i)_{i=1}^n \sim \pi^n$, we define the successful exclusion event $\mathcal{E}$ as:
\begin{align*}
\mathcal{E} =
\left\{
\begin{aligned}
  &\exists S_+ \subseteq \{\widetilde{\mathbf{x}}_1, \dots, \widetilde{\mathbf{x}}_n\}, \ |S_+|=(1-\alpha)n, \ S_+ \cap E_{n,\rho} = \emptyset \quad \text{such that} \\
  &\sup_{\substack{S \subseteq \{\widetilde{\mathbf{x}}_1, \dots, \widetilde{\mathbf{x}}_n\} \\ |S|=(1-\alpha)n, \ S \cap E_{n,\rho} \neq \emptyset}} W_1\left(\frac{1}{|S|}\sum_{\widetilde{\mathbf{x}}_i \in S} \delta_{\widetilde{\mathbf{x}}_i}, \, \widetilde{P}_n\right) < W_1\left(\frac{1}{|S_+|}\sum_{\widetilde{\mathbf{x}}_i \in S_+} \delta_{\widetilde{\mathbf{x}}_i}, \, \widetilde{P}_n\right)
\end{aligned}
\right\}.
\end{align*}
We assume there exists a sequence $\delta_n \leq \frac{1}{20}$ such that $\pi^n(\mathcal{E}^c) \leq \delta_n$.
\end{con}

Condition \ref{con:1} defines a canonical coupling $\pi$ between the population distribution pair $(P^*, \widetilde{P})$ under the population-level FELP contamination model. It asserts that under this coupling, the generated contaminated samples $\{\widetilde{\mathbf{x}}_i\}_{i=1}^n$ satisfy the empirical far exclusion event $\mathcal{E}$—which is the central requirement of the empirical FELP model—with high probability.

By viewing the observed samples $\{\widetilde{\mathbf{x}}_i\}_{i=1}^n$ as being generated from the clean samples $\{\mathbf{x}_i\}_{i=1}^n$ via this coupled process, Condition \ref{con:1} provides a rigorous bridge from the population setting to the empirical setting. Specifically, for any valid population pair $(P^*, \widetilde{P})$, we can construct a sample sequence $(\widetilde{\mathbf{x}}_1, \dots, \widetilde{\mathbf{x}}_n)$ such that their joint distribution $\widetilde{\mathbb{P}}_n$ belongs to the empirical FELP family $\mathcal{D}_n^{\mathrm{FELP}}(P^*, \alpha)$. This construction implies that the empirical FELP contamination family is richer than its population counterpart. Consequently, the minimax risk under empirical FELP contamination is naturally lower bounded by the corresponding population-level minimax risk.

In what follows, we demonstrate that Condition \ref{con:1} is satisfied in the same uniform cube setting introduced in our population-level example.

\noindent\textbf{Example (Uniform Cube in the Empirical Setting).}
Consider the $d$-dimensional uniform cube $S^* = [0,1]^d$ with the clean distribution $P^* = \operatorname{Unif}(S^*)$. Let $S_{\mathrm{loc}}^{\mathrm{in}}$, $S_{\mathrm{far}}^{\mathrm{in}}$, $S_{\mathrm{loc}}^{\mathrm{out}}$, and $\mathbf{x}_{\mathrm{far}}^{\mathrm{out}}$ be defined exactly as in the population setting (with $l = \frac{\alpha}{10}$), yielding the population contaminated distribution:
\begin{align*}
\widetilde{P} = P^*\vert_{S^* \setminus (S_{\mathrm{loc}}^{\mathrm{in}} \cup S_{\mathrm{far}}^{\mathrm{in}})} + l \operatorname{Unif}(S_{\mathrm{loc}}^{\mathrm{out}}) + l \delta_{\mathbf{x}_{\mathrm{far}}^{\mathrm{out}}}.
\end{align*}
From the population analysis, we have $\widetilde{P} \in \mathcal{C}^{\mathrm{FELP}}_{P^*}\left(\frac{\alpha}{5}, \sqrt{d}\right)$ with the projection map $T(\mathbf{x}) = \mathbf{x} - \mathbf{e}_1$ and the shared support $\bar{S} = S^* \setminus (S_{\mathrm{loc}}^{\mathrm{in}} \cup S_{\mathrm{far}}^{\mathrm{in}})$.

Let $\mathbf{x}_1, \dots, \mathbf{x}_n \overset{\mathrm{i.i.d.}}{\sim} P^*$. We construct the empirical contaminated samples $\widetilde{\mathbf{x}}_i$ for $i \in [n]$ as:
\begin{align*}
\widetilde{\mathbf{x}}_i =
\begin{cases}
\mathbf{x}_i, & \mathbf{x}_i \notin S_{\mathrm{loc}}^{\mathrm{in}} \cup S_{\mathrm{far}}^{\mathrm{in}}, \\
\mathbf{x}_i + \mathbf{e}_1, & \mathbf{x}_i \in S_{\mathrm{loc}}^{\mathrm{in}}, \\
\mathbf{x}_{\mathrm{far}}^{\mathrm{out}}, & \mathbf{x}_i \in S_{\mathrm{far}}^{\mathrm{in}}.
\end{cases}
\end{align*}
Then, the coupled samples satisfy $(\mathbf{x}_i, \widetilde{\mathbf{x}}_i)_{i=1}^n \sim \pi^n$ with $\pi = (\mathrm{Id}, \mathrm{Id})_{\sharp}\eta + (T, \mathrm{Id})_{\sharp}\nu_{\mathrm{loc}} + \pi_{\mathrm{far}}$, and we have $\pi^n(\mathcal{E}^c) \leq \frac{1}{20}$ for sufficiently large $n$.

\begin{proof}[Proof sketch]
Let $M_{\mathrm{loc}} = \{\widetilde{\mathbf{x}}_i \in S_{\mathrm{loc}}^{\mathrm{out}} \mid i \in [n]\}$ and $M_{\mathrm{far}} = \{\widetilde{\mathbf{x}}_i = \mathbf{x}_{\mathrm{far}}^{\mathrm{out}} \mid i \in [n]\}$. Applying Hoeffding's inequality, we define the high-probability concentration event $\mathcal{A}_\epsilon$ as:
\begin{align*}
\mathbb{P}\{\mathcal{A}_\epsilon\} := \mathbb{P}\left\{ \left| \frac{|M_{\mathrm{loc}}|}{n} - l \right| \leq \epsilon, \ \left| \frac{|M_{\mathrm{far}}|}{n} - l \right| \leq \epsilon \right\} \geq 1 - 4e^{-2\epsilon^2 n}.
\end{align*}
For any $\widetilde{\mathbf{x}}_i \in M_{\mathrm{loc}}$, we have $d(\widetilde{\mathbf{x}}_i, \operatorname{supp}(P_n)) \leq d(\widetilde{\mathbf{x}}_i, T(\widetilde{\mathbf{x}}_i)) = 1 \leq \sqrt{d} = \rho$, which confirms $M_{\mathrm{loc}} \subseteq N_{n,\rho}$. Under the assumption $L > \sqrt{d} + 1$, any $\widetilde{\mathbf{x}}_i \in M_{\mathrm{far}}$ satisfies $d(\widetilde{\mathbf{x}}_i, \operatorname{supp}(P_n)) \geq d(\widetilde{\mathbf{x}}_i, \operatorname{supp}(P^*)) = L-1 > \sqrt{d} = \rho$, confirming $M_{\mathrm{far}} \subseteq E_{n,\rho}$. Since $M_{\mathrm{loc}} \cup M_{\mathrm{far}} = N_{n,\rho} \cup E_{n,\rho}$ and $N_{n,\rho} \cap E_{n,\rho} = \emptyset$, we obtain the partition $M_{\mathrm{loc}} = N_{n,\rho}$ and $M_{\mathrm{far}} = E_{n,\rho}$.

Next, we construct a valid uncorrupted empirical support $S_+$ of size $(1-\alpha)n$ by removing $M_{\mathrm{far}}$, $M_{\mathrm{loc}}$, and any remaining $\alpha n - |M_{\mathrm{far}}| - |M_{\mathrm{loc}}|$ samples from the set of clean inliers $\{\widetilde{\mathbf{x}}_i \mid \widetilde{\mathbf{x}}_i = \mathbf{x}_i\}$. This construction is feasible on the event $\mathcal{A}_\epsilon$ because:
%\begin{align*}
$
|M_{\mathrm{far}}| + |M_{\mathrm{loc}}| \leq 2(l + \epsilon)n = \left(\frac{\alpha}{5} + 2\epsilon\right)n < \alpha n,
$
%\end{align*}
provided that $\epsilon < \frac{2}{5}\alpha$. Consequently, $S_+ \cap E_{n,\rho} = \emptyset$. Since $\widetilde{P}_n$ places a mass of $\frac{|M_{\mathrm{far}}|}{n}$ on the atom $\mathbf{x}_{\mathrm{far}}^{\mathrm{out}}$, we establish the lower bound:
\begin{align*}
W_1(\widetilde{P}_{nS_+}, \widetilde{P}_n) > \frac{|M_{\mathrm{far}}|}{n}(L-1).
\end{align*}
Conversely, let $S$ be any corrupted subset of size $(1-\alpha)n$ that contains at least one far outlier. Employing the same transportation plan coupling logic as in the population case, we obtain:
\begin{align*}
\sup_{\substack{S \subseteq \operatorname{supp}(\widetilde{P}_n) \\ \widetilde{P}_n(S) = 1-\alpha \\ \widetilde{P}_n(S \cap E_{n,\rho}) > 0}} W_1(\widetilde{P}_{nS}, \widetilde{P}_n) \leq \sup_{1 \leq |S \cap M_{\mathrm{far}}| \leq |M_{\mathrm{far}}|} \left| \frac{|M_{\mathrm{far}}|}{n} - \frac{|S \cap M_{\mathrm{far}}|}{n(1-\alpha)} \right| \sqrt{(L-2l)^2 + d - 1} + \sqrt{(1-l)^2 + d - 1}.
\end{align*}
Therefore, the far exclusion event $\mathcal{E}$ is guaranteed to hold as long as:
\begin{align*}
\frac{|M_{\mathrm{far}}|}{n}(L-1) > \sup_{1 \leq |S \cap M_{\mathrm{far}}| \leq |M_{\mathrm{far}}|} \left| \frac{|M_{\mathrm{far}}|}{n} - \frac{|S \cap M_{\mathrm{far}}|}{n(1-\alpha)} \right| \sqrt{(L-2l)^2 + d - 1} + \sqrt{(1-l)^2 + d - 1}.
\end{align*}
On the concentration event $\mathcal{A}_\epsilon$ with $\epsilon < l$, this inequality is satisfied by choosing:
\begin{gather*}
L > \max\left\{\sqrt{d}+1, \, C_{l,d}, \, 2l, \, \frac{l+\epsilon + (l+\epsilon-\Delta_{\epsilon,n,\alpha})\sqrt{d-1} + \sqrt{(1-l)^2+d-1}}{\Delta_{\epsilon,n,\alpha}}\right\}, \\
\text{where} \quad \Delta_{\epsilon,n,\alpha} = \min\left\{\frac{1}{(1-\alpha)n}, \, \frac{(l-\epsilon)(1-2\alpha)}{1-\alpha}\right\}.
\end{gather*}
Asymptotically, this requires $L = \Omega(n\sqrt{d})$. This ensures $\mathcal{A}_\epsilon \subseteq \mathcal{E}$, meaning $\pi^n(\mathcal{E}^c) \leq \pi^n(\mathcal{A}_\epsilon^c) \leq 4e^{-2\epsilon^2 n} \leq \frac{1}{20}$ for a sufficiently large sample size $n$.
\end{proof}

We consider the finite-sample empirical risk:
\begin{align*}
\mathcal{R}_{1,P^*,\mathcal{T}_n,n}(\alpha,\rho) = \sup_{\widetilde{\mathbb{P}}_n \in \mathcal{D}_n^{\mathrm{FELP}}(P^*,\alpha)} \mathbb{E}_{\widetilde{\mathbb{P}}_n} \!\left[ W_1(\mathcal{T}_n(\widetilde{P}_n), P^*) \right]\!,
\end{align*}
where $\mathcal{T}_n$ is an empirical distribution estimator, with the subscript $n$ distinguishing it from the population-level estimator $\mathcal{T}$. The corresponding empirical minimax risk is defined as:
\begin{align*}
\mathcal{R}_{1,n}(\alpha,\rho) = \inf_{\mathcal{T}_n} \sup_{P^* \in \mathcal{G}} \mathcal{R}_{1,P^*,\mathcal{T}_n,n}(\alpha,\rho).
\end{align*}

We establish the minimax optimality of the empirical risk under the FELP model as follows:

\begin{thm}[Minimax Optimality of Empirical Risk]\label{thm:empirical_minimax}
Let $P^* \in \mathcal{G} \subseteq \mathcal{W}_1(\rho, \alpha)$ with $0 \leq \alpha \leq \frac{1}{2}$. Under Condition \ref{con:1}, the empirical minimax risk satisfies:
\begin{align*}
\mathcal{R}_{1,\infty}\left(\frac{\alpha}{5}, \rho\right) + \inf_{\mathcal{T}_n}\sup_{P^*\in\mathcal{G}}\mathbb{E}\left[W_1(\mathcal{T}_n(P_n),P^*)\right]
\lesssim
\inf_{\mathcal{T}_n}\sup_{P^*\in\mathcal{G}}\mathcal{R}_{1,P^*,\mathcal{T}_n,n}(\alpha,\rho)
\lesssim
\rho + \sup_{P^*\in\mathcal{G}}\mathbb{E}\left[W_1(P_n,P^*)\right],
\end{align*}
where the population-level minimax limit is given by:
\begin{align*}
\mathcal{R}_{1,\infty}\left(\frac{\alpha}{5}, \rho\right) = \inf_{\mathcal{T}}\sup_{P^*\in\mathcal{G}}\sup_{\widetilde{P}\in \mathcal{C}^{\mathrm{FELP}}_{P^*}\left(\frac{\alpha}{5},\rho\right)} W_1(\mathcal{T}(\widetilde{P}), P^*).
\end{align*}
In particular, when $\mathcal{G}$ is the family of distributions with bounded covariance, $\mathcal{G}_{\mathrm{cov}}(\sigma) = \{\mu \in \mathcal{Q}(\mathcal{X}) : \mathbf{\Sigma}_{\mu} \preceq \sigma^2 \mathbf{I}\}$, and assuming $0 < \alpha_0 \leq \alpha \leq \frac{1}{2}$ alongside the dimension constraint $d \lesssim \frac{1}{\alpha(1-\alpha^2)}$, we have:
\begin{align*}
\inf_{\mathcal{T}_n}\sup_{P^*\in\mathcal{G}_{\mathrm{cov}}(\sigma)}\mathcal{R}_{1,P^*,\mathcal{T}_n,n}(\alpha,\rho_\sigma)
\asymp
\sigma\sqrt{d\alpha} + \sigma\sqrt{d} n^{-\frac{1}{d}}.
\end{align*}
Furthermore, the Wasserstein Filtering estimator $\widehat{P}_{\mathrm{WF}}$ achieves this optimal empirical minimax rate.
\end{thm}

The first part of Theorem \ref{thm:empirical_minimax} establishes that the empirical minimax risk is lower bounded by the sum of the population-level risk $\mathcal{R}_{1,\infty}\left(\frac{\alpha}{5}, \rho\right)$ (representing the infinite-sample limit) and the minimax rate of clean distribution estimation under the $W_1$ metric. Conversely, the upper bound is characterized by the population-level resilience $\rho$ and the expected estimation error of the empirical measure $P_n$ over the class $\mathcal{G}$.

When we assume bounded second moments (i.e., $P^* \in \mathcal{G}_{\mathrm{cov}}(\sigma)$), we have $\mathcal{R}_{1,\infty}\left(\frac{\alpha}{5}, \rho\right) \asymp \rho \asymp \sigma\sqrt{d\alpha}$ by Theorem \ref{thm:minimax_optimal}, while the empirical measure $P_n$ achieves the optimal $W_1$-error rate of $\mathcal{O}(\sigma\sqrt{d}n^{-1/d})$ \citep{lei2020convergence, niles2022minimax}. This matches the lower and upper bounds up to constant factors therefore shows that our Wasserstein Filtering estimator $\widehat{P}_{\mathrm{WF}}$ is minimax optimal.

The focus on the bounded-covariance class $\mathcal{G}_{\mathrm{cov}}(\sigma)$ is particularly significant because it allows the clean distribution to be heavy-tailed. Under heavy-tailed distributions, genuine data points can appear far from the mean, which naturally obscures the distinction between clean data and adversarial outliers. This makes outlier detection and robust estimation exceptionally challenging, highlighting the strength of our framework's guarantees.

\subsection{Convergence of the Algorithm}

In this section, we analyze the algorithmic convergence of the sample weight optimization scheme employed in the \texttt{SinkWF} and \texttt{SlicedWF} algorithms. For simplicity of exposition, we let $\sigma(\mathbf{b}) = \operatorname{softmax}(\mathbf{b})$ denote the softmax mapping and define the clipping threshold as $u_\alpha = \frac{1}{(1-\alpha)n}$. We consider the idealized objective function with the exact $W_1$ distance:
\begin{align*}
F(\mathbf{b}) = W_1(P_n, P_{\sigma(\mathbf{b})}) - \beta \sum_{i=1}^n (\sigma_i(\mathbf{b}) - u_\alpha)_+,
\end{align*}
where $(t)_+ = \max\{t, 0\}$. We further simplify the optimization dynamics from Adam to standard gradient ascent updates with step size $\gamma > 0$:
%\begin{align*}
$
\mathbf{b}_{r+1} = \mathbf{b}_r + \gamma \nabla F(\mathbf{b}_r).
$
%\end{align*}
We assume there exists an iteration index $r_0$ and a compact set $\mathbb{K} \subset \mathbb{R}^n$ such that the iterates satisfy $\mathbf{b}_r \in \mathbb{K}$ for all $r \geq r_0$. Without loss of generality, we treat $r_0$ as the initial step and analyze the subsequent $R$ iterations. To establish convergence, we introduce the following regularity conditions on the optimization landscape:

\begin{con}[Unique Dual Potential]\label{con:unique_dual}
For every parameter vector $\mathbf{b} \in \mathbb{K}$ and its induced probability vector $\mathbf{p} = \sigma(\mathbf{b})$, the exact Optimal Transport (OT) dual problem,
\begin{align*}
W_1(P_n, P_{\mathbf{p}}) = \max_{\substack{\mathbf{u}, \mathbf{v} \in \mathbb{R}^n \\ u_i + v_j \leq C_{ij}}} \left\{ \frac{1}{n}\sum_{i=1}^n u_i + \sum_{j=1}^n p_j v_j \right\},
\end{align*}
admits a unique optimal target potential $\mathbf{v}^*(\mathbf{p})$ satisfying the normalization constraint $\sum_{j=1}^n v_j^*(\mathbf{p}) = 0$.
\end{con}

\begin{con}[Stable Active Set and Non-boundary Hitting]\label{con:stable_active}
Let $\mathbf{m}(\mathbf{p}) \in \{0,1\}^n$ define the active-set indicator vector of the penalty term, where $m_i(\mathbf{p}) = \mathbb{I}_{\{p_i > u_\alpha\}}$. There exists a constant vector $\mathbf{m}_0 \in \{0,1\}^n$ such that $\mathbf{m}(\sigma(\mathbf{b})) = \mathbf{m}_0$ for all $\mathbf{b} \in \mathbb{K}$. Furthermore, the softmax weights do not hit the boundary of the threshold, i.e., $\sigma_i(\mathbf{b}) \neq u_\alpha$ for all $i \in [n]$ and $\mathbf{b} \in \mathbb{K}$.
\end{con}

\begin{con}[Smoothness of the Optimal Dual]\label{con:smooth_dual}
The optimal target dual potential $\mathbf{v}^*(\cdot)$ is Lipschitz continuous on the simplex image $\sigma(\mathbb{K})$. That is, for any $\mathbf{b}, \mathbf{c} \in \mathbb{K}$ with $\mathbf{p} = \sigma(\mathbf{b})$ and $\mathbf{q} = \sigma(\mathbf{c})$, there exists a constant $L_{\mathbf{v}^*} > 0$ such that:
%\begin{align*}
$
\|\mathbf{v}^*(\mathbf{p}) - \mathbf{v}^*(\mathbf{q})\|_2 \leq L_{\mathbf{v}^*} \|\mathbf{p} - \mathbf{q}\|_2.
$
%\end{align*}
\end{con}

\begin{thm}[Convergence Rate of the WF Algorithm]\label{thm:convergence_rate}
Under Conditions \ref{con:unique_dual}--\ref{con:smooth_dual}, the gradient $\nabla F$ is $L_F$-Lipschitz continuous on $\mathbb{K}$ for some constant $L_F > 0$. If the step size satisfies $\gamma < \frac{2}{L_F}$, the iterates of the gradient ascent scheme satisfy:
\begin{align*}
\min_{r\in\{r_0,\ldots,r_0+R-1\}} \|\nabla F(\mathbf{b}_r)\|_2 \leq \sqrt{\frac{\max_{i,j\in[n]}\|\mathbf{x}_i-\mathbf{x}_j\|_2 - F(\mathbf{b}_{r_0})}{\left(\gamma-\frac{L_F\gamma^2}{2}\right)R}}.
\end{align*}
\end{thm}

Condition \ref{con:unique_dual} and the non-boundary hitting portion of Condition \ref{con:stable_active} ensure that the objective function $F(\mathbf{b})$ is continuously differentiable on $\mathbb{K}$. This guarantees the existence of a unique, explicit gradient, thereby avoiding the analytical complexities associated with subgradient calculus. Additionally, Condition \ref{con:smooth_dual} together with the stable active set property in Condition \ref{con:stable_active} ensures the $L_F$-Lipschitz smoothness of $\nabla F$.

Theorem \ref{thm:convergence_rate} demonstrates that our sample-weight optimization framework achieves a sublinear convergence rate of $\mathcal{O}(1/\sqrt{R})$, matching the standard optimal rate for first-order methods in non-concave smooth optimization.

\section{Numerical Studies}

In this section, we evaluate the performance of our proposed methods against several state-of-the-art baselines:
\begin{itemize}
    \item {DIF (Deep Isolation Forest)} \citep{xu2023deep}.
    \item {MDE (Minimum Distance Estimation)} \citep{nietert2024robust}. Note that the MDE algorithm evaluated here is a specialized version of the framework in \cite{nietert2024robust} designed for TV corruption, which reduces to a standard spectral iterative filtering procedure (see Section 2.3 of \cite{nietert2024robust}).
    \item {COPOD (Copula-based Outlier Detection)} \citep{li2020copod}.
\end{itemize}
To assess performance on these binary classification (outlier detection) tasks, we employ two widely used metrics: the Area Under the Receiver Operating Characteristic curve ({AUC-ROC}) and the Area Under the Precision-Recall curve ({AUC-PR}).

\subsection{Two-Dimensional Synthetic Data}

We first compare our methods with the competing baselines on two-dimensional synthetic datasets to facilitate clear visualization. Specifically, we generate two classic datasets—the two-moons dataset and the concentric circles dataset, using the \texttt{make\_moons} and \texttt{make\_circles} functions from the \texttt{scikit-learn} Python library with a noise level of $0.1$. For each dataset, we generate $n = 500$ samples, where a fraction $\alpha$ of the samples is corrupted. The corrupted samples are generated from two different Gaussian distributions representing different contamination strengths:
\begin{align*}
\text{\textbf{Weak Case:}} \quad \mathcal{N}\left(\begin{bmatrix}
-0.5 \\
-0.5
\end{bmatrix}, \begin{bmatrix}
0.01 & 0 \\
0 & 0.01
\end{bmatrix} \right)\!, \quad \text{and} \quad
\text{\textbf{Strong Case:}} \quad \mathcal{N}\left(\begin{bmatrix}
-2.5 \\
-2.5
\end{bmatrix}, \begin{bmatrix}
0.01 & 0 \\
0 & 0.01
\end{bmatrix} \right)\!.
\end{align*}
The ``strong" contamination scenario features outliers that are well-separated from the clean data cluster, whereas the "weak" scenario places outliers much closer to the clean data, as illustrated in Figure \ref{fig:2d_visualization} in the Supplementary Material.

We consider two contamination proportions: $\alpha = 0.05$ and $\alpha = 0.15$. For the algorithms requiring the outlier fraction as an input (the robust WF variants and MDE), we provide the true value of $\alpha$. For \texttt{SinkWF} and \texttt{SlicedWF}, the optimization parameters are initialized as $\mathbf{b}_0 = \mathbf{1} + 10^{-3} \mathbf{z}$, where $\mathbf{z} \sim \mathcal{N}(\mathbf{0}, \mathbf{I})$, and we set the number of repetitions to $T = 10$. Other hyperparameters are selected via grid search, with the search spaces detailed in Appendix C.1.

We repeat each experiment $50$ times, executing an independent grid search for each run. The average AUC-ROC and AUC-PR, along with their corresponding standard errors, are reported in Table \ref{table:2d_data}. Our implementations of the three WF algorithms leverage the Python Optimal Transport (\texttt{POT}) toolbox \citep{flamary2021pot} for \texttt{SlicedWF} and the \texttt{GeomLoss} library \citep{feydy2019interpolating} for \texttt{SinkWF} and \texttt{SinkMarg}.

\begin{table}[!h]
\centering
\caption{Results of 2D data. Weak and strong contamination implies the signal of corruption strength. The mean of $50$ times repetition is recorded and the corresponding standard error is recorded in parenthesis right side.}
\label{table:2d_data}
\footnotesize{
\begin{tabular}{llcccc} \bottomrule
                &       \multicolumn{5}{c}{Weak Contamination} \\ \cline{2-6}
     Dataset    &       &  \multicolumn{2}{c}{AUC-ROC} & \multicolumn{2}{c}{AUC-PR} \\ \cline{3-6}
                &       & $\alpha=0.05$ & $\alpha=0.15$ & $\alpha=0.05$ & $\alpha=0.15$  \\
                &       COPOD          & 0.89(0.02) & 0.80(0.01) & 0.25(0.05)   & 0.34(0.02)  \\
                &       DIF            & 0.96(0.02) & 0.32(0.06) & 0.49(0.12)   & 0.13(0.02)   \\
                &       MDE            & 0.67(0.05) & 0.65(0.03) & 0.08(0.01)   & 0.20(0.02)    \\
                &       SlicedWF       & 0.77(0.12) & 0.85(0.15) & 0.23(0.29)   & 0.58(0.32)    \\
                &       SinkWF         & 0.85(0.10) & 1.00(0.00) & 0.31(0.28)   & 1.00(0.00)     \\
                &       SinkMarg       & 0.63(0.06) & 0.85(0.03) & 0.11(0.05)   & 0.77(0.05)      \\

     Two moons    &       \multicolumn{5}{c}{Strong Contamination} \\ \cline{2-6}
                &       & \multicolumn{2}{c}{AUC-ROC} & \multicolumn{2}{c}{AUC-PR} \\ \cline{3-6}
                &       & $\alpha=0.05$ & $\alpha=0.15$ & $\alpha=0.05$ & $\alpha=0.15$  \\
                &       COPOD         & 1.00(0.00)  & 0.97(0.00) & 0.97(0.01)  & 0.83(0.01)   \\
                &       DIF           & 0.99(0.00)  & 0.42(0.10) & 0.85(0.07)  & 0.21(0.06)    \\
                &       MDE           & 0.98(0.01)  & 1.00(0.00) & 0.66(0.12)  & 1.00(0.00)     \\
                &       SlicedWF      & 1.00(0.01)  & 0.99(0.02) & 0.99(0.09)  & 0.96(0.11)     \\
                &       SinkWF        & 1.00(0.00)  & 1.00(0.00) & 1.00(0.00)  & 1.00(0.00) \\
                &       SinkMarg      & 0.90(0.02)  & 0.91(0.02) & 0.30(0.15)  & 0.88(0.03)      \\ \hline
                                &       \multicolumn{5}{c}{Weak Contamination} \\ \cline{2-6}
                &       &  \multicolumn{2}{c}{AUC-ROC} & \multicolumn{2}{c}{AUC-PR} \\ \cline{3-6}
                &       & $\alpha=0.05$ & $\alpha=0.15$ & $\alpha=0.05$ & $\alpha=0.15$  \\
                &       COPOD          & 0.58(0.02) & 0.48(0.01) & 0.06(0.00)  & 0.13(0.00)   \\
                &       DIF            & 0.24(0.07) & 0.05(0.02) & 0.03(0.00)  & 0.08(0.00)   \\
                &       MDE            & 0.42(0.05) & 0.38(0.02) & 0.04(0.00)  & 0.12(0.00)     \\
                &       SlicedWF       & 0.80(0.14) & 0.82(0.14) & 0.19(0.09)  & 0.41(0.15)    \\
                &       SinkWF         & 0.88(0.07) & 0.99(0.00) & 0.33(0.17)  & 0.92(0.03)   \\
                &       SinkMarg       & 0.72(0.08) & 0.92(0.02) & 0.16(0.08)  & 0.74(0.05)     \\
     Circles    &       \multicolumn{5}{c}{Strong Contamination} \\ \cline{2-6}
                &       & \multicolumn{2}{c}{AUC-ROC} & \multicolumn{2}{c}{AUC-PR} \\ \cline{3-6}
                &       & $\alpha=0.05$ & $\alpha=0.15$ & $\alpha=0.05$ & $\alpha=0.15$  \\
                &       COPOD         & 1.00(0.00) & 0.96(0.00) & 0.97(0.01) & 0.77(0.01)    \\
                &       DIF           & 0.99(0.00) & 0.49(0.09) & 0.88(0.08) & 0.19(0.04)   \\
                &       MDE           & 1.00(0.00) & 1.00(0.00) & 1.00(0.00) & 1.00(0.00)     \\
                &       SlicedWF      & 1.00(0.00) & 0.99(0.03) & 1.00(0.01) & 0.95(0.12)   \\
                &       SinkWF        & 1.00(0.00) & 1.00(0.00) & 1.00(0.00) & 1.00(0.00) \\
                &       SinkMarg      & 0.91(0.03) & 0.93(0.02) & 0.64(0.14) & 0.89(0.02)     \\ \bottomrule
\end{tabular}
}
\end{table}

The experimental results in Table \ref{table:2d_data} demonstrate that \texttt{SinkWF} consistently achieves the best performance across most scenarios. Notably, \texttt{SinkWF} exhibits a more pronounced performance advantage over \texttt{SlicedWF} in the weak contamination regime. This behavior is visually illuminated by Figure \ref{fig:2d_visualization} in the Supplementary Material, which highlights that the directional signal of contamination is significantly weaker in the weak case.

As the data dimension $d$ increases, this directional signal becomes even more severely diluted unless the number of random projections $S$ in the sliced distance grows exponentially with $d$. In contrast, the entropic Optimal Transport used in \texttt{SinkWF} is inherently dimension-free, albeit at the expense of explicitly storing and computing the $n \times n$ cost matrix.

Additionally, the two robust joint WF algorithms perform better as the contamination fraction $\alpha$ increases, a trend particularly evident under weak contamination. However, \texttt{SinkWF} remains highly sensitive to variations in this contamination proportion.

Benefiting from the global geometric properties of the Wasserstein distance, our WF algorithms are capable of detecting outliers even when they are not exposed at the extreme edges of the clean data distribution. This capability is clearly illustrated in Figure \ref{fig:2d_visualization}(a) in the Supplementary Material.

In the two-moons dataset, while baselines like \texttt{DIF} and \texttt{COPOD} can identify some outlying points, the contrast is far more stark in the concentric circles dataset. Here, the outliers concentrate in the region between the inner and outer circles, completely surrounded by normal points. In this challenging scenario, \texttt{DIF}, \texttt{MDE}, and \texttt{COPOD} mistakenly flag points on the outer boundary of the clean distribution as outliers. This limitation stems from their underlying mechanisms, which rely heavily on tail mass or spatial isolation to determine outlier scores.

By contrast, \texttt{SinkWF} and \texttt{SinkMarg} successfully isolate nearly all hidden outliers, while \texttt{SlicedWF} successfully recovers a significant portion of them. This is also reflected quantitatively in Table \ref{table:2d_data}, where all three WF algorithms exhibit a substantial performance lead on the concentric circles dataset, particularly under weak contamination.

\subsection{Benchmark Datasets}

In this section, we systematically validate and compare our proposed framework across several standard benchmark datasets spanning
% three
two distinct modalities: tabular and graph-structured.

% and time-series data.

\begin{itemize}
    \item {Tabular Data:} We utilize representative datasets from the comprehensive anomaly detection benchmark \texttt{ADBench} \citep{han2022adbench}.
    \item {Graph-Structured Data:} We evaluate our methods on the Tox21 molecular dataset. To process the graphs, we vectorize the molecular structures using the pretrained \texttt{Mol2vec} model \citep{jaeger2018mol2vec}. Both the raw datasets and the pretrained representations are retrieved using the \texttt{DeepChem} ecosystem \citep{ramsundar2019deep}.
\iffalse
    \item {Time-Series Data:} We select benchmarks from the widely recognized UCR Time Series Classification Archive \citep{UCRArchive2018, dau2019ucr}. To extract high-quality features, we transform the raw time-series inputs into low-dimensional representations using the self-supervised framework \texttt{TS2vec} \citep{yue2022ts2vec}.
\fi
\end{itemize}

The empirical results and performance comparisons for the Tabular Data are summarized
in Table \ref{table:benchmark_tabular_penalty}.
The results for  Graph-Structured Data are given in
Table \ref{table:benchmark_graph_penalty} in the Supplementary Material.

%and Graph-Structured Data modalities are summarized in Tables %\ref{table:benchmark_tabular_penalty} and \ref{table:benchmark_graph_penalty}.
% and \ref{table:benchmark_time_penalty}, respectively.

\begin{table}[!h]
\footnotesize{
\centering
\caption{Full results of tabular benchmark datasets. SinkMarg on smtp dataset is abandoned since its runtime exceeds 24h.}
\label{table:benchmark_tabular_penalty}
\begin{tabular}{lllcccccc} \bottomrule
& & & \multicolumn{6}{c}{AUC-ROC} \\ \cline{4-9}
            & Dataset & $(n,d,\alpha)$ & COPOD & DIF & MDE & SlicedWF &  SinkWF & SinkMarg   \\ \cline{2-9}
\multirow{14}{*}{Tabular}    & vertebral   & $(240,6,12.50\%)$    & 0.39  & 0.39  & 0.43 & 0.61 & 0.74 & 0.34    \\
                             & yeast       & $(1484,8,34.16\%)$   & 0.41  & 0.42  & 0.50 & 0.50 & 0.70 & 0.35   \\
                             & annthyroid  & $(7200,6,7.42\%)$    & 0.68  & 0.71  & 0.67 & 0.59 & 0.61 & 0.82    \\
                             & breastw     & $(683,9,34.99\%)$    & 0.99  & 0.98  & 0.99 & 0.99 & 0.99 & 0.98    \\
                             & cardio      & $(1831,21,9.61\%)$   & 0.88  & 0.95  & 0.82 & 0.80 & 0.76 & 0.56    \\
                             & tocography    & $(2114,21,22.04\%)$  & 0.65  & 0.68  & 0.71 & 0.65 & 0.77 & 0.46    \\
                             & Hepatitis   & $(80,19,16.25\%)$    & 0.80  & 0.81  & 0.60 & 0.72  & 0.68 & 0.63    \\
                             & Lymph       & $(148,18,4.05\%)$    & 1.00  & 1.00  & 0.92 & 0.95  & 0.96 & 0.99   \\
                             & shuttle     &$(49097,9,7.15\%)$    & 0.99  & 0.99  & 0.96 & 0.95  & 0.98 & 0.75    \\
                             & pendigits   &$(6870,16,2.27\%)$    & 0.82  & 0.96  & 0.76 & 0.97  & 0.58 & 0.72    \\
                             & satellite   &$(6435,36,31.64\%)$   & 0.50  & 0.73  & 0.50 & 0.75  & 0.66 & 0.66    \\
                             & WBC         &$(223,9,4.48\%)$      & 0.99  & 0.98  & 0.99 & 1.00  & 0.94 & 0.97    \\
                             & smtp        &$(95156,3,0.03\%)$    & 0.90  & 0.93  & 0.83 & 0.77  & 0.58  &  -   \\
                             & optdigits   &$(5216,64,2.88\%)$    & 0.71  & 0.57  & 0.50 & 0.58  & 0.67 & 0.35    \\ \hline
                             & mean        &                      & 0.77  & 0.79  & 0.73 & 0.77  & 0.76 & 0.66   \\
                             & wins        &                      & 28.57\% & 35.71\% & 7.14\% & 28.57\% & 28.57\% & 7.14\% \\
& & & \multicolumn{6}{c}{AUC-PR} \\ \cline{4-9}
             & Dataset & $(n,d,\alpha)$ & COPOD & DIF & MDE & SlicedWF &  SinkWF & SinkMarg  \\ \cline{2-9}
\multirow{14}{*}{Tabular}     & vertebral   & $(240,6,12.50\%)$    & 0.10  & 0.10 & 0.11  & 0.15 & 0.23 & 0.09     \\
                              & yeast       & $(1484,8,34.16\%)$   & 0.31  & 0.30 & 0.34  & 0.34 & 0.52 & 0.27    \\
                              & annthyroid  & $(7200,6,7.42\%)$    & 0.16  & 0.27 & 0.16  & 0.10 & 0.14 & 0.24    \\
                              & breastw     & $(683,9,34.99\%)$    & 0.98  & 0.94 & 0.97  & 0.99 & 0.99 & 0.94    \\
                              & cardio      & $(1831,21,9.61\%)$   & 0.46  & 0.58 & 0.24  & 0.52 & 0.31 & 0.17    \\
                              & tocography    & $(2114,21,22.04\%)$  & 0.37  & 0.42 & 0.35  & 0.41 & 0.46 & 0.24    \\
                              & Hepatitis   & $(80,19,16.25\%)$    & 0.42  & 0.47 & 0.19  & 0.34 & 0.36 & 0.22    \\
                              & Lymph       & $(148,18,4.05\%)$    & 0.87  & 0.95 & 0.21  & 0.31 & 0.54 & 0.65    \\
                              & shuttle     &$(49097,9,7.15\%)$    & 0.91  & 0.88 & 0.74  & 0.77 & 0.96 & 0.18        \\
                              & pendigits   &$(6870,16,2.27\%)$    & 0.11  & 0.29 & 0.05  & 0.38 & 0.03 & 0.07    \\
                              & satellite   &$(6435,36,31.64\%)$   & 0.34  & 0.64 & 0.32  & 0.68 & 0.69 & 0.49    \\
                              & WBC         &$(223,9,4.48\%)$      & 0.90  & 0.67 & 0.80  & 0.96 & 0.38 & 0.69    \\
                              & smtp        &$(95156,3,0.03\%)$    & 0.47  & 0.39 & 0.06  & 0.01 & 0.00 &  -   \\
                              & optdigits   &$(5216,64,2.88\%)$    & 0.05  & 0.03 & 0.03  & 0.04 & 0.28 & 0.02    \\ \hline
                              & mean        &                      & 0.46  & 0.50 & 0.33  & 0.43 & 0.42 & 0.33   \\
                              & wins        &                      & 7.14\% & 28.57\%  & 0.00\%  & 21.43\%  & 50.00\%  & 0.00\%  \\
\bottomrule
\end{tabular}
}
\end{table}

\subsubsection{Discussion of Empirical Results}

\paragraph{Tabular Data (Table \ref{table:benchmark_tabular_penalty}).}
As shown in Table \ref{table:benchmark_tabular_penalty}, \texttt{SinkWF} achieves state-of-the-art or near-optimal performance on datasets with higher contamination fractions $\alpha$ (e.g., \emph{yeast}, \emph{cardiotocography}, \emph{satellite}, and \emph{shuttle}), though its performance degrades in low-$\alpha$ regimes (e.g., \emph{pendigits}). Conversely, \texttt{SlicedWF} exhibits a clear advantage on low-to-moderate dimensional datasets (such as \emph{pendigits}, \emph{WBC}, and \emph{satellite}). In these settings, the lower-dimensional informative features ensure that random 1D projections preserve the class-separating geometric structure. However, the performance of \texttt{SlicedWF} systematically degrades as the dimensionality increases (e.g., \emph{optdigits}).

A notable failure mode for all Wasserstein Filtering (WF) variants occurs on the \emph{smtp} dataset. In scenarios characterized by an extremely large sample size $n$ and a microscopic contamination rate $\alpha$, the total mass of the contaminated samples becomes statistically negligible. Consequently, the Wasserstein distance between the full empirical distribution and any sub-sample is dominated by the benign intra-class variation of the clean data rather than the geometry of the outliers. In this regime, the core principle of our framework, maximizing distributional discrepancy by withholding an $\alpha$-fraction of the data, loses its signal. In contrast, marginal-based methods such as \texttt{COPOD} and isolation-based methods like \texttt{DIF} succeed here; their mechanisms (marginal extremeness and recursive isolation, respectively) are designed to identify isolated, rare anomalies in low-dimensional spaces without requiring a population-level geometric signal.

For \texttt{SinkMarg}, the best performance is achieved on \emph{annthyroid}, which is a large dataset ($n=7200$) with a low dimension ($d=6$) and simple structure, where leave-one-out Sinkhorn scoring is highly effective at identifying well-isolated outliers. However, \texttt{SinkMarg} performs poorly on \emph{vertebral} (AUC-ROC of $0.34$) and \emph{yeast} ($0.35$), performing worse than a random classifier. This highlights a fundamental limitation of \texttt{SinkMarg}: it cannot capture group-structured outliers, where contaminated points form tight, coherent clusters rather than isolated Dirac masses.

Finally, \texttt{MDE} is theoretically elegant but practically weak across all tabular datasets. The datasets in \texttt{ADBench} feature highly heterogeneous distributions (comprising medical records, sensor data, and handwritten digits). The spectral filtering mechanism of \texttt{MDE} relies heavily on the assumption that the clean distribution is isotropic, a condition that is severely violated in real-world tabular data.

\paragraph{Graph-Structured Data (Table \ref{table:benchmark_graph_penalty} in the Supplementary Material).}
In the high-dimensional molecular representation space of Tox21, \texttt{MDE} fails completely. \texttt{MDE} iteratively computes the leading eigenvector of the sample covariance matrix and filters samples along this direction. In $d=300$ dimensions with only a small fraction of outliers ($\sim$3–12\%), the leading eigenvector of the covariance matrix is entirely dominated by the variance of the clean data, rendering the diffuse contamination signal spectrally invisible.

For \texttt{SlicedWF}, using $S=1000$ random projections (see Appendix C.2) in a $300$-dimensional space is insufficient. Most random 1D projections fail to align with the meaningful directions of molecular variation, making the sliced Wasserstein distance a poor proxy for the true $W_1$ metric. This highlights a fundamental statistical-computational tradeoff: while \texttt{SlicedWF} enjoys an $\mathcal{O}(n)$ space complexity, it suffers a severe approximation bottleneck in high dimensions. In contrast, \texttt{SinkWF} incurs an $\mathcal{O}(n^2)$ space complexity but preserves full geometric fidelity, rendering it robust to high dimensionality.

\iffalse
\paragraph{Time-Series Data (Table \ref{table:benchmark_time_penalty}).}
Interestingly, in Table \ref{table:benchmark_time_penalty}, \texttt{SlicedWF} often matches or outperforms \texttt{SinkWF} (e.g., on \emph{Power}, \emph{Surface1}, and \emph{FordB}), which is a reversal of the trend observed in the graph datasets. This difference can be explained by the contrasting geometric structures of \texttt{TS2vec} and \texttt{Mol2vec} embeddings. Temporal signals processed through self-supervised contrastive learning (\texttt{TS2vec}) typically exhibit a much lower effective intrinsic dimensionality and highly aligned directional trajectories, which random projections can easily exploit. In contrast, \texttt{Mol2vec} encodes subgraph co-occurrences in a highly entangled 300-dimensional space where no single direction dominates.

Lastly, \texttt{DIF} substantially outperforms both \texttt{SlicedWF} and \texttt{SinkWF} on the \emph{Wafer} dataset, which features the lowest contamination rate and the largest sample size $n$ in the time-series suite. This again underscores the limitation of WF methods in low-$\alpha$ regimes: when the Wasserstein discrepancy signal is weak, the geometric advantage of WF diminishes, allowing tree-based recursive isolation methods (\texttt{DIF}) to construct tighter, data-adaptive decision boundaries around minority anomalous points.
\fi

When the true contamination rate $\alpha$ is unknown, we show in Appendix C.2 that our automatic determination procedure for $\alpha$ incurs negligible performance loss compared to using the ground-truth value.

\subsection{Robust Generative Learning}

In this subsection, we demonstrate how our proposed framework can facilitate downstream tasks, specifically robust generative learning. The goal of generative learning is to synthesize novel samples $\mathbf{x}_i$ ($i \geq n+1$) that mimic the target distribution $P^*$ based on corrupted observed data $\widetilde{\mathbf{x}}_1, \dots, \widetilde{\mathbf{x}}_n$. For this experiment, we deploy a diffusion-based generative model—the Denoising Diffusion Probabilistic Model (\texttt{DDPM}) \citep{ho2020denoising}—on the MNIST handwritten digit dataset.

We combine the standard training and testing partitions of MNIST to obtain a total sample size of $n = 70,000$. The dataset is subjected to two distinct types of adversarial corruptions:
\begin{itemize}
    \item {Strong Contamination:} The top half of each contaminated image is completely masked (replaced with white pixels).
    \item {Weak Contamination:} A random one-fourth portion of each contaminated image is masked, as illustrated in Figure \ref{figure:MNIST}(a).
\end{itemize}
We compare the visual quality and distribution of the synthesized images generated by \texttt{DDPM} when trained on: (i) the raw contaminated dataset, and (ii) the clean dataset retrieved after filtering out suspected outliers using our proposed methods.

Performance is evaluated using two sets of metrics: the classification accuracy of outlier detection ({AUC-ROC} and {AUC-PR}) and the quality of the generated images, quantified by the Fréchet Inception Distance ({FID}) \citep{heusel2017gans}. Because the information in the corrupted regions is fundamentally unrecoverable, the FID is computed strictly between the synthesized images and the ground-truth clean (uncorrupted) images. Additionally, we evaluate a baseline scenario with no actual corruption ($\alpha = 0$) while still running our algorithms with a nominal input of $\alpha = 0.05$. This allows us to quantify the potential performance degradation (or "robustness tax") incurred by filtering uncorrupted datasets.

In our implementation, we leverage \texttt{POT} for \texttt{SlicedWF} and the memory-efficient \texttt{GeomLoss} library for \texttt{SinkWF} and \texttt{SinkMarg}. We initialize the parameters as $\mathbf{b}_0 = \mathbf{1} + 10^{-3}\mathbf{z}$, where $\mathbf{z} \sim \mathcal{N}(\mathbf{0}, \mathbf{I})$, and utilize the squared Wasserstein-2 ($p=2$) distance across all three WF variants. We set the optimization iterations to $T=1$ for \texttt{SinkWF} and $T=10$ for \texttt{SlicedWF}. All remaining hyperparameter configurations are detailed in Appendix C.3. The empirical results are summarized in Table \ref{table:MNIST_main}.

\paragraph{Discussion of Generative Results.}
As reported in Table \ref{table:MNIST_main}, \texttt{SlicedWF} achieves slightly superior performance compared to \texttt{SinkWF} in outlier classification metrics (AUC-ROC and AUC-PR), whereas \texttt{SinkWF} yields a lower (better) FID score. Notably, \texttt{SinkMarg} proves computationally impractical in this setting; the extremely large sample size ($n = 70,000$) severely slows down the leave-one-out Sinkhorn computations, making it scale poorly to large-scale datasets. Consistent with our previous experiments, the outlier detection capabilities of \texttt{SinkWF} improve as the contamination rate increases from $\alpha = 0.05$ to $\alpha = 0.15$.

Regarding image synthesis quality, the $\alpha = 0$ baseline experiment demonstrates that applying our filtering algorithms to entirely clean data has a negligible impact on downstream generative modeling, indicating that our methods do not discard valuable clean samples. Conversely, the utility of filtering becomes striking under heavy contamination ($\alpha = 0.15$): pre-filtering the training data with our WF algorithms yields a dramatic 3-to-4-fold reduction in FID score compared to training on the raw contaminated data, highlighting the necessity of robust data sanitization in generative workflows.

\begin{figure}[!h]
    \centering

    % Column 1: original / contaminated images
    \begin{subfigure}[t]{\textwidth}
        \centering
         \includegraphics[width=0.12\linewidth]{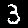}
        \includegraphics[width=0.12\linewidth]{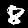}
        \ \
%
 %       \medskip
%
      \includegraphics[width=0.12\linewidth]{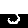}
        \includegraphics[width=0.12\linewidth]{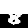}
        \ \
%
 %       \medskip
%
  \includegraphics[width=0.12\linewidth]{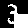}
        \includegraphics[width=0.12\linewidth]{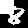}

        \caption{Original and contaminated images. From left to right: clean images, strong contaminated images, and weak contaminated images.}
    \end{subfigure}\\
%    \hfill
    % Column 2: generated images
    \begin{subfigure}[t]{\textwidth}
        \centering
        \includegraphics[width=0.12\linewidth]{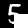}
         \includegraphics[width=0.12\linewidth]{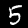}
         \includegraphics[width=0.12\linewidth]{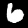}
         \ \ \ \
        \includegraphics[width=0.12\linewidth]{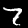}
%
    %    \medskip
%
        \includegraphics[width=0.12\linewidth]{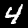}
%
%     \medskip
%
        \includegraphics[width=0.12\linewidth]{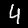}

        \caption{Generated images. First 3 images from the left side: trained on raw data with $\alpha=0.15$ weak contamination. The 3 images on the right: trained after outlier removal by SlicedWF.}
    \end{subfigure}

    \caption{Result of generated images without and with outlier removal preprocessing.}
    \label{figure:MNIST}
\end{figure}

\begin{table}[!h]
\centering
\caption{Results of DDPM training on MNIST. FID evaluate the quality of generated images, AUC-ROC and AUC-PR evaluate the performance of outlier detection, weak and strong contamination are two modes of contamination. SinkMarg is abandoned since its runtime exceeds 24h.}
\label{table:MNIST_main}
\begin{tabular}{lccccccc} \bottomrule
              & \multicolumn{7}{c}{Weak contamination} \\ \cline{2-8}
              & \multicolumn{3}{c}{FID score} & \multicolumn{2}{c}{AUC-ROC} & \multicolumn{2}{c}{AUC-PR} \\ \cline{2-8}
              & $\alpha=0$ & $\alpha=0.05$ & $\alpha=0.15$ & $\alpha=0.05$ & $\alpha=0.15$ & $\alpha=0.05$ & $\alpha=0.15$ \\
              Original data  & 12.45  & 17.23   & 21.13  & - & - & - & - \\
              SlicedWF       & 11.25  & 20.77   & 4.83   & 1.00  & 1.00  & 1.00  & 1.00 \\
              SinkWF         & 10.13  & 16.17   & 4.83   & 0.80  & 1.00  & 0.71  & 1.00 \\
              SinkMarg       & -  &  -  &  -  & -  & -  & -  & - \\
\hline
              & \multicolumn{7}{c}{Strong contamination} \\ \cline{2-8}
              & \multicolumn{3}{c}{FID score} & \multicolumn{2}{c}{AUC-ROC} & \multicolumn{2}{c}{AUC-PR} \\ \cline{2-8}
              & $\alpha=0$ & $\alpha=0.05$ & $\alpha=0.15$ & $\alpha=0.05$ & $\alpha=0.15$ & $\alpha=0.05$ & $\alpha=0.15$ \\
              Original data   & 12.45  & 12.13  & 14.85  & - & - & - & - \\
              SlicedWF        & 13.27  & 20.77  & 4.83   & 1.00  & 1.00  & 1.00  & 1.00 \\
              SinkWF          & 16.15  & 8.98   & 5.04   & 1.00  & 1.00  & 0.94  & 1.00 \\
              SinkMarg        & -  &  - & - & - & - & - & - \\
              \bottomrule
\end{tabular}
\end{table}

\section{Conclusion}

In this paper, we introduced Wasserstein Filtering (\texttt{WF}), a novel task-agnostic sample selection framework designed to detect and eliminate outliers that cause significant distributional distortion under the Wasserstein distance. By isolating a clean subset of the data, the empirical measure of the remaining samples naturally serves as a robust estimator of the underlying clean distribution.

The core advantages of the \texttt{WF} framework are summarized as follows:
\begin{itemize}
    \item \text{Task-Agnostic Preprocessing:} As a pure sample-selection scheme, \texttt{WF} operates independently of downstream learning tasks, making it a highly versatile preprocessing tool for general data sanitization.
    \item \text{Minimal Distributional Assumptions:} Unlike existing robust estimators, \texttt{WF} does not impose restrictive parametric assumptions on the clean distribution, nor does it require prior knowledge of its specific parameters.
    \item \text{Computational Tractability and Adaptivity:} Harnessing modern deep-learning optimization frameworks (e.g., automatic differentiation and GPU acceleration), the optimization of sample weights is computationally efficient. Moreover, the framework is equipped with an automated data-driven procedure to estimate the unknown contamination rate $\alpha$.
    \item \text{Robustness to Non-Tail and Structured Outliers:} By leveraging the global geometric properties of optimal transport, \texttt{WF} successfully identifies structured outliers and hidden anomalies that do not lie on the extreme tails of the clean distribution, surpassing the limitations of isolation- and density-based methods.
    \item \text{Flexibility across Dimensions:} The framework offers scalable options tailored to different data regimes: \texttt{SlicedWF} provides $\mathcal{O}(n)$ space efficiency in large-sample, low-dimensional settings ($n \gg d$), while \texttt{SinkWF} maintains geometric fidelity in high-dimensional, moderate-sample regimes ($d \gg n$).
\end{itemize}

We systematically validated the empirical performance of the \texttt{WF} framework across tabular and graph data, %and time-series modalities,
demonstrating its consistent superiority over state-of-the-art baselines. Furthermore, on the theoretical front, we established the minimax optimality of the empirical risk under the Far Exclusion with Local Projection (\texttt{FELP}) contamination model, confirming that our framework achieves the optimal statistical rate of robust distribution estimation.

%%%%%%%%%%%%%%%%%%%%%%%%%%%%%%%%%%%%%%%%%%%%%%%%%%%%%%%%%%%%%%%%%%%%%%%%%%%%%%
%\clearpage

\renewcommand{\thefigure}{A.\arabic{figure}}
\appendix

\vspace{2 cm}
\noindent
\textbf{\LARGE Supplementary Material}

\medskip
This Supplementary Material contains the complete proofs for our theoretical results, additional algorithmic details, and further numerical evaluations.

\section{Additional Methodological Details}

\subsection{Projection-Based Implementation}
In Algorithms 2 and Algorithm 1 below in A.3, the constraint $\operatorname{conv}(\Delta_\alpha)$ is handled by appending a penalty term to the objective function. In this section, we present an alternative approach based on direct projection. Specifically, we solve the projection problem onto the capped simplex $\operatorname{conv}(\Delta_\alpha)$:
\begin{align*}
\operatorname{proj}_{\operatorname{conv}(\Delta_\alpha)}(\mathbf{b}_r) = \argmin_{\mathbf{z} \in \operatorname{conv}(\Delta_\alpha)} \frac{1}{2} \|\mathbf{z} - \mathbf{b}_r\|^2.
\end{align*}
To solve this, we employ the iterative algorithm proposed by \citet{adam2022projections}, which utilizes a bisection method to find the root of the optimality equation established in their Theorem 3.8. Crucially, the operations involved in this projection are differentiable almost everywhere, enabling seamless integration into gradient-based optimization via automatic differentiation (e.g., PyTorch's \emph{AutoGrad}).

Specifically, for the Sinkhorn Wasserstein Filtering (\texttt{SinkWF}) algorithm, the optimal transport problem and its associated gradient computation are modified as follows:
\begin{gather*}
  \mathbf{T}_{r} = \argmin_{\mathbf{T} \in \Pi\left(\frac{1}{n}\mathbf{1}, \, \operatorname{proj}_{\operatorname{conv}(\Delta_\alpha)}(\mathbf{b}_r) \right)} \langle \mathbf{T}, \mathbf{C} \rangle + \lambda \sum_{i,j \in [n]} t_{ij} \log(t_{ij}),
  \\
  \mathbf{g}_r = \mathbf{v}^* \frac{\partial \operatorname{proj}_{\operatorname{conv}(\Delta_\alpha)}(\mathbf{b}_r)}{\partial \mathbf{b}_r}.
\end{gather*}
This adaptation can be applied analogously to the Sliced Wasserstein Filtering (\texttt{SlicedWF}) algorithm:
\begin{align*}
  W_{1,S}\left(\widetilde{P}_n, P_{\operatorname{proj}_{\operatorname{conv}(\Delta_\alpha)}(\mathbf{b}_r)}\right) &= \frac{1}{S} \sum_{\mathbf{u}_i \sim \mathbb{U}^d, \, i \in [S]} \int_{0}^1 \left| F^{-1}_{\mathbf{u}_i^\top \widetilde{\mathbf{x}}, \, \widetilde{\mathbf{x}} \sim \widetilde{P}_n}(t) - F^{-1}_{\mathbf{u}_i^\top \mathbf{y}, \, \mathbf{y} \sim P_{\operatorname{proj}_{\operatorname{conv}(\Delta_\alpha)}(\mathbf{b}_r)}}(t) \right| dt,
  \\
  \mathbf{g}_r & = \frac{\partial W_{1,S}\left(\widetilde{P}_n, P_{\operatorname{proj}_{\operatorname{conv}(\Delta_\alpha)}(\mathbf{b}_r)}\right)}{\partial \operatorname{proj}_{\operatorname{conv}(\Delta_\alpha)}(\mathbf{b}_r)} \frac{\partial \operatorname{proj}_{\operatorname{conv}(\Delta_\alpha)}(\mathbf{b}_r)}{\partial \mathbf{b}_r}.
\end{align*}

By eliminating the penalty parameter $\beta$, this projection-based implementation bypasses the need for hyperparameter tuning on $\beta$, making it user-friendly in practice. We empirically evaluate the performance of this projection-based variant in Appendices C.1-C.3.

\subsection{Automatic Determination of the Contamination Level}
\label{app:hyperparameters}

In Algorithms 2 and Algorithm 1 below in A.3, the contamination level $\alpha$ is an input parameter that typically must be specified a priori. To address settings where the ground-truth contamination rate is unknown, we introduce an automatic selection strategy that dynamically estimates $\alpha$ during the optimization process.

Taking Algorithm 2 as an example, starting from an initial conservative estimate $\alpha_0$ (e.g., $\alpha_0 = 0.1$), we update the contamination level iteratively. Let $\mathbf{p}_{r+1} = \operatorname{Sort}(\operatorname{softmax}(\mathbf{b}_{r+1}))$ denote the vector of sample weights sorted in descending order at iteration $r+1$. We then update $\alpha_{r+1}$ according to:
\begin{align*}
\alpha_{r+1} = 1 - \frac{\argmax_{i\in[n-1]} (p_{r+1,i} - p_{r+1,i+1})}{n}.
\end{align*}

The intuition behind this strategy is that clean and contaminated samples tend to be separated by a sharp transition in their assigned probability masses. By identifying the index of the largest adjacent gap (i.e., the ``elbow" point) in the sorted weights, we can dynamically and adaptively partition the dataset into clean and contaminated subsets. The empirical performance of this adaptive $\alpha$ strategy is evaluated on benchmark datasets in Appendix C.

\subsection{SlicedWF Algorithm}
The complete optimization process for the SlicedWF algorithm is given below.

\begin{algorithm}[H]
    \DontPrintSemicolon
    \SetAlgoLined
 \caption[A]{Sliced Wasserstein Filtering Algorithm (\texttt{SlicedWF})}
    \label{alg:sliced_wf}
    \KwIn{$\widetilde{\mathcal{D}}$, $\alpha \in (0,1)$, $\mathbf{b}_0$, $\gamma \in\mathbb{R}_+$, $R, T, S\in \mathbb{Z}_+$}

    Normalize $\widetilde{\mathcal{D}}$ \;
    \For{$t=1, \dots, T$}{
        \For{$r = 0, \dots, R-1$}{
            Sample $S$ random directions $\{\mathbf{u}_1, \dots, \mathbf{u}_S\}$ uniformly from $\mathbb{S}^{d-1}$ \;
            Compute $W_{1,S}(\widetilde{P}_n, P_{\softmax(\mathbf{b}_r)})$ via:
            \begin{align*}
            \frac{1}{S} \sum_{i=1}^S \int_{0}^1 \left| F^{-1}_{\mathbf{u}_i^\top \widetilde{\mathbf{x}}, \, \widetilde{\mathbf{x}} \sim \widetilde{P}_n}(t) - F^{-1}_{\mathbf{u}_i^\top \mathbf{y}, \, \mathbf{y} \sim P_{\softmax(\mathbf{b}_r)}}(t) \right| dt
            \end{align*}
            Compute the gradient $\mathbf{g}_r$:
            \begin{align*}
            \mathbf{g}_r = \frac{\partial \left[ W_{1,S}(\widetilde{P}_n, P_{\softmax(\mathbf{b}_r)}) - \beta \mathbf{1}^\top \ReLU\left( \softmax(\mathbf{b}_r) - \frac{1}{(1-\alpha)n}\mathbf{1} \right) \right]}{\partial \mathbf{b}_r}
            \end{align*}
            Update parameters using Adam:
            \begin{align*}
            \mathbf{b}_{r+1} = \mathrm{Adam}(\gamma, \mathbf{b}_r, \mathbf{g}_{j, j\leq r})
            \end{align*}
        }
        Save $\mathbf{h}_t = \softmax(\mathbf{b}_R)$ \;
    }
    Thresholding: filter $\widetilde{\mathcal{D}}$ to obtain $\mathcal{S} = \{\widetilde{\mathbf{x}}_i \mid \bar{h}_i > \tau_\alpha, \, i \in [n]\}$ where
    $\bar{h}_i$ is the $i$th entry of $\bar{\mathbf{h}} = \frac{1}{T} \sum_{t=1}^T \mathbf{h}_t$, $\tau_\alpha$ is the $\alpha$ quantile of $\bar{h}_i,i\in[n]$

    \KwOut{$\mathcal{S}$}
\end{algorithm}

\section{Additional Details for Section 3:  Theoretical Properties}

\begin{figure}
    \centering
    \includegraphics[width=0.7\textwidth]{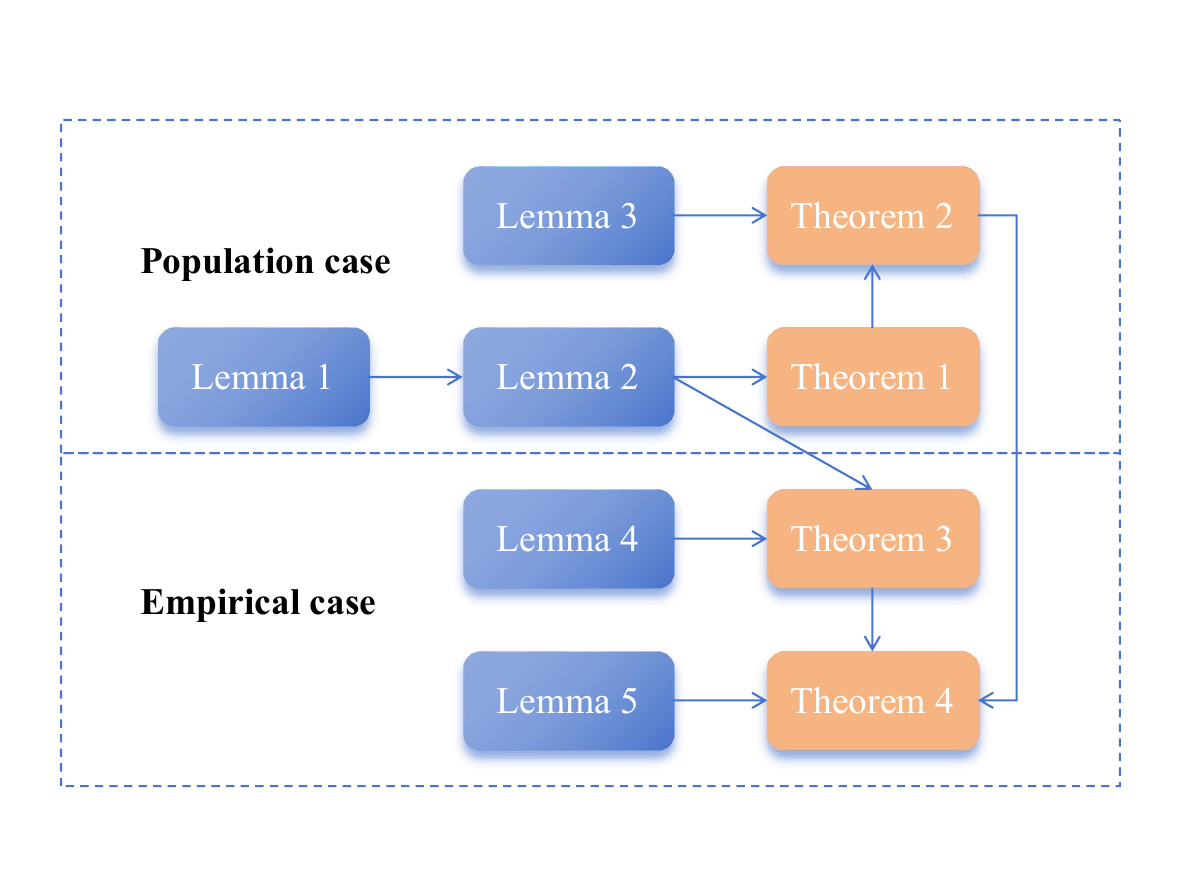}
    \caption{Roadmap of estimator convergence}
    % \label{fig:your_label}
\end{figure}

\subsection{Convergence of the Estimator}
\subsubsection{Population Setting}

\begin{lem}[Complementary Wasserstein Resilience]\label{lem:complementary_resilience}
Let $(\mathcal{X},d)$ be a metric space, $p \in [1,\infty)$, and $\alpha \in [0,1)$. If $\mu \in \mathcal{Q}_p(\mathcal{X})$ is $(\rho,\alpha)$-resilient with respect to the $W_p$ distance, then $\mu$ is also $\left( \frac{(1-\alpha)^{1/p}}{1-(1-\alpha)^{1/p}} \rho, \, 1-\alpha \right)$-resilient under $W_p$. That is, for every probability measure $\nu \in \mathcal{Q}_p(\mathcal{X})$ satisfying $\nu \le \frac{1}{\alpha} \mu$, we have:
\begin{align*}
W_p(\mu,\nu) \le \frac{(1-\alpha)^{1/p}}{1-(1-\alpha)^{1/p}}\rho.
\end{align*}
\end{lem}

\begin{proof}
Following the proof strategy of Lemma 10 in \citet{steinhardt2018resilience} (which establishes complementary mean resilience), let $\nu \in \mathcal{Q}_p(\mathcal{X})$ be any probability measure satisfying $\nu \le \frac{1}{\alpha} \mu$, and define
\begin{align*}
\eta := \frac{\mu - \alpha \nu}{1 - \alpha}.
\end{align*}
It is straightforward to verify that $\eta \in \mathcal{Q}_p(\mathcal{X})$ and $\eta \le \frac{1}{1-\alpha}\mu$.

Let $\pi$ be an optimal coupling between $\eta$ and $\nu$ with respect to the $W_p$ distance. We construct a coupling $\Gamma_1$ defined as:
\begin{align*}
\Gamma_1 := \alpha(\operatorname{Id}, \operatorname{Id})_{\#} \nu + (1-\alpha)\pi,
\end{align*}
where $\operatorname{Id}$ denotes the identity map. Since the first marginal of $\Gamma_1$ is $\alpha\nu + (1-\alpha)\eta = \mu$ and its second marginal is $\nu$, it constitutes a valid coupling between $\mu$ and $\nu$. Consequently, we obtain:
\begin{align*}
W_p^p(\mu,\nu) &\le \int d(\mathbf{x},\mathbf{y})^p \, d\Gamma_1(\mathbf{x},\mathbf{y}) \\
&= (1-\alpha) \int d(\mathbf{x},\mathbf{y})^p \, d\pi(\mathbf{x},\mathbf{y}) \\
&= (1-\alpha) W_p^p(\nu,\eta).
\end{align*}
Taking the $p$-th root on both sides yields $W_p(\mu,\nu) \le (1-\alpha)^{1/p} W_p(\nu,\eta)$. By the triangle inequality, we have:
\begin{align*}
W_p(\mu,\nu) &\le (1-\alpha)^{1/p} \left( W_p(\mu,\nu) + W_p(\mu,\eta) \right).
\end{align*}
Rearranging terms, we find:
\begin{align*}
\left( 1 - (1-\alpha)^{1/p} \right) W_p(\mu,\nu) &\le (1-\alpha)^{1/p} W_p(\mu,\eta) \\
&\le (1-\alpha)^{1/p} \rho,
\end{align*}
where the final inequality follows from the relation $\eta \le \frac{1}{1-\alpha}\mu$ and the assumed $(\rho,\alpha)$-resilience of $\mu$. This completes the proof.
\end{proof}

\begin{lem}[Risk under Support Total Variation Contamination]\label{lem:population_risk_under_STV}
Assume the true distribution $P^* \in \mathcal{G} \subseteq \mathcal{W}_1(\rho,\alpha)$ and the contaminated distribution $\widetilde{P}$ satisfies $\|\widetilde{P}-P^*\|_{\operatorname{STV}} \le \alpha$ for some $0 \le \alpha \le \frac{1}{2}$. Then, the population Wasserstein Filtering estimator $\widehat{P}_{\operatorname{WF}}$ satisfies:
\begin{align*}
W_1(\widehat{P}_{\operatorname{WF}},P^*) \lesssim \kappa_{\widetilde{P},P^*}(S_{\operatorname{WF}}) \vee \rho,
\end{align*}
where the discrepancy term $\kappa_{\widetilde{P},P^*}(S_{\operatorname{WF}})$ is defined as:
\begin{align*}
\kappa_{\widetilde{P},P^*}(S_{\operatorname{WF}}) = \inf_{\mu \le \frac{1}{\widetilde{P}(S_{\operatorname{WF}} \setminus \bar{S})} P^*} W_1\left(\widetilde{P}_{S_{\operatorname{WF}} \setminus \bar{S}}, \, \mu\right),
\end{align*}
and $\bar{S}$ is the shared support guaranteed by the definition of the support total variation distance.
\end{lem}

\begin{proof}
Since $\|\widetilde{P}-P^*\|_{\operatorname{STV}} \le \alpha$, by the definition of the Support Total Variation (STV) distance, there exists a measurable subset $\bar{S} \subseteq \mathcal{X}$ such that $\widetilde{P}(\bar{S}) = P^*(\bar{S}) = 1-\alpha$ and $\widetilde{P}|_{\bar{S}} = P^*|_{\bar{S}}$. Consequently, we can write:
\begin{align*}
W_1(\widehat{P}_{\operatorname{WF}}, P^*)
&\le W_1(\widehat{P}_{\operatorname{WF}}, \widetilde{P}_{\bar{S}}) + W_1(\widetilde{P}_{\bar{S}}, P^*) \\
&= W_1(\widehat{P}_{\operatorname{WF}}, \widetilde{P}_{\bar{S}}) + W_1(P^*_{\bar{S}}, P^*) \\
&\le W_1(\widehat{P}_{\operatorname{WF}}, \widetilde{P}_{\bar{S}}) + \rho,
\end{align*}
where the last inequality follows from the relation $P^*_{\bar{S}} \le \frac{1}{1-\alpha} P^*$ and the assumed $(\rho,\alpha)$-resilience of $P^*$.

Recall the Kantorovich-Rubinstein dual formulation for $W_1$. Let $(\phi^*, S_{\operatorname{WF}})$ be an optimal pair realizing $W_1(\widehat{P}_{\operatorname{WF}}, \widetilde{P})$, where $\phi^*$ is the optimal $1$-Lipschitz dual potential. This yields:
\begin{align*}
W_1(\widehat{P}_{\operatorname{WF}}, \widetilde{P})
&= \sup_{\|\phi\|_{\operatorname{Lip}} \le 1} \int \phi \, d\widehat{P}_{\operatorname{WF}} - \int \phi \, d\widetilde{P}  \\
&= \int \phi^* \, d\widehat{P}_{\operatorname{WF}} - \int \phi^* \, d\widetilde{P} \\
&= \sup_{S \subseteq \operatorname{supp}(\widetilde{P}), \, \widetilde{P}(S) = 1-\alpha} \int \phi^* \, d\widetilde{P}_S - \int \phi^* \, d\widetilde{P},
\end{align*}
where the third equality follows directly from the definition of the estimator $\widehat{P}_{\operatorname{WF}}$. Thus, for any subset $S \subseteq \operatorname{supp}(\widetilde{P})$ with $\widetilde{P}(S) = 1-\alpha$, we have:
\begin{align} \label{eq:1}
\int \phi^* \, d\widetilde{P}_S \le \int \phi^* \, d\widehat{P}_{\operatorname{WF}}.
\end{align}

Now, let $\psi$ be any $1$-Lipschitz function ($\|\psi\|_{\operatorname{Lip}} \le 1$). We observe that:
\begin{align*}
\int \psi \, d\widehat{P}_{\operatorname{WF}} - \int \psi \, d\widetilde{P}_{\bar{S}}
&\le \int \psi \, d\widehat{P}_{\operatorname{WF}} - \int \psi \, d\widetilde{P}_{\bar{S}} + \int \phi^* \, d\widehat{P}_{\operatorname{WF}} - \int \phi^* \, d\widetilde{P}_{\bar{S}} \\
&= \frac{1}{1-\alpha} \left( \int_{S_{\operatorname{WF}}} (\psi + \phi^*) \, d\widetilde{P} - \int_{\bar{S}} (\psi + \phi^*) \, d\widetilde{P} \right) \\
&= \frac{1}{1-\alpha} \left( \int_{S_{\operatorname{WF}} \setminus \bar{S}} (\psi + \phi^*) \, d\widetilde{P} - \int_{\bar{S} \setminus S_{\operatorname{WF}}} (\psi + \phi^*) \, d\widetilde{P} \right) \\
&= \frac{\widetilde{P}(S_{\operatorname{WF}} \setminus \bar{S})}{1-\alpha} \left( \int (\psi + \phi^*) \, d\widetilde{P}_{S_{\operatorname{WF}} \setminus \bar{S}} - \int (\psi + \phi^*) \, d\widetilde{P}_{\bar{S} \setminus S_{\operatorname{WF}}} \right) \\
&= \frac{\delta}{1-\alpha} \left( \int (\psi(\mathbf{x}) + \phi^*(\mathbf{x})) \, d\widetilde{P}_{S_{\operatorname{WF}} \setminus \bar{S}}(\mathbf{x}) - \int (\psi(\mathbf{y}) + \phi^*(\mathbf{y})) \, d\widetilde{P}_{\bar{S} \setminus S_{\operatorname{WF}}}(\mathbf{y}) \right) \\
&= \frac{\delta}{1-\alpha} \int \left( \psi(\mathbf{x}) - \psi(\mathbf{y}) + \phi^*(\mathbf{x}) - \phi^*(\mathbf{y}) \right) d\pi(\mathbf{x}, \mathbf{y}) \\
&\le \frac{2\delta}{1-\alpha} \int d(\mathbf{x}, \mathbf{y}) \, d\pi(\mathbf{x}, \mathbf{y}),
\end{align*}
where the first inequality is obtained by setting $\widetilde{P}_{S} = \widetilde{P}_{\bar{S}}$ in \eqref{eq:1};
the third equality follows because $\widetilde{P}(S_{\operatorname{WF}}) = \widetilde{P}(\bar{S}) = 1-\alpha$, which implies $\widetilde{P}(S_{\operatorname{WF}} \setminus \bar{S}) = \widetilde{P}(\bar{S} \setminus S_{\operatorname{WF}}) := \delta \le \alpha$;
the fifth equality holds for any arbitrary coupling $\pi$ of $\widetilde{P}_{S_{\operatorname{WF}} \setminus \bar{S}}$ and $\widetilde{P}_{\bar{S} \setminus S_{\operatorname{WF}}}$;
the last inequality is due to the Lipschitz constraints $\|\psi\|_{\operatorname{Lip}} \le 1$ and $\|\phi^*\|_{\operatorname{Lip}} \le 1$.

Taking the supremum over $\psi$ on the left-hand side and the infimum over all valid couplings $\pi \in \Pi(\widetilde{P}_{S_{\operatorname{WF}} \setminus \bar{S}}, \widetilde{P}_{\bar{S} \setminus S_{\operatorname{WF}}})$ on the right-hand side, we get:
\begin{align*}
W_1(\widehat{P}_{\operatorname{WF}}, \widetilde{P}_{\bar{S}})
&= \sup_{\|\psi\|_{\operatorname{Lip}} \le 1} \int \psi \, d\widehat{P}_{\operatorname{WF}} - \int \psi \, d\widetilde{P}_{\bar{S}} \\
&\le \frac{2\delta}{1-\alpha} \inf_{\pi \in \Pi(\widetilde{P}_{S_{\operatorname{WF}} \setminus \bar{S}}, \widetilde{P}_{\bar{S} \setminus S_{\operatorname{WF}}})} \int d(\mathbf{x}, \mathbf{y}) \, d\pi(\mathbf{x}, \mathbf{y}) \\
&= \frac{2\delta}{1-\alpha} W_1\left(\widetilde{P}_{S_{\operatorname{WF}} \setminus \bar{S}}, \widetilde{P}_{\bar{S} \setminus S_{\operatorname{WF}}}\right) \\
&\le \frac{2\delta}{1-\alpha} \left\{ \inf_{\mu \le \frac{1}{\delta} P^*} W_1\left(\widetilde{P}_{S_{\operatorname{WF}} \setminus \bar{S}}, \mu\right) + W_1\left(\mu, \widetilde{P}_{\bar{S} \setminus S_{\operatorname{WF}}}\right) \right\} \\
&\le \frac{2\delta}{1-\alpha} \left\{ \inf_{\mu \le \frac{1}{\delta} P^*} W_1\left(\widetilde{P}_{S_{\operatorname{WF}} \setminus \bar{S}}, \mu\right) + \sup_{\mu \le \frac{1}{\delta} P^*, \, V \subseteq \bar{S}, \, P^*(V) = \delta} W_1(\mu, P^*_{V}) \right\} \\
&\le \frac{2\delta}{1-\alpha} \left\{ \inf_{\mu \le \frac{1}{\delta} P^*} W_1\left(\widetilde{P}_{S_{\operatorname{WF}} \setminus \bar{S}}, \mu\right) + \sup_{\mu, \nu \le \frac{1}{\delta} P^*} W_1(\mu, \nu) \right\} \\
&\le \frac{2\delta}{1-\alpha} \left\{ \inf_{\mu \le \frac{1}{\delta} P^*} W_1\left(\widetilde{P}_{S_{\operatorname{WF}} \setminus \bar{S}}, \mu\right) + \frac{2(1-\delta)}{\delta} \rho \right\} \\
&\le 2 \kappa_{\widetilde{P}, P^*}(S_{\operatorname{WF}}) + 8\rho,
\end{align*}
where
the third inequality uses the fact that $\widetilde{P} = P^*$ on the support $\bar{S}$;
 the fifth inequality follows from Lemma \ref{lem:complementary_resilience};
 the final inequality is verified by noting that $\delta \le \alpha \le \frac{1}{2}$ (which implies $\frac{2\delta}{1-\alpha} \le 2$ and $\frac{4(1-\delta)}{1-\alpha} \le 8$).

Combining this result with the initial bound $W_1(\widehat{P}_{\operatorname{WF}}, P^*) \le W_1(\widehat{P}_{\operatorname{WF}}, \widetilde{P}_{\bar{S}}) + \rho$ completes the proof.
\end{proof}

We are now ready to prove establish the population risk bounds.

\begin{proof}[Proof of Theorem \ref{thm:population_risk}]
Suppose for contradiction that $\widetilde{P}(S_{\operatorname{WF}} \cap E_\rho) > 0$. By \eqref{eq:3} in the definition of population FELP contamination, there exists a subset $S_+ \subseteq \operatorname{supp}(\widetilde{P})$ with $\widetilde{P}(S_+) = 1-\alpha$ and $S_+ \cap E_\rho = \emptyset$ such that:
\begin{align*}
W_1(\widehat{P}_{\operatorname{WF}}, \widetilde{P}) < W_1(\widetilde{P}_{S_+}, \widetilde{P}).
\end{align*}
Since $\widetilde{P}_{S_+}$ is a valid $(1-\alpha)$-truncated version of $\widetilde{P}$, this strictly contradicts the optimality of $S_{\operatorname{WF}}$. Therefore, we must have $S_{\operatorname{WF}} \cap E_\rho = \emptyset$, which in turn implies $(S_{\operatorname{WF}} \setminus \bar{S}) \cap E_\rho = \emptyset$. Since $S_{\operatorname{WF}} \setminus \bar{S} \subseteq E_\rho \cup N_\rho$, it follows that $S_{\operatorname{WF}} \setminus \bar{S} \subseteq N_\rho$.

Let $\widetilde{P}(S_{\operatorname{WF}} \setminus \bar{S}) = \delta$, where $0 \le \delta \le \alpha$. We define the push-forward measure $\mu = T_{\#} \widetilde{P}_{S_{\operatorname{WF}} \setminus \bar{S}}$. For any measurable set $B \subseteq \mathcal{X}$, we have:
\begin{align*}
\mu(B) &= \widetilde{P}_{S_{\operatorname{WF}} \setminus \bar{S}}(T^{-1}(B)) \\
&\le \frac{1}{\delta} \widetilde{P}|_{N_\rho}(T^{-1}(B)) \\
&= \frac{1}{\delta} T_{\#} \widetilde{P}|_{N_\rho}(B) \\
&\le \frac{1}{\delta} P^*(B),
\end{align*}
where the last inequality follows from \eqref{eq:2}, thereby verifying that $\mu \le \frac{1}{\delta} P^*$.

Next, let $\pi_0 = (\operatorname{Id}, T)_{\#} \widetilde{P}_{S_{\operatorname{WF}} \setminus \bar{S}}$. It is clear that $\pi_0 \in \Pi(\widetilde{P}_{S_{\operatorname{WF}} \setminus \bar{S}}, \mu)$ represents a valid coupling. Consequently, we obtain:
\begin{align*}
W_1(\widetilde{P}_{S_{\operatorname{WF}} \setminus \bar{S}}, \mu) &= \inf_{\pi \in \Pi(\widetilde{P}_{S_{\operatorname{WF}} \setminus \bar{S}}, \mu)} \int d(\mathbf{x}, \mathbf{y}) \, d\pi(\mathbf{x}, \mathbf{y}) \\
&\le \int d(\mathbf{x}, \mathbf{y}) \, d\pi_0(\mathbf{x}, \mathbf{y}) \\
&= \int d(\mathbf{x}, T(\mathbf{x})) \, d\widetilde{P}_{S_{\operatorname{WF}} \setminus \bar{S}}(\mathbf{x}) \\
&\le \rho.
\end{align*}
This shows there exists a measure $\mu = T_{\#} \widetilde{P}_{S_{\operatorname{WF}} \setminus \bar{S}} \le \frac{1}{\delta} P^*$ such that $W_1(\mu, \widetilde{P}_{S_{\operatorname{WF}} \setminus \bar{S}}) \le \rho$, which directly implies that $\kappa_{\widetilde{P}, P^*}(S_{\operatorname{WF}}) \le \rho$. Applying Lemma \ref{lem:population_risk_under_STV}, we conclude that $W_1(\widehat{P}_{\operatorname{WF}}, P^*) \lesssim \rho$.
\end{proof}

To establish the minimax lower bound in Theorem~\ref{thm:minimax_optimal}, we first introduce the following auxiliary lemma, which provides a generic lower bound on the minimax risk via the modulus of continuity.

\begin{lem}[Generic Lower Bound for Minimax Risk]\label{lem:generic_lower_bound_for_minimax_risk}
Let $\mathcal{G} \subseteq \mathcal{W}_1(\rho,\alpha)$. The minimax estimation risk over $\mathcal{G}$ is lower bounded as:
\begin{align*}
\inf_{\mathcal{T}} \sup_{P^* \in \mathcal{G}} \mathcal{R}_{1,P^*,\mathcal{T}}(\alpha,\rho) \ge \frac{1}{2} \omega(\mathcal{G}, \alpha, \rho),
\end{align*}
where the modulus of continuity $\omega(\mathcal{G}, \alpha, \rho)$ is defined as:
\begin{align*}
\omega(\mathcal{G}, \alpha, \rho) = \sup \left\{ W_1(P_0, P_1) : P_0, P_1 \in \mathcal{G}, \, \mathcal{C}^{\operatorname{FELP}}_{P_0}(\alpha,\rho) \cap \mathcal{C}^{\operatorname{FELP}}_{P_1}(\alpha,\rho) \neq \emptyset \right\}.
\end{align*}
\end{lem}

\begin{proof}
Following the standard construction of the modulus of continuity in robust estimation \citep[see, e.g.,][]{donoho1988automatic, zhu2022generalized}, we fix an arbitrary estimator $\mathcal{T}$. For any $P_0, P_1 \in \mathcal{G}$ sharing a common contaminated distribution $\widetilde{P} \in \mathcal{C}^{\operatorname{FELP}}_{P_0}(\alpha,\rho) \cap \mathcal{C}^{\operatorname{FELP}}_{P_1}(\alpha,\rho)$, the triangle inequality yields:
\begin{align*}
W_1(P_0, P_1) &\le W_1(P_0, \mathcal{T}(\widetilde{P})) + W_1(\mathcal{T}(\widetilde{P}), P_1) \\
&\le 2 \max\left\{ W_1(P_0, \mathcal{T}(\widetilde{P})), \, W_1(P_1, \mathcal{T}(\widetilde{P})) \right\}.
\end{align*}
Taking the supremum over all such valid pairs $(P_0, P_1)$ and their shared contaminated distributions $\widetilde{P}$ on both sides, we obtain:
\begin{align*}
\omega(\mathcal{G}, \alpha, \rho) &= \sup_{P_0, P_1, \widetilde{P}} W_1(P_0, P_1) \\
&\le 2 \sup_{P_0, P_1, \widetilde{P}} \max\left\{ W_1(P_0, \mathcal{T}(\widetilde{P})), \, W_1(P_1, \mathcal{T}(\widetilde{P})) \right\} \\
&\le 2 \sup_{P^* \in \mathcal{G}} \sup_{\widetilde{P} \in \mathcal{C}^{\operatorname{FELP}}_{P^*}(\alpha,\rho)} W_1(P^*, \mathcal{T}(\widetilde{P})) \\
&= 2 \sup_{P^* \in \mathcal{G}} \mathcal{R}_{1,P^*,\mathcal{T}}(\alpha, \rho).
\end{align*}
Taking the infimum over all estimators $\mathcal{T}$ on both sides completes the proof.
\end{proof}

We are now ready to establish the minimax optimality results.

\begin{proof}[Proof of Theorem \ref{thm:minimax_optimal}]
We first consider the general setting where the distribution class $\mathcal{G}$ satisfies $\mathcal{G} \subseteq \mathcal{W}_1(\rho,\alpha)$.

Let $\mathbf{e}_1 = (1, 0, \dots, 0)^\top \in \mathbb{R}^d$ denote the first standard basis vector, and let $Q_{U\mathbf{e}_1}$ represent the uniform law on the line segment:
\begin{align*}
L_{\frac{\rho}{4}} := \left\{ t \mathbf{e}_1 : 0 \le t \le \frac{\rho}{4} \right\} \subset \mathbb{R}^d,
\end{align*}
where $U \sim \operatorname{Unif}[0, \frac{\rho}{4}]$. We define the candidate distributions $P_0, P_1$ and a shared contaminated distribution $\widetilde{P}$ as follows:
\begin{align*}
P_0 = (1-\alpha)Q_{U\mathbf{e}_1} + \alpha \delta_{-\frac{\rho}{2}\mathbf{e}_1}, \quad
P_1 = (1-\alpha)Q_{U\mathbf{e}_1} + \alpha \delta_{\frac{\rho}{2}\mathbf{e}_1}, \quad
\widetilde{P} = (1-\alpha)Q_{U\mathbf{e}_1} + \alpha \delta_{-\frac{\rho}{4}\mathbf{e}_1}.
\end{align*}

To apply Lemma \ref{lem:generic_lower_bound_for_minimax_risk}, we verify that $\widetilde{P} \in \mathcal{C}^{\operatorname{FELP}}_{P_0}(\alpha,\rho) \cap \mathcal{C}^{\operatorname{FELP}}_{P_1}(\alpha,\rho)$. By choosing the shared support $\bar{S} = L_{\frac{\rho}{4}}$, we have:
\begin{align*}
P_0(\bar{S}) = P_1(\bar{S}) = \widetilde{P}(\bar{S}) = 1-\alpha, \quad \text{and} \quad
P_0|_{\bar{S}} = P_1|_{\bar{S}} = \widetilde{P}|_{\bar{S}} = (1-\alpha)Q_{U\mathbf{e}_1},
\end{align*}
which immediately implies $\|\widetilde{P}-P_i\|_{\operatorname{STV}} \le \alpha$ for $i=0,1$.

Next, we check the Far Exclusion and Local Projection (FELP) conditions. The only contaminated point under $\widetilde{P}$ is $-\frac{\rho}{4}\mathbf{e}_1$, which satisfies:
\begin{align*}
d\left(-\frac{\rho}{4}\mathbf{e}_1, \, \operatorname{supp}(P_0)\right) = d\left(-\frac{\rho}{4}\mathbf{e}_1, \, \operatorname{supp}(P_1)\right) = \frac{\rho}{4} < \rho.
\end{align*}
Consequently, the far contamination region $E_\rho$ is empty for both $P_0$ and $P_1$. For the local projection condition, we define the transport maps $T_0$ and $T_1$ by setting $T_0\left(-\frac{\rho}{4}\mathbf{e}_1\right) = -\frac{\rho}{2}\mathbf{e}_1$ and $T_1\left(-\frac{\rho}{4}\mathbf{e}_1\right) = \frac{\rho}{2}\mathbf{e}_1$. These maps satisfy:
\begin{align*}
d\left(-\frac{\rho}{4}\mathbf{e}_1, \, T_0\left(-\frac{\rho}{4}\mathbf{e}_1\right)\right) = \frac{\rho}{4} < \rho, \quad
d\left(-\frac{\rho}{4}\mathbf{e}_1, \, T_1\left(-\frac{\rho}{4}\mathbf{e}_1\right)\right) = \frac{3\rho}{4} < \rho,
\end{align*}
as well as $T_{0\#}\widetilde{P}|_{N_\rho} = \alpha\delta_{-\frac{\rho}{2}\mathbf{e}_1} \le P_0$ and $T_{1\#}\widetilde{P}|_{N_\rho} = \alpha\delta_{\frac{\rho}{2}\mathbf{e}_1} \le P_1$. This confirms that $\widetilde{P} \in \mathcal{C}^{\operatorname{FELP}}_{P_0}(\alpha,\rho) \cap \mathcal{C}^{\operatorname{FELP}}_{P_1}(\alpha,\rho)$.

Assuming $P_0, P_1 \in \mathcal{G}$, Lemma \ref{lem:generic_lower_bound_for_minimax_risk} yields the following minimax lower bound:
\begin{align*}
\inf_{\mathcal{T}} \sup_{P^* \in \mathcal{G}} \mathcal{R}_{1,P^*,\mathcal{T}}(\alpha,\rho)
\ge \frac{1}{2} \omega(\mathcal{G}, \alpha, \rho)
\ge \frac{1}{2} W_1(P_0, P_1)
= \frac{1}{2} \alpha \rho
\gtrsim \rho,
\end{align*}
where the final inequality holds since $\alpha \ge \alpha_0 > 0$. Combined with the upper bound from Theorem \ref{thm:population_risk}, which establishes that
\begin{align*}
\inf_{\mathcal{T}} \sup_{P^* \in \mathcal{G}} \mathcal{R}_{1,P^*,\mathcal{T}}(\alpha,\rho)
\le \sup_{P^* \in \mathcal{G}} \mathcal{R}_{1,P^*,\widehat{P}_{\operatorname{WF}}}(\alpha,\rho)
\lesssim \rho,
\end{align*}
we obtain the tight minimax rate of $\Theta(\rho)$.

We now turn to the specialized case where $\mathcal{G} = \mathcal{G}_{\operatorname{cov}}(\sigma)$. It remains to verify that $P_0, P_1 \in \mathcal{G}_{\operatorname{cov}}(\sigma)$. For a general distribution of the form $P_{t,\eta} = (1-\alpha)Q_{U\mathbf{e}_1} + \alpha \delta_{t\mathbf{e}_1}$ with $U \sim \operatorname{Unif}[0, \eta]$, its covariance matrix is given by:
\begin{align*}
\mathbf{\Sigma}_{P_{t,\eta}} = \left( (1-\alpha)\frac{\eta^2}{12} + \alpha(1-\alpha)\left(t-\frac{\eta}{2}\right)^2 \right) \mathbf{e}_1 \mathbf{e}_1^\top.
\end{align*}
For $P_0$, we have $\eta = \frac{\rho}{4}$ and $t = -\frac{\rho}{2}$, while for $P_1$, we have $\eta = \frac{\rho}{4}$ and $t = \frac{\rho}{2}$. To ensure that $P_0, P_1 \in \mathcal{G}_{\operatorname{cov}}(\sigma)$, we require the operator norm of the covariance to be bounded by $\sigma^2$:
\begin{align*}
(1-\alpha)\frac{\rho^2}{192} + \alpha(1-\alpha)\left(\frac{\rho}{2} + \frac{\rho}{8}\right)^2 \le \sigma^2
\implies \rho \lesssim \frac{\sigma}{\sqrt{1-\alpha^2}}.
\end{align*}
Furthermore, by Theorem 2 of \citet{nietert2023robust}, the resilience parameter under $\mathcal{G}_{\operatorname{cov}}(\sigma)$ scales as $\rho \asymp \sigma \sqrt{d \alpha}$. Substituting this relation, we find that the covariance constraint is satisfied provided that the dimension satisfies $d \lesssim \frac{1}{\alpha(1-\alpha^2)}$, which holds by assumption. This completes the proof.
\end{proof}

\subsubsection{Empirical Setting}

\begin{lem}[Upper Bound on Empirical Resilience]\label{lem:empirical_resilience}
Let $P^* \in \mathcal{Q}(\mathcal{X})$ be a probability measure, let $P_n = \frac{1}{n}\sum_{i=1}^n \delta_{\mathbf{x}_i}$ be its empirical measure, and let $0 \le \alpha < 1$. Define the population and empirical resilience parameters respectively as:
\begin{align*}
\rho = \sup_{\mu \le \frac{1}{1-\alpha} P^*} W_p(P^*, \mu) \quad \text{and} \quad \rho_n = \sup_{\mu \le \frac{1}{1-\alpha} P_n} W_p(P_n, \mu).
\end{align*}
Then, the empirical resilience parameter satisfies:
\begin{align*}
\rho_n \le \rho + \left(1 + \frac{1}{(1-\alpha)^{1/p}}\right) W_p(P_n, P^*).
\end{align*}
\end{lem}

\begin{proof}
For any $\alpha$-trimmed version of $P_n$ satisfying $\mu_\alpha \le \frac{1}{1-\alpha} P_n$, the Radon–Nikodym theorem guarantees the existence of a measurable function $f \colon \mathcal{X} \to \left[0, \frac{1}{1-\alpha}\right]$ with $\int f \, dP_n = 1$ such that:
\begin{align*}
\mu_\alpha(A) = \int_A f \, dP_n
\end{align*}
for all Borel measurable sets $A \subseteq \mathcal{X}$.

Let $\pi \in \Pi(P_n, P^*)$ be an optimal coupling between $P_n$ and $P^*$. We construct a new probability measure $\bar{\pi}$ on $\mathcal{X}^2$ by defining, for any Borel set $B \subseteq \mathcal{X}^2$:
\begin{align*}
\bar{\pi}(B) = \int_B f(\mathbf{x}) \, d\pi(\mathbf{x}, \mathbf{y}).
\end{align*}
For any Borel set $A \subseteq \mathcal{X}$, the first marginal of $\bar{\pi}$ satisfies:
\begin{align*}
\bar{\pi}(A \times \mathcal{X}) = \int_{A \times \mathcal{X}} f(\mathbf{x}) \, d\pi(\mathbf{x}, \mathbf{y}) = \int_{A} f(\mathbf{x}) \, dP_n(\mathbf{x}) = \mu_\alpha(A).
\end{align*}
Letting $\nu_\alpha$ denote the second marginal of $\bar{\pi}$, we have for any Borel set $B \subseteq \mathcal{X}$:
\begin{align*}
\nu_\alpha(B) = \bar{\pi}(\mathcal{X} \times B) = \int_{\mathcal{X} \times B} f(\mathbf{x}) \, d\pi(\mathbf{x}, \mathbf{y}) \le \frac{1}{1-\alpha} P^*(B),
\end{align*}
which confirms that $\nu_\alpha$ is indeed a feasible $\alpha$-trimmed version of $P^*$.

Using $\bar{\pi}$ as a (possibly sub-optimal) coupling between $\mu_\alpha$ and $\nu_\alpha$, we bound their Wasserstein distance as:
\begin{align*}
W_p(\mu_\alpha, \nu_\alpha)
&= \left( \inf_{\pi_0 \in \Pi(\mu_\alpha, \nu_\alpha)} \int \|\mathbf{x}-\mathbf{y}\|^p \, d\pi_0(\mathbf{x},\mathbf{y}) \right)^{1/p} \\
&\le \left( \int \|\mathbf{x}-\mathbf{y}\|^p \, d\bar{\pi}(\mathbf{x},\mathbf{y}) \right)^{1/p} \\
&\le \left( \frac{1}{1-\alpha} \int \|\mathbf{x}-\mathbf{y}\|^p \, d\pi(\mathbf{x},\mathbf{y}) \right)^{1/p} \\
&= \frac{1}{(1-\alpha)^{1/p}} W_p(P_n, P^*).
\end{align*}
Applying the triangle inequality for the Wasserstein metric, we obtain:
\begin{align*}
W_p(P_n, \mu_{\alpha})
&\le W_p(P_n, P^*) + W_p(P^*, \nu_\alpha) + W_p(\nu_\alpha, \mu_\alpha) \\
&\le W_p(P^*, \nu_\alpha) + \left(1 + \frac{1}{(1-\alpha)^{1/p}}\right) W_p(P_n, P^*) \\
&\le \sup_{\nu \le \frac{1}{1-\alpha} P^*} W_p(P^*, \nu) + \left(1 + \frac{1}{(1-\alpha)^{1/p}}\right) W_p(P_n, P^*).
\end{align*}
Since this bound holds for any arbitrary $\alpha$-trimmed measure $\mu_\alpha \le \frac{1}{1-\alpha} P_n$, taking the supremum over all such $\mu_\alpha$ yields:
\begin{align} \label{eq:4}
\sup_{\mu \le \frac{1}{1-\alpha} P_n} W_p(P_n, \mu) \le \sup_{\nu \le \frac{1}{1-\alpha} P^*} W_p(P^*, \nu) + \left(1 + \frac{1}{(1-\alpha)^{1/p}}\right) W_p(P_n, P^*),
\end{align}
which completes the proof.
\end{proof}

We are now ready to establish the empirical risk results.

\begin{proof}[Proof of Theorem \ref{thm:empirical_risk}]
By setting $p=1$ in Lemma \ref{lem:empirical_resilience} and using the assumption that $\alpha \le \frac{1}{2}$, we obtain $P_n \in \mathcal{W}_1(\mathcal{O}(\rho + W_1(P_n, P^*)), \alpha)$.

Using the triangle inequality, we write $W_1(\widehat{P}_{\operatorname{WF}}, P^*) \le W_1(\widehat{P}_{\operatorname{WF}}, P_n) + W_1(P_n, P^*)$. Since $\|\widetilde{P}_n - P_n\|_{\operatorname{STV}} \le \alpha$, we can apply Lemma \ref{lem:population_risk_under_STV} by treating $\widetilde{P}_n$ as the contaminated distribution $\widetilde{P}$ and the empirical measure $P_n$ as the clean target $P^*$. This yields:
\begin{align*}
W_1(\widehat{P}_{\operatorname{WF}}, P^*) \le W_1(\widehat{P}_{\operatorname{WF}}, P_n) + W_1(P_n, P^*) \lesssim \kappa_{\widetilde{P}_n, P_n}(S_{\operatorname{WF}}) \vee \left( \rho + W_1(P_n, P^*) \right),
\end{align*}
which can be explicitly written as:
\begin{align*}
W_1(\widehat{P}_{\operatorname{WF}}, P^*) \lesssim \max \left\{ \inf_{\mu \le \frac{1}{\widetilde{P}_n(S_{\operatorname{WF}} \setminus \bar{S})} P_n} W_1\left(\widetilde{P}_{nS_{\operatorname{WF}} \setminus \bar{S}}, \, \mu\right), \ \rho + W_1(P_n, P^*) \right\},
\end{align*}
where $\bar{S} \subseteq \{\mathbf{x}_j : j \in [n], \, \mathbf{x}_j = \widetilde{\mathbf{x}}_j\}$ is the set of uncorrupted samples, satisfying $|\bar{S}| = (1-\alpha)n$. Taking the expectation on both sides with respect to the contaminated data generating distribution $\widetilde{\mathbb{P}}_n$, we get:
\begin{align} \label{eq:5}
\mathbb{E}_{\widetilde{\mathbb{P}}_n} \left[ W_1(\widehat{P}_{\operatorname{WF}}, P^*) \right] \lesssim \max \left\{ \mathbb{E}_{\widetilde{\mathbb{P}}_n} \left[ \inf_{\mu \le \frac{1}{\widetilde{P}_n(S_{\operatorname{WF}} \setminus \bar{S})} P_n} W_1\left(\widetilde{P}_{nS_{\operatorname{WF}} \setminus \bar{S}}, \, \mu\right) \right], \ \rho + \mathbb{E}_{\mathbb{P}_n}\left[ W_1(P_n, P^*) \right] \right\}.
\end{align}

Next, by substituting the effective empirical parameter $\mathcal{O}(\rho + W_1(P_n, P^*))$ in place of $\rho$ in the definition of the empirical FELP condition, we can apply the exact same conditioning and projection arguments as in the proof of Theorem \ref{thm:population_risk}. This directly implies:
\begin{align*}
\inf_{\mu \le \frac{1}{\widetilde{P}_n(S_{\operatorname{WF}} \setminus \bar{S})} P_n} W_1\left(\widetilde{P}_{nS_{\operatorname{WF}} \setminus \bar{S}}, \, \mu\right) \lesssim \rho + W_1(P_n, P^*).
\end{align*}
Taking the expectation with respect to $\widetilde{\mathbb{P}}_n$ on both sides and substituting this result back into \eqref{eq:5} completes the proof.
\end{proof}

\begin{lem}[Empirical Lifting under FELP Contamination]\label{lem:empirical_lifting}
For any $P^* \in \mathcal{G}$ and any $\widetilde{P} \in \mathcal{C}^{\operatorname{FELP}}_{P^*}\left(\frac{\alpha}{5}, \rho\right)$, the following properties hold:
\begin{enumerate}
    \item First, we have $T_{\#} \widetilde{P}|_{N_\rho} \le P^*|_{\bar{S}^c}$.
    \item Second, let $\eta := P^*|_{\bar{S}} = \widetilde{P}|_{\bar{S}}$, $\nu_{\operatorname{loc}} := \widetilde{P}|_{N_\rho}$, $\nu_{\operatorname{far}} := \widetilde{P}|_{E_\rho}$, $\lambda_{\operatorname{loc}} := T_{\#} \nu_{\operatorname{loc}}$, and $\lambda_{\operatorname{far}} := P^*|_{\bar{S}^c} - \lambda_{\operatorname{loc}} \ge 0$. For any coupling $\pi_{\operatorname{far}} \in \Pi(\lambda_{\operatorname{far}}, \nu_{\operatorname{far}})$, we define the joint coupling $\pi \in \Pi(P^*, \widetilde{P})$ as:
    \begin{align*}
    \pi = (\operatorname{Id}, \operatorname{Id})_{\#} \eta + (T, \operatorname{Id})_{\#} \nu_{\operatorname{loc}} + \pi_{\operatorname{far}}.
    \end{align*}
\end{enumerate}

Under the i.i.d.\ sample generation $(\mathbf{x}_i, \widetilde{\mathbf{x}}_i)_{i=1}^n \sim \pi^n$, we define the event $\mathcal{E}_2$ as:
\begin{align*}
\mathcal{E}_2 = \left\{
\begin{aligned}
&\exists S_+ \subseteq \{\widetilde{\mathbf{x}}_1, \dots, \widetilde{\mathbf{x}}_n\}, \, |S_+| = (1-\alpha)n, \, S_+ \cap E_{n,\rho} = \emptyset \text{ such that:} \\
&\sup_{\substack{S \subseteq \{\widetilde{\mathbf{x}}_1, \dots, \widetilde{\mathbf{x}}_n\} \\ |S| = (1-\alpha)n, \, S \cap E_{n,\rho} \neq \emptyset}} W_1\left(\frac{1}{|S|}\sum_{\widetilde{\mathbf{x}}_i \in S} \delta_{\widetilde{\mathbf{x}}_i}, \, \widetilde{P}_n\right) < W_1\left(\frac{1}{|S_+|}\sum_{\widetilde{\mathbf{x}}_i \in S_+} \delta_{\widetilde{\mathbf{x}}_i}, \, \widetilde{P}_n\right)
\end{aligned}
\right\}.
\end{align*}
Then, we have $\pi^n(\mathcal{E}_2^c) \le \delta_n \le \frac{1}{20}$.

Furthermore, letting $\mathcal{E}_1 = \left\{ \sum_{i=1}^n \mathbb{I}_{\{\mathbf{x}_i \neq \widetilde{\mathbf{x}}_i\}} \le \alpha n \right\}$, we have $\pi^n(\mathcal{E}_1 \cap \mathcal{E}_2) \ge \frac{3}{4}$, and the second (right) marginal distribution of the conditional joint measure $\pi^n(\cdot, \cdot \mid \mathcal{E}_1 \cap \mathcal{E}_2)$ belongs to $\mathcal{D}_n^{\operatorname{FELP}}(P^*, \alpha)$.
\end{lem}

\begin{proof}
We first bound the probability of the event $\mathcal{E}_1$. By applying Markov's inequality, we obtain:
\begin{align*}
\pi^n(\mathcal{E}_1) &= 1 - \pi^n(\mathcal{E}_1^c) \\
&\ge 1 - \frac{\mathbb{E}_{\pi^n}\left[\sum_{i=1}^n \mathbb{I}_{\{\mathbf{x}_i \neq \widetilde{\mathbf{x}}_i\}}\right]}{\alpha n} \\
&\ge \frac{4}{5},
\end{align*}
where the second inequality follows from the assumption that $\widetilde{P} \in \mathcal{C}_{P^*}^{\operatorname{FELP}}\left(\frac{\alpha}{5}, \rho\right)$ and the construction of the joint coupling $\pi$, which ensures that $\pi(\mathbf{x} \neq \widetilde{\mathbf{x}}) \le \frac{\alpha}{5}$. Applying the union bound, we obtain:
\begin{align*}
\pi^n(\mathcal{E}_1 \cap \mathcal{E}_2) \ge \pi^n(\mathcal{E}_1) - \pi^n(\mathcal{E}_2^c) \ge \frac{4}{5} - \delta_n \ge \frac{3}{4}.
\end{align*}

Next, we show that the second (right) marginal of the conditional distribution $\pi^n(\cdot, \cdot \mid \mathcal{E}_1 \cap \mathcal{E}_2)$ belongs to $\mathcal{D}_n^{\operatorname{FELP}}(P^*, \alpha)$. To establish this, it suffices to show that the empirical samples $\widetilde{\mathbf{x}}_1, \dots, \widetilde{\mathbf{x}}_n$ satisfy the empirical FELP contamination conditions, i.e., $\widetilde{P}_n \in \mathcal{C}_{P_n}^{\operatorname{FELP}}(\alpha, \rho)$.

On the event $\mathcal{E}_1 \cap \mathcal{E}_2$, the first two conditions of FELP contamination (the support total variation bound and the far exclusion condition) are guaranteed by the definitions of $\mathcal{E}_1$ and $\mathcal{E}_2$, respectively. Thus, we only need to verify the third condition, namely the local projection property.

To do this, we partition the contaminated empirical support. Let $S_1 = \{\widetilde{\mathbf{x}}_j : j \in [n], \, (\mathbf{x}_j, \widetilde{\mathbf{x}}_j) \sim (T, \operatorname{Id})_{\#} \nu_{\operatorname{loc}}\}$. For any $\widetilde{\mathbf{x}}_j \in S_1$, we have:
\begin{align*}
d(\widetilde{\mathbf{x}}_j, \, \operatorname{supp}(P_n)) \le d(\widetilde{\mathbf{x}}_j, \mathbf{x}_j) = d(\widetilde{\mathbf{x}}_j, T(\widetilde{\mathbf{x}}_j)) \le \rho,
\end{align*}
which implies $S_1 \subseteq N_{n,\rho}$. Conversely, let $S_2 = \{\widetilde{\mathbf{x}}_j : j \in [n], \, (\mathbf{x}_j, \widetilde{\mathbf{x}}_j) \sim \pi_{\operatorname{far}}\}$. For any $\widetilde{\mathbf{x}}_j \in S_2$, the definition of $\pi_{\operatorname{far}}$ ensures that $\widetilde{\mathbf{x}}_j \in E_\rho$. Consequently,
\begin{align*}
d(\widetilde{\mathbf{x}}_j, \, \operatorname{supp}(P_n)) \ge d(\widetilde{\mathbf{x}}_j, \, \operatorname{supp}(P^*)) \ge \rho,
\end{align*}
which verifies that $S_2 \subseteq E_{n,\rho}$.

Since the contaminated support can be decomposed as
\begin{align*}
S_1 \cup S_2 = \operatorname{supp}(\widetilde{P}_n) \setminus \{ \widetilde{\mathbf{x}}_z : z \in [n], \, \widetilde{\mathbf{x}}_z = \mathbf{x}_z \},
\end{align*}
and since the empirical local and far regions partition the contaminated support such that
\begin{align*}
\left( N_{n,\rho} \setminus \{ \widetilde{\mathbf{x}}_z : z \in [n], \, \widetilde{\mathbf{x}}_z = \mathbf{x}_z \} \right) \cap E_{n,\rho} = \emptyset,
\end{align*}
we identify $S_1 = N_{n,\rho} \setminus \{ \widetilde{\mathbf{x}}_z : z \in [n], \, \widetilde{\mathbf{x}}_z = \mathbf{x}_z \}$ and $S_2 = E_{n,\rho}$.

We now construct the empirical local projection map $T_n \colon N_{n,\rho} \to \operatorname{supp}(P_n)$ by setting:
\begin{align*}
T_n(\widetilde{\mathbf{x}}_j) =
\begin{cases}
T(\widetilde{\mathbf{x}}_j) & \text{if } \widetilde{\mathbf{x}}_j \in S_1, \\
\widetilde{\mathbf{x}}_j & \text{if } \widetilde{\mathbf{x}}_j \in N_{n,\rho} \setminus S_1.
\end{cases}
\end{align*}
By construction, this map satisfies $d(\widetilde{\mathbf{x}}_j, \, T_n(\widetilde{\mathbf{x}}_j)) \le \rho$ for all $\widetilde{\mathbf{x}}_j \in N_{n,\rho}$, and $T_{n\#} \widetilde{P}_n|_{N_{n,\rho}} \le P_n$.

This confirms that $\widetilde{P}_n \in \mathcal{C}^{\operatorname{FELP}}_{P_n}(\alpha, \rho)$ holds on $\mathcal{E}_1 \cap \mathcal{E}_2$. Consequently, the second marginal distribution of the joint conditional measure $\pi^n(\cdot, \cdot \mid \mathcal{E}_1 \cap \mathcal{E}_2)$ belongs to the empirical FELP family $\mathcal{D}_n^{\operatorname{FELP}}(P^*, \alpha)$, completing the proof.
\end{proof}

We now prove the minimax optimality results in the empirical case.

\begin{proof}[Proof of Theorem \ref{thm:empirical_minimax}]
We begin by establishing the upper bound. Taking $\mathcal{T}_n = \widehat{P}_{\operatorname{WF}}$ in the definition of the empirical minimax risk, we have:
\begin{align*}
\mathcal{R}_{1,n}(\alpha,\rho) &= \inf_{\mathcal{T}_n} \sup_{P^* \in \mathcal{G}} \sup_{\widetilde{\mathbb{P}}_n \in \mathcal{D}_n^{\operatorname{FELP}}(P^*,\alpha)} \mathbb{E}_{\widetilde{\mathbb{P}}_n} \left[ W_1(\mathcal{T}_n(\widetilde{P}_n), P^*) \right] \\
&\le \sup_{P^* \in \mathcal{G}} \sup_{\widetilde{\mathbb{P}}_n \in \mathcal{D}_n^{\operatorname{FELP}}(P^*,\alpha)} \mathbb{E}_{\widetilde{\mathbb{P}}_n} \left[ W_1(\widehat{P}_{\operatorname{WF}}, P^*) \right] \\
&\lesssim \rho + \sup_{P^* \in \mathcal{G}} \mathbb{E}_{\mathbb{P}_n} \left[ W_1(P_n, P^*) \right],
\end{align*}
where the last inequality follows directly from Theorem~\ref{thm:empirical_risk}, which holds uniformly for any empirical FELP contamination distribution $\widetilde{\mathbb{P}}_n \in \mathcal{D}_n^{\operatorname{FELP}}(P^*,\alpha)$ and any $P^* \in \mathcal{G}$.

Next, we establish the minimax lower bound. By definition of the empirical FELP family, the uncontaminated empirical measure $P_n$ satisfies $\mathbb{P}_n \in \mathcal{D}_n^{\operatorname{FELP}}(P^*,\alpha)$. Thus, we immediately obtain:
\begin{align*}
\inf_{\mathcal{T}_n} \sup_{P^* \in \mathcal{G}} \mathbb{E}_{\mathbb{P}_n} \left[ W_1(\mathcal{T}_n(P_n), P^*) \right] \le \mathcal{R}_{1,n}(\alpha,\rho).
\end{align*}
To establish the tight rate, it remains to prove that the population minimax risk is controlled by the empirical minimax risk, i.e., $\mathcal{R}_{1,\infty}\left(\frac{\alpha}{5}, \rho\right) \lesssim \mathcal{R}_{1,n}(\alpha, \rho)$. To achieve this, it suffices to construct a population-level estimator $\bar{\mathcal{T}}$ that satisfies:
\begin{align*}
\sup_{P^* \in \mathcal{G}} \sup_{\widetilde{P} \in \mathcal{C}^{\operatorname{FELP}}_{P^*}\left(\frac{\alpha}{5}, \rho\right)} W_1\left( \bar{\mathcal{T}}(\widetilde{P}), P^* \right) \lesssim \mathcal{R}_{1,n}(\alpha, \rho).
\end{align*}

Without loss of generality, we assume that the infimum in the definition of $\mathcal{R}_{1,n}(\alpha, \rho)$ is achieved by an empirical estimator $\bar{\mathcal{T}}_n$ (otherwise, we can proceed with a standard approximation argument and take the limit):
\begin{align*}
\sup_{P^* \in \mathcal{G}} \sup_{\widetilde{\mathbb{P}}_n \in \mathcal{D}_n^{\operatorname{FELP}}(P^*,\alpha)} \mathbb{E}_{\widetilde{\mathbb{P}}_n} \left[ W_1\left( \bar{\mathcal{T}}_n(\widetilde{P}_n), P^* \right) \right] \le \mathcal{R}_{1,n}(\alpha, \rho).
\end{align*}
For any fixed $P^* \in \mathcal{G}$ and any contaminated population distribution $\widetilde{P} \in \mathcal{C}^{\operatorname{FELP}}_{P^*}\left(\frac{\alpha}{5}, \rho\right)$, the conditions of Lemma~\ref{lem:empirical_lifting} are satisfied. Thus, there exists a joint coupling $\pi \in \Pi(P^*, \widetilde{P})$ such that the second (right) marginal distribution of the conditional joint measure $\pi^n(\cdot, \cdot \mid \mathcal{E}_1 \cap \mathcal{E}_2)$, denoted as $\widetilde{\mathbb{P}}'_n$, belongs to the empirical FELP class $\mathcal{D}^{\operatorname{FELP}}_n(P^*, \alpha)$.

By the construction of $\bar{\mathcal{T}}_n$, we can bound the probability of the estimation error under the product measure $\widetilde{P}^n$ as:
\begin{align*}
\widetilde{P}^n \left\{ W_1\left( \bar{\mathcal{T}}_n(\widetilde{P}_n), P^* \right) \le 4 \mathcal{R}_{1,n}(\alpha, \rho) \right\}
&\ge \pi^n(\mathcal{E}_1 \cap \mathcal{E}_2) \cdot \widetilde{\mathbb{P}}'_n \left\{ W_1\left( \bar{\mathcal{T}}_n(\widetilde{P}_n), P^* \right) \le 4 \mathcal{R}_{1,n}(\alpha, \rho) \right\} \\
&\ge \frac{3}{4} \left( 1 - \widetilde{\mathbb{P}}'_n \left\{ W_1\left( \bar{\mathcal{T}}_n(\widetilde{P}_n), P^* \right) > 4 \mathcal{R}_{1,n}(\alpha, \rho) \right\} \right) \\
&\ge \frac{3}{4} \left( 1 - \frac{\mathbb{E}_{\widetilde{\mathbb{P}}'_n}\left[ W_1\left( \bar{\mathcal{T}}_n(\widetilde{P}_n), P^* \right) \right]}{4 \mathcal{R}_{1,n}(\alpha, \rho)} \right) \\
&\ge \frac{3}{4} \left( 1 - \frac{1}{4} \right) = \frac{9}{16} > \frac{1}{2},
\end{align*}
where the second inequality follows from Lemma~\ref{lem:empirical_lifting}, the third from Markov's inequality, and the fourth from the definition of $\bar{\mathcal{T}}_n$ and the fact that $\widetilde{\mathbb{P}}'_n \in \mathcal{D}^{\operatorname{FELP}}_n(P^*, \alpha)$.

Now, we define the following candidate set of clean distributions:
\begin{align*}
\Gamma(\widetilde{P}) := \left\{ \nu \in \mathcal{G} : \widetilde{P}^n \left\{ W_1\left( \bar{\mathcal{T}}_n(\widetilde{P}_n), \nu \right) \le 4 \mathcal{R}_{1,n}(\alpha, \rho) \right\} > \frac{1}{2} \right\}.
\end{align*}
The preceding inequality guarantees that $P^* \in \Gamma(\widetilde{P})$, meaning $\Gamma(\widetilde{P})$ is non-empty. For the given contaminated population $\widetilde{P}$, we define our population-level estimator $\bar{\mathcal{T}}(\widetilde{P})$ by choosing an arbitrary element from $\Gamma(\widetilde{P})$. By definition of $\Gamma(\widetilde{P})$, we have:
\begin{align*}
\widetilde{P}^n \left\{ W_1\left( \bar{\mathcal{T}}_n(\widetilde{P}_n), \bar{\mathcal{T}}(\widetilde{P}) \right) \le 4 \mathcal{R}_{1,n}(\alpha, \rho) \right\} > \frac{1}{2},
\end{align*}
and we already established that:
\begin{align*}
\widetilde{P}^n \left\{ W_1\left( \bar{\mathcal{T}}_n(\widetilde{P}_n), P^* \right) \le 4 \mathcal{R}_{1,n}(\alpha, \rho) \right\} > \frac{1}{2}.
\end{align*}
Since both events hold with probability strictly greater than $1/2$ under $\widetilde{P}^n$, they must have a non-empty intersection. On this intersection, the triangle inequality yields:
\begin{align*}
W_1\left( P^*, \bar{\mathcal{T}}(\widetilde{P}) \right) \le W_1\left( P^*, \bar{\mathcal{T}}_n(\widetilde{P}_n) \right) + W_1\left( \bar{\mathcal{T}}_n(\widetilde{P}_n), \bar{\mathcal{T}}(\widetilde{P}) \right) \le 8 \mathcal{R}_{1,n}(\alpha, \rho).
\end{align*}
Since this bound holds uniformly for any $P^* \in \mathcal{G}$ and $\widetilde{P} \in \mathcal{C}^{\operatorname{FELP}}_{P^*}\left(\frac{\alpha}{5}, \rho\right)$, we conclude that:
\begin{align*}
\frac{1}{8} \mathcal{R}_{1, \infty}\left(\frac{\alpha}{5}, \rho\right) \le \frac{1}{8} \sup_{P^* \in \mathcal{G}} \sup_{\widetilde{P} \in \mathcal{C}^{\operatorname{FELP}}_{P^*}\left(\frac{\alpha}{5}, \rho\right)} W_1\left( P^*, \bar{\mathcal{T}}(\widetilde{P}) \right) \le \mathcal{R}_{1,n}(\alpha, \rho).
\end{align*}

Finally, we apply this general result to the specialized class $\mathcal{G} = \mathcal{G}_{\operatorname{cov}}(\sigma)$. By Theorem 3.1 of \citet{lei2020convergence}, we have $\sup_{P^* \in \mathcal{G}_{\operatorname{cov}}(\sigma)} \mathbb{E}[W_1(P_n, P^*)] = \mathcal{O}(\sqrt{d} n^{-1/d})$, while Theorem 3 of \citet{niles2022minimax} provides the corresponding lower bound $\inf_{\mathcal{T}_n} \sup_{P^* \in \mathcal{G}_{\operatorname{cov}}(\sigma)} \mathbb{E}[W_1(\mathcal{T}_n(P_n), P^*)] = \Omega(\sqrt{d} n^{-1/d})$. Furthermore, Theorem~\ref{thm:minimax_optimal} guarantees that $\mathcal{R}_{1,\infty}\left(\frac{\alpha}{5}, \rho_{\sigma}\right) \asymp \rho_{\sigma}$, where the covariance-resilience parameter scales as $\rho_{\sigma} \asymp \sigma \sqrt{d \alpha}$ \citep[Theorem 2]{nietert2023robust}. Combining these rates yields the matching minimax rate $\Theta\left( \sigma \sqrt{d\alpha} + \sqrt{d} n^{-1/d} \right)$, which is achieved by the proposed filter-based estimator $\widehat{P}_{\operatorname{WF}}$.
\end{proof}

\subsection{Convergence of the Algorithm}
In this section, we analyze the convergence properties of the proposed Wasserstein Filter (WF) algorithm and provide the proof of Theorem~\ref{thm:convergence_rate}.

\begin{proof}[Proof of Theorem \ref{thm:convergence_rate}]
Let $\mathbf{p} = \sigma(\mathbf{b})$ denote the filtered probability vector. We recall the dual formulation of the discrete $W_1$ optimal transport problem:
\begin{align*}
\Psi(\mathbf{p}) := W_1(P_n, P_{\mathbf{p}}) = \max_{\substack{\mathbf{u}, \mathbf{v} \in \mathbb{R}^n \\ u_i + v_j \le C_{ij}}} \langle \mathbf{u}, \, \mathbf{1}/n \rangle + \langle \mathbf{v}, \, \mathbf{p} \rangle.
\end{align*}
Under Condition 2, the optimal dual potential $\mathbf{v}^*(\mathbf{p})$ is unique. Consequently, Danskin's theorem implies that $\Psi(\mathbf{p})$ is differentiable with gradient:
\begin{align*}
\nabla \Psi(\mathbf{p}) = \mathbf{v}^*(\mathbf{p}).
\end{align*}
For the penalty term, Condition 3 ensures that $p_i \neq u_\alpha$ for all $i \in [n]$, which guarantees differentiability. The gradient of the penalty with respect to $\mathbf{p}$ is given by:
\begin{align*}
\nabla_{\mathbf{p}} \left[ \beta \sum_{i=1}^n (p_i - u_\alpha)_+ \right] = \beta \mathbf{m}(\mathbf{p}), \quad \text{where} \quad \mathbf{m}_i(\mathbf{p}) = \mathbb{I}_{\{p_i > u_\alpha\}}.
\end{align*}
By applying the chain rule, we obtain the gradient of the objective function $F(\mathbf{b})$ with respect to the filtered probability vector $\sigma(\mathbf{b})$:
\begin{align*}
\nabla_{\sigma(\mathbf{b})} F(\mathbf{b}) = \mathbf{v}^*(\sigma(\mathbf{b})) - \beta \mathbf{m}(\sigma(\mathbf{b})).
\end{align*}
Therefore, the full gradient of $F$ with respect to the parameter vector $\mathbf{b}$ is:
\begin{align*}
\nabla F(\mathbf{b}) = \frac{\partial F(\mathbf{b})}{\partial \sigma(\mathbf{b})} \frac{\partial \sigma(\mathbf{b})}{\partial \mathbf{b}} = \mathbf{J}_\sigma(\mathbf{b}) \left( \mathbf{v}^*(\sigma(\mathbf{b})) - \beta \mathbf{m}(\sigma(\mathbf{b})) \right),
\end{align*}
where $\mathbf{J}_\sigma(\mathbf{b}) = \operatorname{diag}(\sigma(\mathbf{b})) - \sigma(\mathbf{b})\sigma(\mathbf{b})^\top$ is the Jacobian matrix of the softmax function $\sigma(\mathbf{b})$.

To analyze the Lipschitz properties of this gradient, we define the following shorthand notation:
\begin{align*}
\mathbf{A}_\sigma(\mathbf{b}) := \mathbf{v}^*(\sigma(\mathbf{b})) - \beta \mathbf{m}(\sigma(\mathbf{b})),
\end{align*}
and introduce the following supremum and Lipschitz constants over the compact set $\mathbb{K}$:
\begin{align*}
L_\sigma := \operatorname{Lip}_{\mathbb{K}}(\sigma), \quad
L_{\mathbf{J}} := \operatorname{Lip}_{\mathbb{K}}(\mathbf{J}_{\sigma}), \quad
M_{\mathbf{J}} := \sup_{\mathbf{b} \in \mathbb{K}} \|\mathbf{J}_\sigma(\mathbf{b})\|, \quad \text{and} \quad
M_{\mathbf{A}} := \sup_{\mathbf{b} \in \mathbb{K}} \|\mathbf{A}_{\sigma}(\mathbf{b})\|.
\end{align*}
Since the softmax function $\sigma$ is smooth, both $\sigma$ and its Jacobian $\mathbf{J}_\sigma$ are Lipschitz continuous and bounded on the compact set $\mathbb{K}$. Furthermore, by Conditions 3 and 4, the vector $\mathbf{A}_\sigma(\mathbf{b})$ is bounded, ensuring that $M_{\mathbf{J}}$, $M_{\mathbf{A}}$, $L_\sigma$, and $L_{\mathbf{J}}$ are all finite.

Now, for any $\mathbf{b}, \mathbf{c} \in \mathbb{K}$ with corresponding filtered vectors $\mathbf{p} = \sigma(\mathbf{b})$ and $\mathbf{q} = \sigma(\mathbf{c})$, we bound the gradient difference as:
\begin{align*}
\|\nabla F(\mathbf{b}) - \nabla F(\mathbf{c}) \|
&= \left\| \left[\mathbf{J}_\sigma(\mathbf{b}) - \mathbf{J}_\sigma(\mathbf{c})\right] \mathbf{A}_\sigma(\mathbf{b}) + \mathbf{J}_\sigma(\mathbf{c}) \left[ \mathbf{A}_\sigma(\mathbf{b}) - \mathbf{A}_\sigma(\mathbf{c}) \right] \right\| \\
&\le \left\| \left[\mathbf{J}_\sigma(\mathbf{b}) - \mathbf{J}_\sigma(\mathbf{c})\right] \mathbf{A}_\sigma(\mathbf{b}) \right\| + \left\| \mathbf{J}_\sigma(\mathbf{c}) \left[ \mathbf{v}^*(\mathbf{p}) - \mathbf{v}^*(\mathbf{q}) \right] \right\| \\
&\le L_{\mathbf{J}} M_{\mathbf{A}} \|\mathbf{b} - \mathbf{c}\| + M_{\mathbf{J}} \|\mathbf{v}^*(\mathbf{p}) - \mathbf{v}^*(\mathbf{q})\| \\
&\le \left( L_{\mathbf{J}} M_{\mathbf{A}} + M_{\mathbf{J}} L_\sigma L_{\mathbf{v}^*} \right) \|\mathbf{b} - \mathbf{c}\|,
\end{align*}
where the first inequality follows from Condition 3 (which guarantees that $\mathbf{m}(\sigma(\mathbf{b})) = \mathbf{m}_0$ remains locally constant on $\mathbb{K}$), and the third inequality follows from the $L_{\mathbf{v}^*}$-Lipschitz continuity of the optimal dual potential $\mathbf{v}^*$ assumed in Condition 4.

Thus, the gradient $\nabla F(\mathbf{b})$ is $L_F$-Lipschitz continuous on $\mathbb{K}$ with Lipschitz constant $L_F := L_{\mathbf{J}} M_{\mathbf{A}} + M_{\mathbf{J}} L_\sigma L_{\mathbf{v}^*}$. By the standard descent (ascent) lemma, this Lipschitz property implies:
\begin{align*}
F(\mathbf{b}_{r+1}) &\ge F(\mathbf{b}_r) + \langle \nabla F(\mathbf{b}_r), \, \mathbf{b}_{r+1} - \mathbf{b}_{r} \rangle - \frac{L_F}{2} \|\mathbf{b}_{r+1} - \mathbf{b}_r\|^2 \\
&= F(\mathbf{b}_r) + \left( \gamma - \frac{L_F \gamma^2}{2} \right) \|\nabla F(\mathbf{b}_r)\|^2.
\end{align*}
Since the step size satisfies $0 < \gamma < \frac{2}{L_F}$, the coefficient $\gamma - \frac{L_F \gamma^2}{2}$ is strictly positive, guaranteeing that the objective values $\{F(\mathbf{b}_r)\}$ are non-decreasing. Summing this inequality from iteration $r = r_0$ to $r = r_0 + R - 1$ yields:
\begin{align*}
\sum_{r=r_0}^{r_0+R-1} \|\nabla F(\mathbf{b}_r)\|^2
&\le \frac{F(\mathbf{b}_{r_0+R}) - F(\mathbf{b}_{r_0})}{\gamma - \frac{L_F\gamma^2}{2}} \\
&\le \frac{\max_{i,j \in [n]} \|\mathbf{x}_i - \mathbf{x}_j\| - F(\mathbf{b}_{r_0})}{\gamma - \frac{L_F\gamma^2}{2}},
\end{align*}
where we used the fact that the maximum possible value of the $W_1$ distance (and thus $F(\mathbf{b})$) is bounded by the diameter of the empirical support. It follows immediately that:
\begin{align*}
\min_{r \in \{r_0, \dots, r_0 + R - 1\}} \|\nabla F(\mathbf{b}_r)\| \le \sqrt{ \frac{\max_{i,j \in [n]} \|\mathbf{x}_i - \mathbf{x}_j\| - F(\mathbf{b}_{r_0})}{\left( \gamma - \frac{L_F\gamma^2}{2} \right) R} },
\end{align*}
which completes the proof.
\end{proof}

\section{Additional Numerical Experiment Results}
\subsection{Two-Dimensional Synthetic Data}
\label{app:auto_alpha}

Figure \ref{fig:2d_visualization} shows the two-moon and circle data examples.

\begin{figure}[!h]
    \centering
    \begin{subfigure}[b]{0.48\textwidth}
        \centering
        \includegraphics[width=\linewidth]{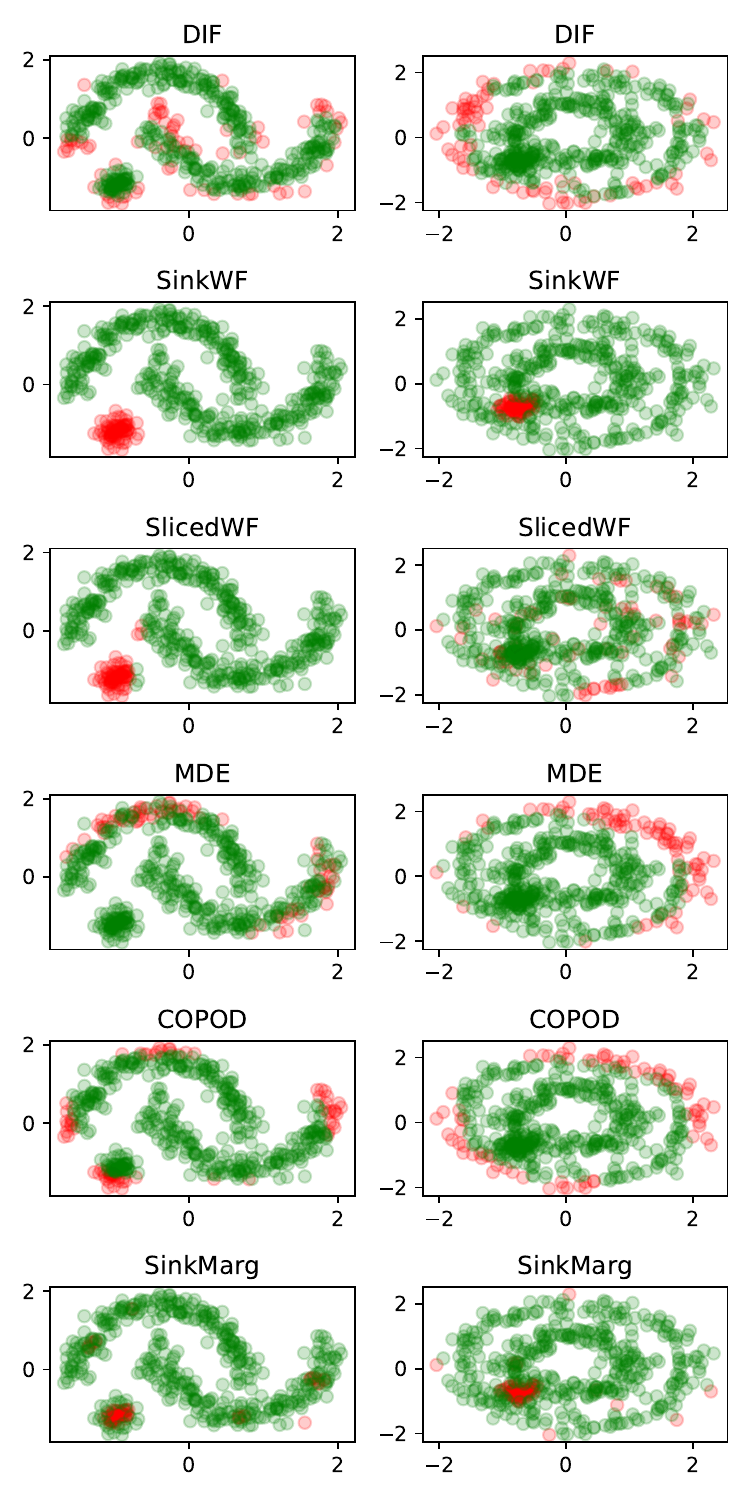}
        \caption{$\alpha=0.15$ with weak contamination, corrupted samples are gathered around $(-0.5,-0.5)$.}
    \end{subfigure}
    \hfill
    \begin{subfigure}[b]{0.48\textwidth}
        \centering
        \includegraphics[width=\linewidth]{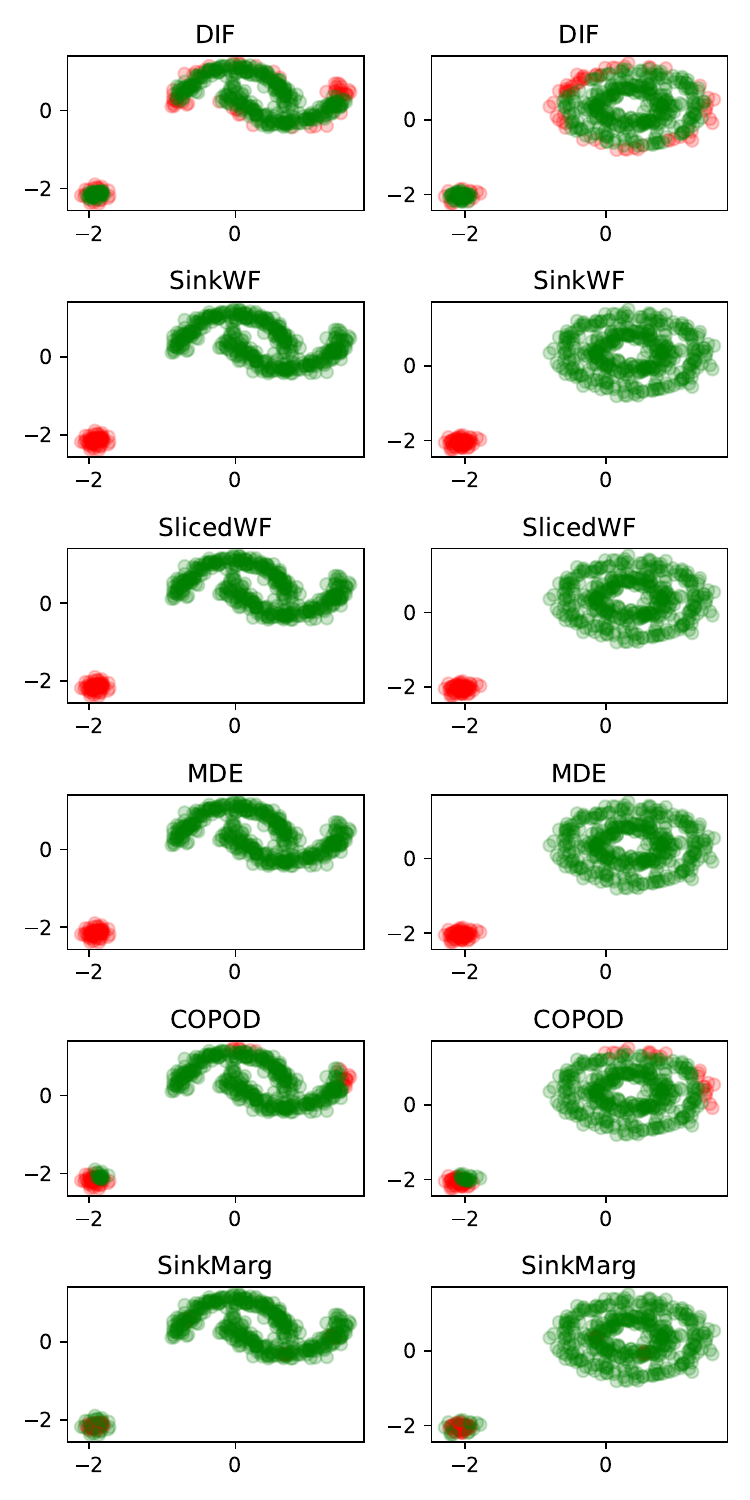}
        \caption{$\alpha=0.15$ with strong contamination, corrupted samples are gathered around $(-2.5,-2.5)$.}
    \end{subfigure}

    \caption{The two-moon and circle data examples. Green and red points are normal and corrupted samples classified by algorithm.}
    \label{fig:2d_visualization}
\end{figure}

\begin{table}[!h]
\centering
\caption{Results of SinkWF and SlicedWF on 2-D data using projection implementation. Weak and strong contamination implies the signal of corruption
strength. The mean of 50 times repetition are recorded and the corresponding standard error is
recorded in parenthesis right side.}
\begin{tabular}{clcccc} \bottomrule
                &       \multicolumn{5}{c}{Weak Contamination} \\ \cline{2-6}
                &       &  \multicolumn{2}{c}{AUC-ROC} & \multicolumn{2}{c}{AUC-PR} \\ \cline{3-6}
                &       & $\alpha=0.05$ & $\alpha=0.15$ & $\alpha=0.05$ & $\alpha=0.15$  \\
                &       SlicedWF  & 0.51(0.04) & 0.50(0.11) & 0.06(0.02) & 0.18(0.05)   \\
                &       SinkWF    & 0.67(0.12) & 0.90(0.07) & 0.18(0.15) & 0.73(0.16)    \\
  Twomoon data  &       \multicolumn{5}{c}{Strong Contamination} \\ \cline{2-6}
                &       & \multicolumn{2}{c}{AUC-ROC} & \multicolumn{2}{c}{AUC-PR} \\ \cline{3-6}
                &       & $\alpha=0.05$ & $\alpha=0.15$ & $\alpha=0.05$ & $\alpha=0.15$  \\
                &       SlicedWF  & 0.67(0.25) & 0.60(0.27) & 0.40(0.45) & 0.42(0.39)   \\
                &       SinkWF    & 0.69(0.05) & 0.79(0.04) & 0.24(0.11) & 0.49(0.11) \\ \hline
                                &       \multicolumn{5}{c}{Weak Contamination} \\ \cline{2-6}
                &       &  \multicolumn{2}{c}{AUC-ROC} & \multicolumn{2}{c}{AUC-PR} \\ \cline{3-6}
                &       & $\alpha=0.05$ & $\alpha=0.15$ & $\alpha=0.05$ & $\alpha=0.15$  \\
                &       SlicedWF  & 0.53(0.05) & 0.52(0.11) & 0.07(0.03) & 0.19(0.05)   \\
                &       SinkWF    & 0.63(0.06) & 0.72(0.03) & 0.14(0.06) & 0.35(0.05)
                    \\
  Circles data  &       \multicolumn{5}{c}{Strong Contamination} \\ \cline{2-6}
                &       & \multicolumn{2}{c}{AUC-ROC} & \multicolumn{2}{c}{AUC-PR} \\ \cline{3-6}
                &       & $\alpha=0.05$ & $\alpha=0.15$ & $\alpha=0.05$ & $\alpha=0.15$  \\
                &       SlicedWF  & 0.64(0.21) & 0.58(0.25) & 0.31(0.34) & 0.38(0.34)   \\
                &       SinkWF    & 0.69(0.04) & 0.78(0.06) & 0.22(0.09) & 0.45(0.13)   \\ \bottomrule
\end{tabular}
\end{table}

\begin{table}[!h]
\centering
\caption{Hyperparameters setting in 2-D data example. (p) means handle constraint with projection implementation while (r) means regularization implementation. In MDE, $\delta$ is the threshold for maximum eigenvalue of covariance matrix, $\mu$ is the factor that $(\mu\alpha) n$ samples with largest scores for random filtering.
Parameters notation in DIF are consistent with its corresponding paper. COPOD is tuning free. }
\begin{tabular}{cc} \bottomrule
Algorithms     &   Parameter Search Scope \\ \hline
DIF            &   $r\in\{10,50\},t\in\{6,10\},n\in\{128,256\} $ \\
MDE            &   $\delta\in\{0.1,0.4,0.7\}, \mu\in\{2,3,4\}$   \\
SlicedWF (p)   &   $R\in\{100,300\}, \gamma\in\{1,10^{-1},10^{-2},10^{-3}\}, S=3000$   \\
SlicedWF (r)   &   $R\in\{100,300\}, \gamma\in\{1,10^{-1},10^{-2},10^{-3}\}, S=3000, \beta\in\{1,10, 100\}$   \\
SinkWF (p)     &  $R=100,\lambda\in\{10^{-1},10^{-2},10^{-3}\},\gamma\in\{10^{-1},10^{-2},10^{-3},10^{-4}\}$   \\
SinkWF (r)     &  $R=100,\lambda\in\{10^{-1},10^{-2},10^{-3}\},\gamma\in\{10^{-1},10^{-2},10^{-3},10^{-4}\}, \beta\in\{1,10,100\}$   \\
SinkMarg       &  $\lambda\in\{10^{-1},10^{-2},10^{-3}\}$         \\
\bottomrule
\end{tabular}
\end{table}

\clearpage
\subsection{Benchmark Datasets}

Table \ref{table:benchmark_graph_penalty} below give the full performance results of the methods considered in this work for the graph benchmark datasets.

\begin{table}[!h]
\footnotesize{
\centering
\caption{Full results of graph benchmark datasets.}
\label{table:benchmark_graph_penalty}
\begin{tabular}{lllcccccc} \bottomrule
& & & \multicolumn{6}{c}{AUC-ROC} \\ \cline{4-9}
            & Dataset      & $(n,d,\alpha)$ & COPOD & DIF & MDE & SlicedWF &  SinkWF & SinkMarg \\ \cline{2-9}
\multirow{12}{*}{Graph}
            & NR-AR        & $(7823,300,3.94\%)$  & 0.71  & 0.71  & 0.50 & 0.56 & 0.76 & 0.56  \\
            & NR-ALBD      & $(7823,300,3.03\%)$  & 0.75  & 0.75  & 0.50 & 0.51 & 0.77 & 0.54 \\
            & NR-AhR       & $(7823,300,9.82\%)$  & 0.39  & 0.44  & 0.50 & 0.51 & 0.51 & 0.48 \\
            & NR-Aro       & $(7823,300,3.83\%)$  & 0.62  & 0.65  & 0.50 & 0.53 & 0.65 & 0.58 \\
            & NR-ER        & $(7823,300,10.11\%)$ & 0.47  & 0.51  & 0.50 & 0.56 & 0.57 & 0.47 \\
            & NR-ELBD      & $(7823,300,4.46\%)$  & 0.55  & 0.60  & 0.50 & 0.56 & 0.62 & 0.48 \\
            & NR-gamma     & $(7823,300,2.38\%)$  & 0.54  & 0.56  & 0.50 & 0.52 & 0.56 & 0.58 \\
            & SR-ARE       & $(7823,300,12.04\%)$ & 0.53  & 0.57  & 0.50 & 0.54 & 0.56 & 0.54 \\
            & SR-ATAD5     & $(7823,300,3.37\%)$  & 0.44  & 0.50  & 0.50 & 0.52 & 0.53 & 0.52 \\
            & SR-HSE       & $(7823,300,4.76\%)$  & 0.54  & 0.57  & 0.50 & 0.50 & 0.55 & 0.57 \\
            & SR-MMP       & $(7823,300,11.73\%)$ & 0.51  & 0.57  & 0.50 & 0.59 & 0.60 & 0.49 \\
            & SR-p53       & $(7823,300,5.41\%)$  & 0.57  & 0.61  & 0.50 & 0.55 & 0.61 & 0.62 \\ \hline
            & mean         &                      & 0.55  & 0.59  & 0.50 & 0.54 & 0.61 & 0.54 \\
            & wins         &                      & 0.00\% & 25.00\% & 0.00\% & 8.33\% & 66.67\% & 25.00\% \\
& & & \multicolumn{6}{c}{AUC-PR} \\ \cline{4-9}
             & Dataset & $(n,d,\alpha)$ & COPOD & DIF & MDE & SlicedWF &  SinkWF & SinkMarg \\ \cline{2-9}
\multirow{12}{*}{Graph}
            & NR-AR        & $(7823,300,3.94\%)$  & 0.09  & 0.09 & 0.04 & 0.05 & 0.16 & 0.05  \\
            & NR-ALBD      & $(7823,300,3.03\%)$  & 0.08  & 0.07 & 0.03 & 0.03 & 0.13 & 0.03 \\
            & NR-AhR       & $(7823,300,9.82\%)$  & 0.07  & 0.08 & 0.10 & 0.10 & 0.10 & 0.09 \\
            & NR-Aro       & $(7823,300,3.83\%)$  & 0.06  & 0.06 & 0.04 & 0.04 & 0.06 & 0.05 \\
            & NR-ER        & $(7823,300,10.11\%)$ & 0.10  & 0.11 & 0.10 & 0.11 & 0.13 & 0.09 \\
            & NR-ELBD      & $(7823,300,4.46\%)$  & 0.06  & 0.06 & 0.04 & 0.05 & 0.08 & 0.05 \\
            & NR-gamma     & $(7823,300,2.38\%)$  & 0.03  & 0.03 & 0.02 & 0.02 & 0.03 & 0.03 \\
            & SR-ARE       & $(7823,300,12.04\%)$ & 0.13  & 0.13 & 0.12 & 0.13 & 0.14 & 0.14 \\
            & SR-ATAD5     & $(7823,300,3.37\%)$  & 0.03  & 0.03 & 0.03 & 0.04 & 0.04 & 0.04 \\
            & SR-HSE       & $(7823,300,4.76\%)$  & 0.05  & 0.05 & 0.05 & 0.05 & 0.05 & 0.06 \\
            & SR-MMP       & $(7823,300,11.73\%)$ & 0.13  & 0.14 & 0.12 & 0.15 & 0.16 & 0.12 \\
            & SR-p53       & $(7823,300,5.41\%)$  & 0.08  & 0.08 & 0.05 & 0.07 & 0.07 & 0.09 \\ \hline
            & mean         &                      & 0.08  & 0.08 & 0.06 & 0.07 & 0.10 & 0.07  \\
            & wins         &                      & 16.67\%  & 16.67\%  & 8.33\%  & 16.67\% & 83.33\%  & 41.67\%  \\
\bottomrule
\end{tabular}
}
\end{table}

\begin{table}[!h]
\centering
\caption{Full results of tabular benchmark datasets with projection implementation and adaptive $\alpha$ on penalty implementation.}
\begin{tabular}{clcccc} \bottomrule
& & & \multicolumn{3}{c}{AUC-ROC} \\ \cline{4-6}
            & Dataset & $(n,d,\alpha)$  & SlicedWF &  SinkWF  & SinkAda  \\ \cline{2-6}
\multirow{14}{*}{Tabular}    & vertebral   & $(240,6,12.50\%)$    & 0.67  & 0.58  & 0.70   \\
                             & yeast       & $(1484,8,34.16\%)$   & 0.62  & 0.52  & 0.60  \\
                             & annthyroid  & $(7200,6,7.42\%)$    & 0.60  & 0.55  & 0.59  \\
                             & breastw     & $(683,9,34.99\%)$    & 0.99  & 0.97  & 0.75  \\
                             & cardio      & $(1831,21,9.61\%)$   & 0.60  & 0.62  & 0.56  \\
                             & tocography    & $(2114,21,22.04\%)$  & 0.86  & 0.60  & 0.78  \\
                             & Hepatitis   & $(80,19,16.25\%)$    & 0.57  & 0.76  & 0.65   \\
                             & Lymph       & $(148,18,4.05\%)$    & 0.48  & 0.83  & 0.65  \\
                             & shuttle     &$(49097,9,7.15\%)$    & 0.97  & 0.61  & 0.87  \\
                             & pendigits   &$(6870,16,2.27\%)$    & 0.72  & 0.51  & 0.85  \\
                             & satellite   &$(6435,36,31.64\%)$   & 0.52 & 0.55  & 0.69  \\
                             & WBC         &$(223,9,4.48\%)$      & 0.77 & 0.89  & 0.87  \\
                             & smtp        &$(95156,3,0.03\%)$    & 0.50 & 0.50  & 0.64  \\
                             & optdigits   &$(5216,64,2.88\%)$    & 0.67 & 0.53  & 0.73  \\  \hline
                             & mean        &                      & 0.68 & 0.64  & 0.71      \\
& & & \multicolumn{3}{c}{AUC-PR} \\ \cline{2-6}
             & Dataset & $(n,d,\alpha)$  & SlicedWF &  SinkWF  & SinkAda \\ \cline{4-6}
\multirow{14}{*}{Tabular}     & vertebral   & $(240,6,12.50\%)$    & 0.18  & 0.19  & 0.25   \\
                              & yeast       & $(1484,8,34.16\%)$   & 0.42  & 0.36  & 0.45  \\
                              & annthyroid  & $(7200,6,7.42\%)$    & 0.14  & 0.09  & 0.11  \\
                              & breastw     & $(683,9,34.99\%)$    & 0.98  & 0.93  & 0.51  \\
                              & cardio      & $(1831,21,9.61\%)$   & 0.12  & 0.22  & 0.15   \\
                              & tocography    & $(2114,21,22.04\%)$  & 0.59 & 0.39  & 0.54  \\
                              & Hepatitis   & $(80,19,16.25\%)$    & 0.22  & 0.43  & 0.25   \\
                              & Lymph       & $(148,18,4.05\%)$    & 0.04  & 0.20  & 0.16  \\
                              & shuttle     &$(49097,9,7.15\%)$    & 0.90  & 0.10 & 0.26  \\
                              & pendigits   &$(6870,16,2.27\%)$    & 0.08  & 0.02  & 0.11  \\
                              & satellite   &$(6435,36,31.64\%)$   & 0.34 & 0.35 & 0.54  \\
                              & WBC         &$(223,9,4.48\%)$      & 0.22 & 0.47 & 0.29  \\
                              & smtp        &$(95156,3,0.03\%)$    & 0.00 & 0.00  & 0.00  \\
                              & optdigits   &$(5216,64,2.88\%)$    & 0.04 & 0.03  & 0.29  \\   \hline
                              & mean        &                      & 0.31 & 0.27  & 0.28 \\
\bottomrule
\end{tabular}
\end{table}

\begin{table}[!h]
\centering
\caption{Full Results of graph benchmark datasets with projection implementation and adaptive $\alpha$ on penalty implementation.}
\begin{tabular}{clcccc} \bottomrule
& & & \multicolumn{3}{c}{AUC-ROC} \\ \cline{4-6}
            & Dataset      & $(n,d,\alpha)$ & SlicedWF &  SinkWF & SinkAda \\ \cline{2-6}
\multirow{12}{*}{Graph}
            & NR-AR        & $(7823,300,3.94\%)$   & 0.66  & 0.50  & 0.77  \\
            & NR-ALBD      & $(7823,300,3.03\%)$   & 0.58  & 0.53  & 0.78 \\
            & NR-AhR       & $(7823,300,9.82\%)$   & 0.58  & 0.51  & 0.56 \\
            & NR-Aro       & $(7823,300,3.83\%)$   & 0.64  & 0.51  & 0.68 \\
            & NR-ER        & $(7823,300,10.11\%)$  & 0.61  & 0.51  & 0.57 \\
            & NR-ELBD      & $(7823,300,4.46\%)$   & 0.61  & 0.53  & 0.64 \\
            & NR-gamma     & $(7823,300,2.38\%)$   & 0.52  & 0.50  & 0.59 \\
            & SR-ARE       & $(7823,300,12.04\%)$  & 0.57  & 0.50  & 0.56 \\
            & SR-ATAD5     & $(7823,300,3.37\%)$   & 0.55  & 0.50  & 0.55 \\
            & SR-HSE       & $(7823,300,4.76\%)$   & 0.51  & 0.52  & 0.54 \\
            & SR-MMP       & $(7823,300,11.73\%)$  & 0.54  & 0.52  & 0.60 \\
            & SR-p53       & $(7823,300,5.41\%)$   & 0.58  & 0.50  & 0.62 \\ \hline
            & mean         &                       & 0.58  & 0.51  & 0.62  \\
& & & \multicolumn{3}{c}{AUC-PR} \\ \cline{4-6}
             & Dataset & $(n,d,\alpha)$ & SlicedWF &  SinkWF & SinkAda \\ \cline{2-6}
\multirow{12}{*}{Graph}
            & NR-AR        & $(7823,300,3.94\%)$   & 0.06  & 0.04  & 0.17  \\
            & NR-ALBD      & $(7823,300,3.03\%)$   & 0.04  & 0.03  & 0.15 \\
            & NR-AhR       & $(7823,300,9.82\%)$   & 0.12  & 0.10  & 0.11 \\
            & NR-Aro       & $(7823,300,3.83\%)$   & 0.06  & 0.04  & 0.06 \\
            & NR-ER        & $(7823,300,10.11\%)$  & 0.13  & 0.10  & 0.12 \\
            & NR-ELBD      & $(7823,300,4.46\%)$   & 0.06  & 0.05  & 0.08 \\
            & NR-gamma     & $(7823,300,2.38\%)$   & 0.03  & 0.02  & 0.04 \\
            & SR-ARE       & $(7823,300,12.04\%)$  & 0.14  & 0.12  & 0.13 \\
            & SR-ATAD5     & $(7823,300,3.37\%)$   & 0.04  & 0.03  & 0.04 \\
            & SR-HSE       & $(7823,300,4.76\%)$   & 0.05  & 0.05  & 0.05 \\
            & SR-MMP       & $(7823,300,11.73\%)$  & 0.13  & 0.12  & 0.15 \\
            & SR-p53       & $(7823,300,5.41\%)$   & 0.06  & 0.05  & 0.08 \\ \hline
            & mean         &                       & 0.08  & 0.06  & 0.10 \\
\bottomrule
\end{tabular}
\end{table}

\begin{table}[!h]
\centering
\caption{Full results of time series benchmark datasets with projection implementation and adaptive $\alpha$ on penalty implementation.}
\begin{tabular}{clccccc} \bottomrule
& & & \multicolumn{3}{c}{AUC-ROC} \\ \cline{4-6}
            & Dataset & $(n,d,\alpha)$  & SlicedWF &  SinkWF & SinkAda \\ \cline{2-6}
\multirow{9}{*}{Timeseries}  & ECG200                & $(200,320,33.50\%)$   & 0.66 & 0.79  & 0.57  \\
                             & FordA                 & $(4921,320,48.65\%)$  & 0.55 & 0.50  & 0.51 \\
                             & FordB                 & $(4446,320,49.15\%)$  & 0.53 & 0.49  & 0.51 \\
                             & GunPoint              & $(200,320,50.00\%)$   & 0.70 & 0.57  & 0.59  \\
                             & Power                 & $(1096,320,49.91\%)$  & 0.80 & 0.53  & 0.70 \\
                             & Surface1              & $(621,320,43.80\%)$   & 0.37 & 0.54  & 0.68 \\
                             & Surface2              & $(980,320,38.37\%)$   & 0.28 & 0.49  & 0.54 \\
                             & TLECG                 & $(1162,320,50.00\%)$  & 0.24 & 0.66  & 0.74 \\
                             & Wafer                 & $(7164,320,10.64\%)$  & 0.54 & 0.57  & 0.63 \\ \hline
                             & mean                  &                       & 0.52 & 0.57  & 0.61 \\
& & & \multicolumn{3}{c}{AUC-PR} \\ \cline{4-6}
             & Dataset & $(n,d,\alpha)$  & SlicedWF &  SinkWF & SinkAda \\ \cline{2-6}
\multirow{9}{*}{Timeseries}  & ECG200                & $(200,320,33.50\%)$   & 0.45 & 0.59  & 0.37  \\
                             & FordA                 & $(4921,320,48.65\%)$  & 0.50 & 0.48  & 0.49 \\
                             & FordB                 & $(4446,320,49.15\%)$  & 0.51 & 0.49  & 0.50 \\
                             & GunPoint              & $(200,320,50.00\%)$   & 0.77 & 0.60  & 0.59 \\
                             & Power                 & $(1096,320,49.91\%)$  & 0.85 & 0.52  & 0.69 \\
                             & Surface1              & $(621,320,43.80\%)$   & 0.36 & 0.46  & 0.62 \\
                             & Surface2              & $(980,320,38.37\%)$   & 0.35 & 0.37  & 0.41 \\
                             & TLECG                 & $(1162,320,50.00\%)$  & 0.44 & 0.68  & 0.72 \\
                             & Wafer                 & $(7164,320,10.64\%)$  & 0.12 & 0.14  & 0.20 \\ \hline
                             & mean                  &                       & 0.48 & 0.48  & 0.51 \\
\bottomrule
\end{tabular}
\end{table}

\begin{table}[!h]
\centering
\caption{Hyperparameters setting in benchmark datasets. (p) means handle constraint with projection implementation while (r) means regularization implementation. In MDE, we take the quantile of eigenvalues of covariance matrix as the maximum eigenvalue threshold, so here $\delta$ means the quantile level, $\mu$ is the factor that $(\mu\alpha) n$ samples with largest scores for random filtering.
Parameters notation in DIF are consistent with its corresponding paper. COPOD is tuning free. }
\begin{tabular}{ll} \bottomrule
Algorithms     &   Parameter Search Scope \\ \hline
DIF            &   $r\in\{10,50,90\},t\in\{6,10,14\},n\in\{64, 128,256\} $ \\
MDE            &   $\delta\in\{0.5,0.6,0.7,0.8\}, \mu\in\{2,3,4,5\}$   \\
SlicedWF (p)   &   $R=100, \gamma\in\{10^{-1},10^{-3},10^{-5},10^{-7}\}, S=1000$   \\
SlicedWF (r)   &   $R=100, \gamma\in\{10^{-1},10^{-3},10^{-5},10^{-7}\}, S=1000, \beta\in\{30, 50, 90\}$   \\
SinkWF (p)     &  $R=100,\lambda\in\{10^{-1},10^{-2},10^{-3}\},\gamma\in\{10^{-1},10^{-3},10^{-5},10^{-7}\}$   \\
SinkWF (r)     &  $R=100,\lambda\in\{10^{-1},10^{-2},10^{-3}\},\gamma\in\{10^{-1},10^{-3},10^{-5},10^{-7}\}, \beta\in\{30,50,90\}$   \\
SinkMarg       &  $\lambda\in\{10^{-1},10^{-2},10^{-3}\}$   \\
\bottomrule
\end{tabular}
\end{table}

\clearpage
\subsection{Robust Generative Learning}
\label{app:mnist_hyperparameters}

\begin{table}[!h]
\centering
\caption{Result of DDPM training on MNIST with projection implementation. $\bb_0=\mathbf{1}$ for both SinkWF and SlicedWF, other settings are the same as penalty implementation.}
\begin{tabular}{lccccccc} \bottomrule
              & \multicolumn{7}{c}{Weak contamination} \\ \cline{2-8}
              & \multicolumn{3}{c}{FID score} & \multicolumn{2}{c}{AUC-ROC} & \multicolumn{2}{c}{AUC-PR} \\ \cline{2-8}
              & $\alpha=0$ & $\alpha=0.05$ & $\alpha=0.15$ & $\alpha=0.05$ & $\alpha=0.15$ & $\alpha=0.05$ & $\alpha=0.15$ \\
              SlicedWF  & 12.09 & 7.17 & 7.63 & 1.00  & 0.98 & 0.97 & 0.88 \\
              SinkWF    & 13.72 & 22.99  & 23.48  & 0.87  & 0.71  & 0.57  & 0.33 \\
\hline
              & \multicolumn{7}{c}{Strong contamination} \\ \cline{2-8}
              & \multicolumn{3}{c}{FID score} & \multicolumn{2}{c}{AUC-ROC} & \multicolumn{2}{c}{AUC-PR} \\ \cline{2-8}
              & $\alpha=0$ & $\alpha=0.05$ & $\alpha=0.15$ & $\alpha=0.05$ & $\alpha=0.15$ & $\alpha=0.05$ & $\alpha=0.15$ \\
              SlicedWF         & 8.94  & 20.77 & 4.83 & 1.00 & 1.00 & 1.00 & 1.00 \\
              SinkWF           & 15.77  & 16.74  & 4.20 & 1.00  & 1.00 & 1.00 & 1.00 \\
              \bottomrule
\end{tabular}
\end{table}

\begin{table}[!h]
\centering
\caption{Hyperparameters setting for MNIST data example.}
\begin{tabular}{lcccc} \bottomrule
              & & \multicolumn{3}{c}{Weak contamination} \\ \cline{3-5}
              & parameter & $\alpha=0$ & $\alpha=0.05$ & $\alpha=0.15$ \\
              SlicedWF (p)  & $(R,S,\gamma)$ & $(100,3000,10^{-4})$           & $(100,3000,10^{-4})$     & $(100,3000,10^{-4})$   \\
              SlicedWF (r)  & $(R,S,\gamma,\beta)$ & $(100,3000,10^{-2},1)$  & $(100,3000,10^{-2},1)$  & $(100,3000,10^{-1},1)$  \\
              SinkWF (p)  & $(R,\lambda,\gamma)$  & $(100,10^{-3},10^{-5})$   & $(100,10^{-3},10^{-5})$  & $(100,10^{-3},10^{-3})$ \\
              SinkWF (r)  & $(R,\lambda,\gamma,\beta)$ & $(100,10^{-2},10^{-3},100)$  & $(100,10^{-2},10^{-3},100)$  & $(100,10^{-2},10^{-1},100)$ \\
              SinkMarg    & $(\lambda)$  & - & -& - \\
\hline
              & & \multicolumn{3}{c}{Strong contamination} \\ \cline{3-5}
              & parameter & $\alpha=0$ & $\alpha=0.05$ & $\alpha=0.15$ \\
              SlicedWF (p)  & $(R,S,\gamma)$ & $(100,3000,10^{-5})$ & $(100,3000,10^{-5})$ & $(100,3000,10^{-5})$ \\
              SlicedWF (r)  & $(R,S,\gamma,\beta)$ & $(100,3000,10^{-1},1)$  & $(100,3000,10^{-1},1)$  & $(100,3000,10^{-1},1)$\\
              SinkWF (p)    & $(R,\lambda,\gamma)$ & $(100,10^{-2},10^{-3})$   & $(100,10^{-2},10^{-3})$  & $(100,10^{-2},10^{-3})$\\
              SinkWF (r)    & $(R,\lambda,\gamma,\beta)$  & $(100,10^{-1},10^{-3},100)$  & $(100,10^{-1},10^{-3},100)$  &  $(100,10^{-1},10^{-3},100)$  \\
              SinkMarg      & $(\lambda)$  & - & -& - \\
              \bottomrule
\end{tabular}
\end{table}

\bibliographystyle{ims}
\bibliography{ref.bib}

@article{sinkhorn1964relationship,
  author  = {Sinkhorn, Richard},
  title   = {A Relationship Between Arbitrary Positive Matrices and Doubly Stochastic Matrices},
  journal = {The Annals of Mathematical Statistics},
  volume  = {35},
  number  = {2},
  pages   = {876--879},
  year    = {1964},
  publisher = {Institute of Mathematical Statistics}
}

@article{luise2018differential,
  title={Differential properties of sinkhorn approximation for learning with wasserstein distance},
  author={Luise, Giulia and Rudi, Alessandro and Pontil, Massimiliano and Ciliberto, Carlo},
  journal={Advances in Neural Information Processing Systems},
  volume={31},
  year={2018}
}

@article{kingma2014adam,
  title={Adam: A method for stochastic optimization},
  author={Kingma, Diederik P and Ba, Jimmy},
  journal={arXiv preprint arXiv:1412.6980},
  year={2014}
}

@article{cuturi2013sinkhorn,
  title={Sinkhorn distances: Lightspeed computation of optimal transport},
  author={Cuturi, Marco},
  journal={Advances in neural information processing systems},
  volume={26},
  year={2013}
}

@inproceedings{nietert2024robust,
  title={Robust distribution learning with local and global adversarial corruptions},
  author={Nietert, Sloan and Goldfeld, Ziv and Shafiee, Soroosh},
  booktitle={The Thirty Seventh Annual Conference on Learning Theory},
  pages={4007--4008},
  year={2024},
  organization={PMLR}
}

@article{nietert2023robust,
  title={Robust estimation under the Wasserstein distance},
  author={Nietert, Sloan and Cummings, Rachel and Goldfeld, Ziv},
  journal={arXiv preprint arXiv:2302.01237},
  year={2023}
}

@article{xu2023deep,
  title={Deep isolation forest for anomaly detection},
  author={Xu, Hongzuo and Pang, Guansong and Wang, Yijie and Wang, Yongjun},
  journal={IEEE Transactions on Knowledge and Data Engineering},
  volume={35},
  number={12},
  pages={12591--12604},
  year={2023},
  publisher={IEEE}
}

@inproceedings{li2020copod,
  title={Copod: copula-based outlier detection},
  author={Li, Zheng and Zhao, Yue and Botta, Nicola and Ionescu, Cezar and Hu, Xiyang},
  booktitle={2020 IEEE international conference on data mining (ICDM)},
  pages={1118--1123},
  year={2020},
  organization={IEEE}
}

@article{peyre2019computational,
  title={Computational optimal transport: With applications to data science},
  author={Peyr{\'e}, Gabriel and Cuturi, Marco and others},
  journal={Foundations and Trends{\textregistered} in Machine Learning},
  volume={11},
  number={5-6},
  pages={355--607},
  year={2019},
  publisher={Now Publishers, Inc.}
}

@article{adam2022projections,
  title={Projections onto the canonical simplex with additional linear inequalities},
  author={Adam, Luk{\'a}s and M{\'a}cha, V{\'a}clav},
  journal={Optimization Methods and Software},
  volume={37},
  number={2},
  pages={451--479},
  year={2022},
  publisher={Taylor \& Francis}
}

@article{ho2020denoising,
  title={Denoising diffusion probabilistic models},
  author={Ho, Jonathan and Jain, Ajay and Abbeel, Pieter},
  journal={Advances in neural information processing systems},
  volume={33},
  pages={6840--6851},
  year={2020}
}

@article{heusel2017gans,
  title={Gans trained by a two time-scale update rule converge to a local nash equilibrium},
  author={Heusel, Martin and Ramsauer, Hubert and Unterthiner, Thomas and Nessler, Bernhard and Hochreiter, Sepp},
  journal={Advances in neural information processing systems},
  volume={30},
  year={2017}
}

@article{han2022adbench,
  title={Adbench: Anomaly detection benchmark},
  author={Han, Songqiao and Hu, Xiyang and Huang, Hailiang and Jiang, Minqi and Zhao, Yue},
  journal={Advances in neural information processing systems},
  volume={35},
  pages={32142--32159},
  year={2022}
}

@book{ramsundar2019deep,
  title={Deep learning for the life sciences: applying deep learning to genomics, microscopy, drug discovery, and more},
  author={Ramsundar, Bharath and Eastman, Peter and Walters, Pat and Pande, Vijay},
  year={2019},
  publisher={O'Reilly Media}
}

@article{jaeger2018mol2vec,
  title={Mol2vec: unsupervised machine learning approach with chemical intuition},
  author={Jaeger, Sabrina and Fulle, Simone and Turk, Samo},
  journal={Journal of chemical information and modeling},
  volume={58},
  number={1},
  pages={27--35},
  year={2018},
  publisher={ACS Publications}
}

@misc{UCRArchive2018,
        title = {The UCR Time Series Classification Archive},
        author = {Dau, Hoang Anh and Keogh, Eamonn and Kamgar, Kaveh and Yeh, Chin-Chia Michael and Zhu, Yan
                  and Gharghabi, Shaghayegh and Ratanamahatana, Chotirat Ann and Yanping and Hu, Bing
                  and Begum, Nurjahan and Bagnall, Anthony and Mueen, Abdullah and Batista, Gustavo, and Hexagon-ML},
        year = {2018},
        month = {October},
        note = {\url{https://www.cs.ucr.edu/~eamonn/time_series_data_2018/}}
}

@article{dau2019ucr,
  title={The UCR time series archive},
  author={Dau, Hoang Anh and Bagnall, Anthony and Kamgar, Kaveh and Yeh, Chin-Chia Michael and Zhu, Yan and Gharghabi, Shaghayegh and Ratanamahatana, Chotirat Ann and Keogh, Eamonn},
  journal={IEEE/CAA Journal of Automatica Sinica},
  volume={6},
  number={6},
  pages={1293--1305},
  year={2019},
  publisher={IEEE}
}

@inproceedings{yue2022ts2vec,
  title={Ts2vec: Towards universal representation of time series},
  author={Yue, Zhihan and Wang, Yujing and Duan, Juanyong and Yang, Tianmeng and Huang, Congrui and Tong, Yunhai and Xu, Bixiong},
  booktitle={Proceedings of the AAAI conference on artificial intelligence},
  volume={36},
  number={8},
  pages={8980--8987},
  year={2022}
}

@inproceedings{steinhardt2018resilience,
  title={Resilience: A Criterion for Learning in the Presence of Arbitrary Outliers},
  author={Steinhardt, Jacob and Charikar, Moses and Valiant, Gregory},
  booktitle={9th Innovations in Theoretical Computer Science Conference (ITCS 2018)},
  pages={45--1},
  year={2018},
  organization={Schloss Dagstuhl--Leibniz-Zentrum f{\"u}r Informatik}
}

@article{zhu2022generalized,
  title={Generalized resilience and robust statistics},
  author={Zhu, Banghua and Jiao, Jiantao and Steinhardt, Jacob},
  journal={The Annals of Statistics},
  volume={50},
  number={4},
  pages={2256--2283},
  year={2022},
  publisher={Institute of Mathematical Statistics}
}

@article{clark2024finding,
  title={Finding outliers in Gaussian model-based clustering},
  author={Clark, Katharine M and McNicholas, Paul D},
  journal={Journal of Classification},
  volume={41},
  number={2},
  pages={313--337},
  year={2024},
  publisher={Springer}
}

@article{donoho1988automatic,
  title={The ``automatic" robustness of minimum distance functionals},
  author={Donoho, David L and Liu, Richard C},
  journal={The Annals of Statistics},
  pages={552--586},
  year={1988},
  publisher={JSTOR}
}

@article{lei2020convergence,
  title={Convergence and concentration of empirical measures under Wasserstein distance in unbounded functional spaces},
  author={Lei, Jing},
  journal={Bernoulli},
  volume={26},
  number={1},
  pages={767--798},
  year={2020},
  publisher={JSTOR}
}

@article{niles2022minimax,
  title={Minimax estimation of smooth densities in Wasserstein distance},
  author={Niles-Weed, Jonathan and Berthet, Quentin},
  journal={The Annals of Statistics},
  volume={50},
  number={3},
  pages={1519--1540},
  year={2022},
  publisher={JSTOR}
}

@article{huber1964robust,
  title={Robust Estimation of a Location Parameter},
  author={Huber, Peter J},
  journal={Ann. Math. Statist.},
  volume={35},
  number={4},
  pages={73--101},
  year={1964}
}

@article{loh2025theoretical,
  title={A theoretical review of modern robust statistics},
  author={Loh, Po-Ling},
  journal={Annual Review of Statistics and Its Application},
  volume={12},
  number={1},
  pages={477--496},
  year={2025},
  publisher={Annual Reviews}
}

@article{liu2019density,
  title={Density estimation with contamination: minimax rates and theory of adaptation},
  author={Liu, Haoyang and Gao, Chao},
  journal={Electronic Journal of Statistics},
  volume={13},
  number={2},
  year={2019}
}

@article{chao2023statistical,
  title={Statistical estimation under distribution shift: Wasserstein perturbations and minimax theory},
  author={Chao, Patrick and Dobriban, Edgar},
  journal={arXiv preprint arXiv:2308.01853},
  year={2023}
}

@article{kim2012robust,
  title={Robust kernel density estimation},
  author={Kim, JooSeuk and Scott, Clayton D},
  journal={The Journal of Machine Learning Research},
  volume={13},
  number={1},
  pages={2529--2565},
  year={2012},
  publisher={JMLR. org}
}

@inproceedings{humbert2022robust,
  title={Robust kernel density estimation with median-of-means principle},
  author={Humbert, Pierre and Le Bars, Batiste and Minvielle, Ludovic},
  booktitle={International Conference on Machine Learning},
  pages={9444--9465},
  year={2022},
  organization={PMLR}
}

@article{caffarelli2010free,
  title={Free boundaries in optimal transport and Monge-Ampere obstacle problems},
  author={Caffarelli, Luis A and McCann, Robert J},
  journal={Annals of mathematics},
  pages={673--730},
  year={2010},
  publisher={JSTOR}
}

@article{figalli2010optimal,
  title={The optimal partial transport problem},
  author={Figalli, Alessio},
  journal={Archive for rational mechanics and analysis},
  volume={195},
  number={2},
  pages={533--560},
  year={2010},
  publisher={Springer}
}

@article{liero2018optimal,
  title={Optimal entropy-transport problems and a new Hellinger--Kantorovich distance between positive measures},
  author={Liero, Matthias and Mielke, Alexander and Savar{\'e}, Giuseppe},
  journal={Inventiones mathematicae},
  volume={211},
  number={3},
  pages={969--1117},
  year={2018},
  publisher={Springer}
}

@article{chizat2018unbalanced,
  title={Unbalanced optimal transport: Dynamic and Kantorovich formulations},
  author={Chizat, Lenaic and Peyr{\'e}, Gabriel and Schmitzer, Bernhard and Vialard, Fran{\c{c}}ois-Xavier},
  journal={Journal of Functional Analysis},
  volume={274},
  number={11},
  pages={3090--3123},
  year={2018},
  publisher={Elsevier}
}

@article{balaji2020robust,
  title={Robust optimal transport with applications in generative modeling and domain adaptation},
  author={Balaji, Yogesh and Chellappa, Rama and Feizi, Soheil},
  journal={Advances in Neural Information Processing Systems},
  volume={33},
  pages={12934--12944},
  year={2020}
}

@inproceedings{arjovsky2017wasserstein,
  title={Wasserstein generative adversarial networks},
  author={Arjovsky, Martin and Chintala, Soumith and Bottou, L{\'e}on},
  booktitle={International conference on machine learning},
  pages={214--223},
  year={2017},
  organization={Pmlr}
}

@inproceedings{mukherjee2021outlier,
  title={Outlier-robust optimal transport},
  author={Mukherjee, Debarghya and Guha, Aritra and Solomon, Justin M and Sun, Yuekai and Yurochkin, Mikhail},
  booktitle={International Conference on Machine Learning},
  pages={7850--7860},
  year={2021},
  organization={PMLR}
}

@article{ma2025inference,
  title={Inference via robust optimal transportation: theory and methods},
  author={Ma, Yiming and Liu, Hang and La Vecchia, Davide and Lerasle, Matthieu},
  journal={International Statistical Review},
  year={2025},
  publisher={Wiley Online Library}
}

@inproceedings{pham2020unbalanced,
  title={On unbalanced optimal transport: An analysis of sinkhorn algorithm},
  author={Pham, Khiem and Le, Khang and Ho, Nhat and Pham, Tung and Bui, Hung},
  booktitle={International Conference on Machine Learning},
  pages={7673--7682},
  year={2020},
  organization={PMLR}
}

@article{nguyen2023unbalanced,
  title={On unbalanced optimal transport: Gradient methods, sparsity and approximation error},
  author={Nguyen, Quang Minh and Nguyen, Hoang H and Zhou, Yi and Nguyen, Lam M},
  journal={Journal of Machine Learning Research},
  volume={24},
  number={384},
  pages={1--41},
  year={2023}
}

@article{diakonikolas2019robust,
  title={Robust estimators in high-dimensions without the computational intractability},
  author={Diakonikolas, Ilias and Kamath, Gautam and Kane, Daniel and Li, Jerry and Moitra, Ankur and Stewart, Alistair},
  journal={SIAM Journal on Computing},
  volume={48},
  number={2},
  pages={742--864},
  year={2019},
  publisher={SIAM}
}

@book{villani2009optimal,
  title={Optimal transport: old and new},
  author={Villani, C{\'e}dric and others},
  volume={338},
  year={2009},
  publisher={Springer}
}

@article{flamary2021pot,
  author  = {R{\'e}mi Flamary and Nicolas Courty and Alexandre Gramfort and Mokhtar Z. Alaya and Aur{\'e}lie Boisbunon and Stanislas Chambon and Laetitia Chapel and Adrien Corenflos and Kilian Fatras and Nemo Fournier and L{\'e}o Gautheron and Nathalie T.H. Gayraud and Hicham Janati and Alain Rakotomamonjy and Ievgen Redko and Antoine Rolet and Antony Schutz and Vivien Seguy and Danica J. Sutherland and Romain Tavenard and Alexander Tong and Titouan Vayer},
  title   = {POT: Python Optimal Transport},
  journal = {Journal of Machine Learning Research},
  year    = {2021},
  volume  = {22},
  number  = {78},
  pages   = {1-8}
}

@inproceedings{feydy2019interpolating,
    title={Interpolating between Optimal Transport and MMD using Sinkhorn Divergences},
    author={Feydy, Jean and S{\'e}journ{\'e}, Thibault and Vialard, Fran{\c{c}}ois-Xavier and Amari, Shun-ichi and Trouve, Alain and Peyr{\'e}, Gabriel},
    booktitle={The 22nd International Conference on Artificial Intelligence and Statistics},
    pages={2681--2690},
    year={2019}
}

\end{document}